%% file: main.tex
\PassOptionsToPackage{dvipsnames,table}{xcolor}
\documentclass{openwam}

\usepackage[utf8]{inputenc}
\usepackage{array}
\usepackage{ragged2e}
\usepackage{tabularx}
\usepackage{wrapfig}
\usepackage{amsmath,amssymb}
\usepackage{pifont}
\usepackage{amsthm}
\usepackage[colorinlistoftodos]{todonotes}
\usepackage{enumitem}

\let\Letter\relax

\usepackage{marvosym}
\input{wing-style}
\newif\ifoverleaffast
\overleaffastfalse

\definecolor{orange}{HTML}{E8925A}
\definecolor{green}{HTML}{6CA973}
\definecolor{blue}{HTML}{6B9AC4}
\definecolor{gray}{HTML}{808080}
\definecolor{red}{HTML}{D9534F}
\definecolor{deep_purple}{HTML}{845EC2}
\definecolor{deep_red}{HTML}{CD8899}
\definecolor{gain_green}{HTML}{45A049}
\newcolumntype{Y}{>{\RaggedRight\arraybackslash}X}

\input{math_commands.tex}

\newtheorem{proposition}{Proposition}

\newtcolorbox{theorembox}{
  enhanced,
  breakable,
  colback=blue!20!cyan!10!white!50,
  colframe=gray!100!black,
  arc=1mm,
  boxrule=0.8pt,
  left=6pt,
  right=6pt,
  top=4pt,
  bottom=4pt,
  before skip=6pt,
  after skip=6pt
}

\newenvironment{takeaway}{
  \begin{theorembox}
}{
  \end{theorembox}
}

\title{\wingtitleicon \winggradient{WING}: \winggradient{W}orld Action Learning via \winggradient{IN}teraction-Centric Spectral Latent \winggradient{G}uidance}

\author[\dagger]{Zhiming~Liu}
\author[\dagger]{Yikun~Miao}
\author[\dagger]{Ying~Chen}
\author{Hongrui~Yin}
\author{Fangqi~Zhu}
\author{Xiaoyi~Pang}
\author{Quanxin~Shou}
\author{Zhengyang~Yan}
\author{Haodong~Wang}
\author[\wingcorrespondencemark]{Song~Guo}

\affiliation{The Hong Kong University of Science and Technology\\
  Hong Kong SAR, China}

\contribution[\dagger]{Equal contribution}
\correspondence{\email{songguo@cse.ust.hk}}
\homepage{\href{https://mikuz12.github.io/wing/}{\textcolor{wingprojectpink}{\textrm{https://mikuz12.github.io/wing/}}}}

\abstract{\input{text/0_abstract}}

\begin{document}

\maketitle
\input{text/1_introduction}

\input{text/related_work_main}
\input{text/method}

\vspace{1em}
\input{text/6_experiments}

\input{text/7_conclusion}

{\small
\bibliographystyle{arxiv-numbered}
\bibliography{ref}
}

\ifoverleaffast
\else
  \clearpage
  \appendix
  \input{text/appendix/theory}
  \input{text/appendix/related_work}

  \input{text/appendix/a_implementation_training_details}

  \input{text/appendix/b_evaluation_details}

  \input{text/appendix/c_additional_results}

\fi

\end{document}

%% file: wing-style.tex
\usepackage{fontawesome5}
\input{wing-gradient}

\DeclareRobustCommand{\wingtitleicon}{%
  \begingroup
  \pdfimageresolution=144\relax
  \raisebox{-0.20em}{\includegraphics[height=1.28em]{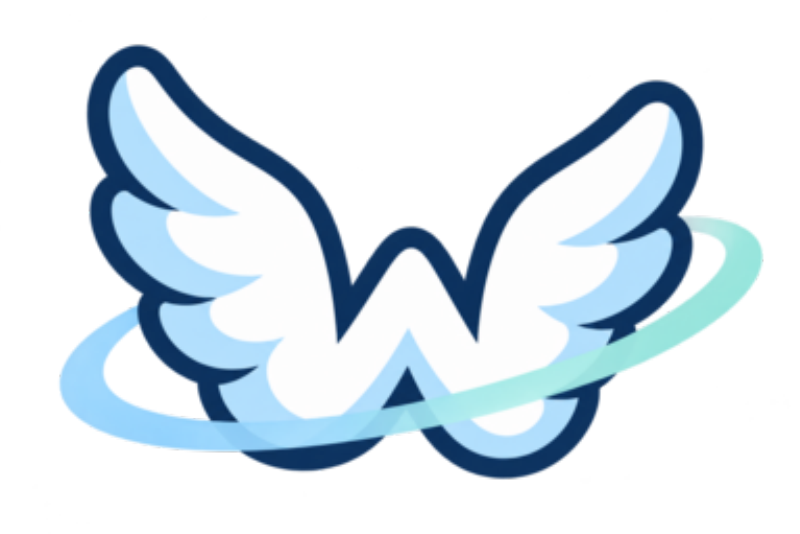}}%
  \endgroup\nobreak\hspace{0.28em}}

\colorlet{wingcite}{MidnightBlue}
\definecolor{wingrefred}{HTML}{FF0000}
\definecolor{wingprojectpink}{HTML}{D63384}
\definecolor{wamtitle}{HTML}{000000}
\definecolor{wambg}{HTML}{EAF2FF}

\DeclareRobustCommand{\wingmetadataicon}[1]{%
  \textcolor{black}{\makebox[1.25em][c]{#1}}}
\newcommand{\homepage}[1]{%
  \metadata[\wingmetadataicon{\faGlobe}~Project Page]{#1}}

\hypersetup{
  colorlinks=true,
  pdfborder={0 0 0},
  citecolor=wingcite,
  linkcolor=wingrefred,
  urlcolor=wingcite,
  pdftitle={WING: World Action Learning via Interaction-Centric Spectral Latent Guidance},
  pdfauthor={Zhiming Liu, Yikun Miao, Ying Chen, Hongrui Yin, Fangqi Zhu, Xiaoyi Pang, Quanxin Shou, Zhengyang Yan, Haodong Wang, Song Guo}
}
\setcitestyle{authoryear,round,semicolon,aysep={,},yysep={;}}

\DeclareRobustCommand{\wingcorrespondencemark}{\mbox{\scriptsize\Letter}}
\renewcommand\contribution[2][]{%
  \addtolist[#1]{#2}{\contributionlist}{\contributionformat}{\par}}
\renewcommand{\correspondence}[1]{%
  \contribution[\wingcorrespondencemark]{{\sffamily\bfseries Correspondence:} #1}}

\renewcommand\affiliation[2][]{%
  \addtolist[#1]{#2}{\affiliationlist}{\affiliationformat}{\par}}

\crefname{figure}{\textcolor{black}{Fig.}}{\textcolor{black}{Figs.}}
\Crefname{figure}{\textcolor{black}{Fig.}}{\textcolor{black}{Figs.}}
\crefname{table}{\textcolor{black}{Tab.}}{\textcolor{black}{Tabs.}}
\Crefname{table}{\textcolor{black}{Tab.}}{\textcolor{black}{Tabs.}}
\crefname{section}{\textcolor{black}{Sec.}}{\textcolor{black}{Secs.}}
\Crefname{section}{\textcolor{black}{Sec.}}{\textcolor{black}{Secs.}}
\crefname{subsection}{\textcolor{black}{Sec.}}{\textcolor{black}{Secs.}}
\Crefname{subsection}{\textcolor{black}{Sec.}}{\textcolor{black}{Secs.}}
\crefname{subsubsection}{\textcolor{black}{Sec.}}{\textcolor{black}{Secs.}}
\Crefname{subsubsection}{\textcolor{black}{Sec.}}{\textcolor{black}{Secs.}}

\newcommand{\wingheadingfont}{\sffamily\bfseries}
\titleformat*{\section}{\Large\wingheadingfont}
\titleformat*{\subsection}{\large\wingheadingfont}
\titleformat*{\subsubsection}{\normalsize\wingheadingfont}
\titleformat*{\paragraph}{\wingheadingfont}
\renewcommand{\title}[1]{\def\titlelist{{\huge\wingheadingfont\color{wamtitle}#1}}}

\patchcmd{\mymaketitle}
  {\tcbset{bottom=0.5cm}}
  {\tcbset{bottom=0.54cm}}
  {}{\PackageError{wing}{Could not tighten title-panel bottom padding}{Check openwam.cls.}}
\patchcmd{\mymaketitle}
  {\vskip 0.5cm}
  {\vskip 0.32cm}
  {}{\PackageError{wing}{Could not tighten project-page spacing}{Check openwam.cls.}}

\patchcmd{\mymaketitle}
  {\hypersetup{linkcolor=wamtitle,citecolor=wamtitle,urlcolor=wamtitle}}
  {\hypersetup{linkcolor=wingrefred,citecolor=wingcite,urlcolor=wingcite}}
  {}{\PackageError{wing}{Could not set title-panel link colors}{Check openwam.cls.}}

%% file: wing-gradient.tex
\newsavebox{\winggradienttext}
\newsavebox{\winggradientform}
\newdimen\winggradientpad
\DeclareRobustCommand{\winggradient}[1]{%
  \leavevmode\begingroup
  \sbox{\winggradienttext}{#1}%
  \pgfmathsetmacro{\owmw}{\wd\winggradienttext/1bp}%
  \pgfmathsetmacro{\owmh}{(\ht\winggradienttext+\dp\winggradienttext)/1bp}%
  \pgfmathsetmacro{\owmd}{\dp\winggradienttext/1bp}%
  \pgfmathsetmacro{\owml}{\owmw*sin(112)+\owmh*abs(cos(112))}%
  \pgfmathsetmacro{\owmxa}{\owmw/2-\owml*sin(112)/2}%
  \pgfmathsetmacro{\owmya}{\owmh/2-\owmd-\owml*cos(112)/2}%
  \pgfmathsetmacro{\owmxb}{\owmw/2+\owml*sin(112)/2}%
  \pgfmathsetmacro{\owmyb}{\owmh/2-\owmd+\owml*cos(112)/2}%
  \immediate\pdfobj{<< /ShadingType 2 /ColorSpace /DeviceRGB
    /Coords [\owmxa\space\owmya\space\owmxb\space\owmyb]
    /Function << /FunctionType 3 /Domain [0 1]
      /Functions [
        << /FunctionType 2 /Domain [0 1]
           /C0 [0.13333333 0.47843137 0.51372549]
           /C1 [0.23529412 0.40784314 0.70196078] /N 1 >>
        << /FunctionType 2 /Domain [0 1]
           /C0 [0.23529412 0.40784314 0.70196078]
           /C1 [0.47450980 0.38039216 0.67843137] /N 1 >>]
      /Bounds [0.48] /Encode [0 1 0 1] >> /Extend [true true] >>}%
  \edef\owmshading{\the\pdflastobj}%
  \winggradientpad=.2em\relax
  \setbox\winggradientform=\hbox{%
    \kern\winggradientpad
    \pdfliteral{q 7 Tr}%
    \copy\winggradienttext
    \kern-\wd\winggradienttext
    \pdfliteral{/OWMGradient sh Q}%
    \kern\wd\winggradienttext
    \kern\winggradientpad
  }%
  \ht\winggradientform=\dimexpr\ht\winggradienttext+\winggradientpad\relax%
  \dp\winggradientform=\dimexpr\dp\winggradienttext+\winggradientpad\relax%
  \immediate\pdfxform resources{/Shading << /OWMGradient \owmshading\space 0 R >>}\winggradientform
  \setbox\winggradientform=\hbox{\kern-\winggradientpad\pdfrefxform\pdflastxform\kern-\winggradientpad}%
  \wd\winggradientform=\wd\winggradienttext
  \ht\winggradientform=\ht\winggradienttext
  \dp\winggradientform=\dp\winggradienttext
  \box\winggradientform
  \endgroup
}

%% file: math_commands.tex
\usepackage{amsmath,amsfonts,bm}

\def\eqref#1{equation~\ref{#1}}

\def\1{\bm{1}}

\DeclareMathAlphabet{\mathsfit}{\encodingdefault}{\sfdefault}{m}{sl}
\SetMathAlphabet{\mathsfit}{bold}{\encodingdefault}{\sfdefault}{bx}{n}



%% file: text/0_abstract.tex

    Learning general-purpose robot policies requires large-scale real-world interaction data, yet collecting robot demonstrations through teleoperation remains expensive and difficult to scale. Egocentric videos provide a rich source of human interaction experience that shares task-relevant semantics with robotic manipulation, creating an opportunity to align human and robot actions in a shared latent action space for cross-embodiment knowledge transfer. However, existing latent action approaches typically infer actions through reconstruction between consecutive frames, which are not inherently interaction-centric and can be dominated by nuisance variations, such as ego-camera motion. Moreover, although human interactions and robot actions share interaction semantics, they often exhibit substantially different temporal dynamics, making direct transfer to robot policies challenging.
    To address these challenges, we propose \textbf{WING} (\textbf{W}orld Action Learning via \textbf{IN}teraction-Centric Spectral Latent \textbf{G}uidance), a framework for transferring interaction knowledge from egocentric videos to robot policies. First, we introduce an interaction--motion decoupling mechanism that separates interaction-relevant signals from task-irrelevant motion and selectively distills the interaction-centric components into latent actions. Second, motivated by the observation that cross-embodiment task semantics are primarily encoded in slowly varying temporal structures, we identify shared components between egocentric latent actions and robot behaviors in the spectral domain and use them as guidance for action generation at inference time.
    With these designs, WING achieves average success rates of 99.20\% on LIBERO, 93.80\% on RoboTwin 2.0, and 57.7\% on RoboCasa--GR1, while also demonstrating strong performance across four real-world tasks under diverse generalization settings. These results demonstrate that WING can effectively distill embodied interaction knowledge from large-scale egocentric videos and transfer it to robot control, providing a scalable pathway for acquiring physical interaction knowledge from human experience.

%% file: text/1_introduction.tex
\section{Introduction}
\label{sec:introduction}

Building generalist robot models that can reliably operate across diverse real-world environments requires large-scale and diverse training data~\citep{pi_0.5,openvla,lbva,dreamzero,halo,behaviorvla}. Directly collecting such data through robot teleoperation remains costly and difficult to scale. Human videos~\citep{humannet,sth2sth,holoassist,egoverse,ego4d}, in contrast, contain abundant embodied interactions across diverse tasks and environments, offering a potentially scalable source of experience for robot learning. \textit{Egocentric videos}~\citep{egodex,egoverse,egoscale,ego4d,holoassist,epic-kitchens} are particularly appealing in this context, as they capture rich human-object interactions from a first-person perspective that closely resembles the visual observations available to robots.

Transferring human experiences to robot control remains challenging because egocentric videos lack direct and generalizable action supervision. Explicit annotations such as hand poses are often tightly coupled to human morphology and kinematics, limiting their ability to provide transferable supervision for robot action learning. \emph{Latent action learning}~\citep{univla,lapa,cdlam,dreamdojo} offers a more flexible alternative by inferring action-like representations directly from visual transitions and encoding them in a latent space, without requiring explicit action labels. However, extracting \emph{transferable interaction knowledge} from visual transitions is fundamentally different from merely explaining how images change over time. Existing latent action models are typically trained to reconstruct future observations, encouraging them to capture any visually predictive variation rather than selectively modeling interaction-relevant dynamics. This leads to a central question: \textbf{\emph{How can latent action learning extract interaction-centric knowledge from large-scale egocentric videos for transferable robot learning?}}

We first examine the sources of nuisance factors in latent actions extracted from egocentric videos. We observe that a major source stems from \textbf{task-irrelevant observer motion}. Unlike robot manipulation settings where the head camera is typically mounted on the robot and remains stable, human-captured egocentric videos often exhibit substantial viewpoint changes caused by rapid head and body movements. These camera-induced variations can become entangled with interaction-related dynamics and even dominate the learned latent representations~\citep{motion,whylatentfail}. However, transferable information primarily lies in the underlying hand-object interactions rather than task-irrelevant motion. This motivates us to design a latent action model that can disentangle interaction dynamics from observer-induced variations~\citep{cdlam,activemimic}.

Moreover, we also need to account for differences in \textbf{execution dynamics} between human and robot interactions. Even when meaningful interaction structures are captured, human and robot behaviors can differ in how the same interaction is executed due to differences in embodiment, execution speed, and control granularity. Such discrepancies are particularly pronounced in fine-grained temporal patterns, which can vary substantially even when the underlying interaction semantics are shared. We therefore analyze the temporal structures of human egocentric interactions and robot actions in the frequency domain. Our analysis reveals substantially stronger correspondence in the low-frequency components of semantically matched human-robot interaction pairs. Motivated by this observation, we treat the low-frequency components of latent actions as transferable interaction information and use them as guidance for robot policy learning and action generation.

\begin{figure}[t]
    \centering
    \includegraphics[width=1\linewidth]{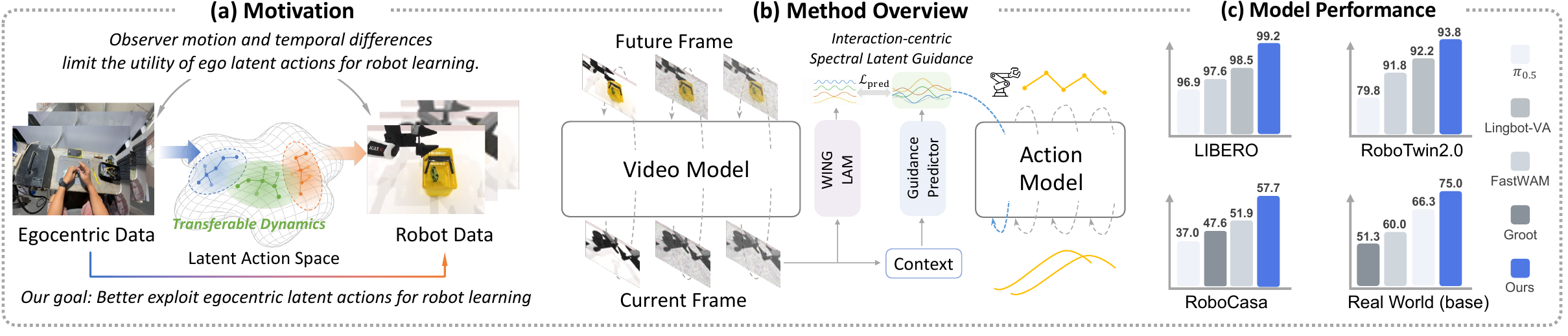}
    \caption{
(a) \textbf{Motivation:}
Egocentric videos provide rich interaction experience, but transferring this knowledge to robot policies remains challenging. We aim to extract transferable latent actions from egocentric data for robot learning.
(b) \textbf{Method:}
WING learns interaction-centric latent actions, extracts low-frequency temporal
structure in the frequency domain, and uses it as guidance for robot action
generation.
(c) \textbf{Performance:}
WING consistently achieves strong performance among the compared methods across simulation and real-world benchmarks.
}

    \label{fig:placeholder}
\end{figure}

Motivated by these observations, we introduce \textbf{WING} (\textbf{W}orld Action Learning via \textbf{IN}teraction-Centric Spectral Latent \textbf{G}uidance), a framework for extracting transferable interaction dynamics from large-scale egocentric videos and leveraging them for robot learning. WING first learns interaction-centric latent actions by suppressing sensitivity to observer motion while preserving interaction-relevant changes. Specifically, we decompose visual transitions into global observer-induced motion and local interaction-related motion, and selectively distill the latter into the latent action representation. WING then decomposes latent action trajectories in the frequency domain using the discrete cosine transform (DCT)~\citep{dct} and retains low-frequency components that exhibit more consistent structure across human and robot executions. At inference time, a lightweight predictor estimates this spectral guidance from the current observation and task context, which is then used to condition robot action generation.

We conduct extensive evaluations of both the representations learned by WING-LAM and the downstream performance of the complete WING framework. Our analysis shows that WING-LAM is substantially less sensitive to observer motion while preserving interaction-related structure, and that low-frequency latent dynamics exhibit stronger correspondence between semantically matched human and robot executions. We further evaluate WING across multiple simulated benchmarks and real-world manipulation tasks. In simulation, WING achieves average success rates of 99.2\% on LIBERO~\citep{libero} and 93.80\% on RoboTwin 2.0~\citep{robotwin2}, and reaches 57.7\% on RoboCasa--GR1~\citep{robocasa} after pre-training on egocentric video data, outperforming strong baselines including $\pi_{0.5}$~\citep{pi_0.5} and LingBot-VA~\citep{lbva}. In real-world experiments, we evaluate WING on four bimanual manipulation tasks under both standard and diverse generalization settings, where it maintains strong performance across variations. Ablation studies further confirm that both interaction-centric latent action learning and low-frequency spectral guidance contribute consistently to downstream policy performance.

Overall, our contributions are threefold:
\begin{itemize}[leftmargin=1.2em, itemsep=0.2em, topsep=0.2em]

\item We show that observer-induced variations can become entangled with interaction-relevant dynamics in egocentric latent actions, and further reveal that low-frequency latent dynamics exhibit stronger correspondence across human and robot interactions.

\item Motivated by these findings, we introduce \textbf{WING}, which combines interaction-centric latent action learning with spectral latent guidance to suppress observer-induced variations and exploit temporal structures that are more consistently shared across embodiments.

\item We validate WING through extensive representation analyses and closed-loop experiments across simulated and real-world manipulation tasks. WING consistently outperforms representative robot policies and alternative egocentric pretraining strategies.

\end{itemize}

%% file: text/related_work_main.tex
\section{Related Work}

Recent vision-language-action (VLA) and world-action models (WAMs) have demonstrated strong generalization by combining large-scale vision-language or video pretraining with robot demonstrations~\citep{rt2,openvla,pi_0,pi_0.5,gr00t,dreamzero,fastwam,lbva}. However, their performance still relies heavily on robot data, which remains expensive to collect and difficult to scale across diverse tasks and environments. To alleviate this limitation, recent studies have explored egocentric human videos as a scalable source of interaction experience. Prior works have explored this direction through representation learning, behavior transfer, and latent-action modeling~\citep{r3m,vip,egomimic,egoscale,lapa,univla,cdlam,motion}. However, ego-centric transitions entangle interaction dynamics with observer-induced camera motion, making it difficult to identify transferable manipulation structure. Existing approaches either model full visual transitions or reduce action-irrelevant variation, while other works exploit camera motion as an active-perception cue~\citep{cdlam,motion,activemimic}, without explicitly separating observer-specific motion from interaction dynamics.

WING addresses this gap by learning interaction-centric latent action guidance for ego-to-robot transfer. It disentangles observer motion from interaction dynamics and extracts low-frequency temporal structure that is more consistently shared between human and robot executions. A more comprehensive discussion of related work is provided in Appendix~\ref{sec:app_rw}.


%% file: text/method.tex
\section{Methods}
\setlength{\textfloatsep}{8pt}
\begin{figure}[t!]
    \centering
    \includegraphics[width=1\linewidth]{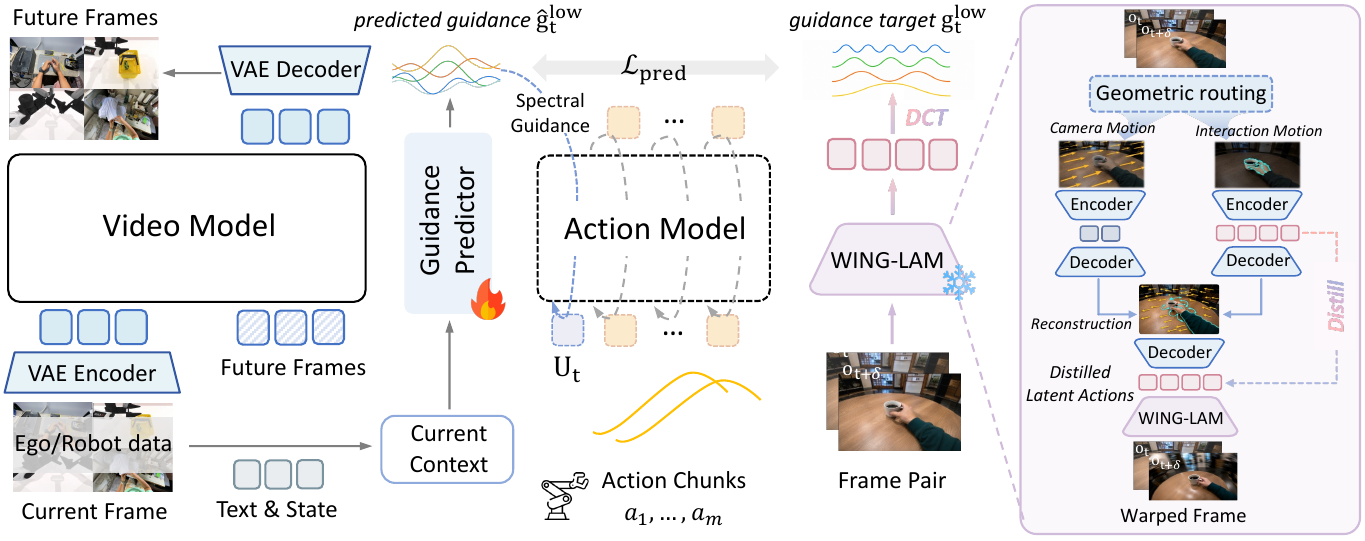}
    \caption{\textbf{Overview of WING.}
WING first learns interaction-centric latent actions that reduce observer-induced variation while preserving interaction dynamics. It then extracts low-frequency temporal structure with stronger correspondence across human and robot executions and predicts this structure from the current context as guidance for robot action generation.}
    \label{fig:main}
\end{figure}


In this section, we present the technical details of \textbf{WING}, including interaction-centric latent action learning, the spectral latent guidance module, and the overall training pipeline. Together, these components extract transferable interaction structure from large-scale egocentric videos and use it as structured latent guidance for robot action learning.

\subsection{Overview}
\label{sec:problem_formulation}

A robot policy $\pi_{\theta}(\mathbf{a}_{t:t+m} \mid \mathbf{l}, \mathbf{o}_{t}, \mathbf{s}_{t})$ maps a language instruction $\mathbf{l}$, visual observation $\mathbf{o}_{t}$, and robot proprioceptive state $\mathbf{s}_{t}$ to an executable action chunk $\mathbf{a}_{t:t+m}$. Our goal is to learn generalizable action knowledge from large-scale egocentric videos through latent action modeling and transfer it to downstream robot manipulation policies. As discussed above, this transfer faces two key challenges: (1) reconstruction-based latent action models entangle interaction-relevant dynamics with task-irrelevant motion, particularly ego-camera motion; and (2) discrepancies in temporal dynamics between human interactions and robot actions hinder direct cross-embodiment transfer.

To address these challenges, \textbf{WING} comprises two key components. First, \textbf{WING-LAM} is an interaction-centric latent action model that encodes egocentric visual transitions $(\mathbf{o}^{\mathrm{ego}}_{t}, \mathbf{o}^{\mathrm{ego}}_{t+\delta})$ into latent actions $\mathbf{z}_{t} = E_{\phi}^{\mathrm{WING\text{-}LAM} }(\mathbf{o}^{\mathrm{ego}}_{t}, \mathbf{o}^{\mathrm{ego}}_{t+\delta})$ while suppressing task-irrelevant motion, thereby encouraging the latent representation to capture interaction-relevant dynamics. Second, the \textbf{Spectral Latent Guidance Module} decomposes latent action trajectories $\mathbf{z}_{t:t+H}$ in the spectral domain and retains the first $K$ low-frequency components, $\mathbf{g}^{\mathrm{low}}_{t} = \left[\operatorname{DCT}(\mathbf{z}_{t:t+H})\right]_{0:K-1}$, as transferable latent guidance. By suppressing rapidly varying high-frequency components, this module isolates slowly varying temporal structures that are more consistently shared across human and robot interactions.

At inference time, a lightweight predictor $P_{\psi}$ estimates the spectral guidance $\widehat{\mathbf{g}}^{\mathrm{low}}_{t} = P_{\psi}(\mathbf{l}, \mathbf{o}_{t}, \mathbf{s}_{t})$, which conditions the world-action model $\pi_{\theta}(\mathbf{a}_{t:t+m} \mid \mathbf{l}, \mathbf{o}_{t}, \mathbf{s}_{t}, \widehat{\mathbf{g}}^{\mathrm{low}}_{t})$ for fine-grained robot action generation. In this way, WING transfers knowledge learned from egocentric human videos as a high-level action prior, while allowing the robot policy to generate embodiment-specific fine-grained controls, thereby bridging large-scale human interaction and downstream robot manipulation.

\input{text/method_lam}
\input{text/method_latent_guidance}

%% file: text/method_lam.tex
\vspace{-5pt}
\subsection{Learning Interaction-Centric Latent Actions}
\label{sec:oilam}

Egocentric videos contain both observer-induced motion
and hand-object interactions. Consequently, a latent action learned to reconstruct the full transition can encode both sources of change~\citep{whatlamlearn}. However, for robot guidance, this calls for the latent to retain interaction dynamics while reducing its
sensitivity to observer motion. \textbf{WING-LAM} therefore uses an explicit motion-routing design that isolates interaction dynamics from observer motion and distills the resulting interaction latent into an RGB-only student.

\noindent\textbf{Camera--Interaction Motion Decomposition.}
Observer motion and hand–object interaction exhibit different spatial motion patterns in egocentric videos: observer motion induces coherent image-wide changes, while interaction motion is more localized. We exploit this distinction to construct motion-aware supervision. Specifically, we decompose each frame transition into a global observer-motion component and a local interaction-motion component. Given two video frames
$\mathbf{o}_{t}$ and $\mathbf{o}_{t+\delta}$, an offline point tracker estimates point correspondences:
$\mathbf p_{t,i},\mathbf p_{t+\delta,i}\in\mathbb R^2$,
where $i$ identifies the same tracked point across frames.
Training-time region masks separate background tracks from hand--object interaction tracks. Using the background tracks, we estimate an invertible
global image warp $W_{t\rightarrow t+\delta}$ as the observer-motion target, capturing the dominant coherent transformation between the two frames.
After compensating this global motion, the remaining displacement of each point represents motion unexplained by the observer movement:
\begin{equation}
    \mathbf r_{t,i}
    =
    \underbrace{
        W_{t\rightarrow t+\delta}^{-1}
        (\mathbf p_{t+\delta,i})
    }_{\text{camera-compensated location}}
    - \quad
    \mathbf p_{t,i},
    \label{eq:oilam_residual}
\end{equation}
where $\mathbf r_{t,i}\in\mathbb R^2$ denotes the residual
displacement expressed in the first frame's coordinates, which serves as the interaction-motion target.

\noindent\textbf{Motion-Routed Latent Learning.}
These decomposed motion signals provide privileged supervision for learning a factorized latent representation. We then train a two-branch teacher to encode visual features
and tracked point coordinates, producing camera and interaction
latents,
$\mathbf z^T_{t,\mathrm{cam}}\in\mathbb R^{D_c}$ and
$\mathbf z^T_{t,\mathrm{int}}\in\mathbb R^{D_i}$,
where the superscript $T$ denotes the teacher and
$D_c,D_i$ are the respective latent dimensions. The camera branch predicts the global warp, while the interaction branch predicts residual displacements, and their decoded motions jointly reconstruct the tracked locations in the second frame. Although motion routing separates the two sources, observer-specific cues may still leak into the interaction latent; we therefore apply camera interventions that preserve interaction dynamics and enforce latent consistency.

\noindent\textbf{Interaction Distillation.}
Although the teacher learns interaction-centric representations using
privileged motion cues, these cues are unavailable during deployment.
We therefore distill the teacher's interaction latent into an RGB-only
student $S_\phi$ that directly extracts latent actions from frame pairs.
For each pair $(\mathbf{o}_{t}^{\mathrm{ego}},
\mathbf{o}_{t+\delta}^{\mathrm{ego}})$, we construct a camera-perturbed view by applying a camera-like global warp
to the second frame, yielding
$\tilde{\mathbf{o}}_{t+\delta}^{\mathrm{ego}}$. The two pairs share the same interaction target $\mathbf z^T_{t,\mathrm{int}}$ while differing only in observer-induced motion.
The student then encodes the original and augmented pairs as:
$\mathbf z_t^{\mathrm{nat}}=S_\phi(\mathbf{o}_{t}^\mathrm{ego},\mathbf{o}_{t+\delta}^\mathrm{ego})$
and
$\mathbf z_t^{\mathrm{aug}}=S_\phi(\mathbf{o}_{t}^\mathrm{ego},\tilde{\mathbf{o}}_{t+\delta}^\mathrm{ego})$,
where both outputs lie in $\mathbb R^{D_i}$, and
$\mathrm{nat}$ and $\mathrm{aug}$ denote the original
and augmented inputs, respectively. 
We optimize the student using:
\begin{equation}
    \mathcal L_{\mathrm{distill}}
    =
    \frac{1}{2D_i}
    \sum_{v\in\{\mathrm{nat},\mathrm{aug}\}}
    \left\|
        \mathbf z_t^v-\mathbf z^T_{t,\mathrm{int}}
    \right\|_2^2
    +
    \frac{\lambda_{\mathrm{view}}}{D_i}
    \left\|
        \mathbf z_t^{\mathrm{nat}}
        -
        \mathbf z_t^{\mathrm{aug}}
    \right\|_2^2,
    \label{eq:oilam_distill}
\end{equation}
where $v$ indexes the two inputs and
$\lambda_{\mathrm{view}}$ weights their prediction consistency.
The first term transfers the teacher's interaction
representation; the second encourages the student to
retain that representation under camera-like image warps. Details of implementation
are provided in Appendix~\ref{sec:app_oilam_training}.

Overall, the above design separates observer-induced variation from interaction dynamics, yielding interaction-centric latent actions that capture transferable manipulation structure and serve as the foundation for the subsequent spectral guidance module.

%% file: text/method_latent_guidance.tex
\subsection{Spectral Latent Action Guidance}
\label{sec:spectral_guidance}
Although WING-LAM extracts interaction-centric latent actions, semantically
matched human and robot executions can differ in temporal dynamics. We find that low-frequency latent
components exhibit stronger cross-embodiment correspondence than full
trajectories, and therefore retain these components as
spectral guidance for robot action generation.

\noindent\textbf{Spectral Latent Representation.}
Given visual observations sampled at a fixed interval $\Delta t$, WING-LAM encodes $H$ successive transitions into
$
\mathbf{Z}_t=[\mathbf{z}_t,\ldots,\mathbf{z}_{t+H-1}]^\top
\in\mathbb{R}^{H\times D_i},
$
where
$
\mathbf{z}_{t+h}=S_\phi(\mathbf{o}_{t+h},\mathbf{o}_{t+h+1}).
$
An orthonormal DCT is then applied along the temporal dimension, and the lowest $K$ components are retained as the guidance target
$\mathbf g_t^{\mathrm{low}}$.
This truncation affects only the latent guidance, not the robot actions.

\noindent\textbf{Predicting Low-Frequency Guidance.} Since $\mathbf g_t^{\mathrm{low}}$ requires future
observations, we train a guidance predictor $P_\psi$ to estimate it from
the current observation and instruction, together with
proprioception when available.
Writing these inputs as $\mathbf c_t$, we minimize
the mean squared prediction error:
\begin{equation}
    \widehat{\mathbf g}_t^{\mathrm{low}}
    =
    P_\psi(\mathbf c_t),
    \qquad
    \mathcal L_{\mathrm{pred}}
    =
    \frac{1}{KD_i}
    \left\|
        \widehat{\mathbf g}_t^{\mathrm{low}}
        -
        \mathbf g_t^{\mathrm{low}}
    \right\|_F^2,
    \label{eq:low_frequency_prediction}
\end{equation}
where $\|\cdot\|_F$ denotes the Frobenius norm.
During robot training, we jointly optimize this loss
and the action-generation objective. Future frames are only used to construct training
targets.

\noindent\textbf{Conditioning Robot Actions.}
We convert the predicted guidance
$\widehat{\mathbf g}_t^{\mathrm{low}}$ into
$\mathbf U_t\in\mathbb R^{K\times D_{\mathrm{hid}}}$,
where $D_{\mathrm{hid}}$ is the action model's hidden dimension.
Each frequency coefficient vector is projected
to this dimension, normalized with LayerNorm, and given
a learned embedding identifying its frequency mode.
The action model attends to these guidance tokens through
a gated residual update:
\begin{equation}
    \mathbf H_a^{(l)}
    \leftarrow
    \mathbf H_a^{(l)}
    +
    \gamma_l
    \operatorname{Attn}\!\left(
        Q=\operatorname{LayerNorm}(\mathbf H_a^{(l)}),
        K=\mathbf U_t,
        V=\mathbf U_t
    \right),
    \label{eq:latent_action_cross_attention}
\end{equation}
where $\mathbf H_a^{(l)}$ denotes the action hidden
sequence at layer $l$, and the learned scalar $\gamma_l$
controls the magnitude of latent guidance.

\subsection{Training Strategy}
\label{sec:wam_training}

We train the world action model (WAM) in two stages: pretraining on egocentric
videos, followed by robot policy learning from robot demonstrations.



\noindent\textbf{Egocentric Pretraining.}
Using our curated egocentric data, we adopt Wan2.2~\citep{wan} and train the video expert for future-frame generation and the guidance predictor $P_\psi$ for low-frequency spectral guidance prediction.
The two pathways learn visual and interaction dynamics through the generation objective and latent guidance prediction, respectively.
Full details of egocentric pretraining are provided in Appendix~\ref{sec:app_pretraining_details}.

\noindent\textbf{Robot Policy Learning.}
After pretraining, we transfer the pretrained visual backbone and guidance predictor to the robot policy and jointly train them with the action expert on robot demonstrations.
The guidance predictor is supervised on robot data, while the WING-LAM encoder remains frozen.
The predicted guidance conditions downstream robot action generation during both training and inference.
This stage adapts the interaction dynamics learned from egocentric videos to robot control, with robot demonstrations providing direct supervision for executable actions.

%% file: text/6_experiments.tex
\section{Experiments}
\label{sec:experiments}
In this section, we evaluate WING from both representation and policy perspectives. 
Specifically, we aim to answer the following four questions:
\begin{enumerate}

\item [\textbf{(Q1)}] What structures in egocentric latent actions reveal transferable interaction knowledge?

\item [\textbf{(Q2)}] Can WING effectively transfer latent action knowledge to improve robot performance?

\item [\textbf{(Q3)}] How do interaction-centric latent actions and spectral latent action guidance contribute?

\item [\textbf{(Q4)}] How effective is WING's pretraining strategy for downstream robot policy learning?
\end{enumerate}

\input{text/6_1_evaluation_setup}

\input{text/6_2_experiments_lam}

\input{text/6_2_robot_policy_results}
\input{text/6_3_ablation_studies}

%% file: text/6_1_evaluation_setup.tex
\subsection{Experimental Setup}

\label{sec:experimental_setup}
We first evaluate whether WING-LAM reduces observer-related variation while retaining interaction-relevant semantics, and then investigate which temporal structures in latent trajectories exhibit stronger cross-embodiment correspondence between human and robot executions. We compare WING-LAM with representative LAM baselines using motion readout, camera perturbations, and action classification, and analyze semantic similarity between matched and mismatched human--robot trajectories across frequency bands. Evaluation details are provided in Appendix~\ref{sec:app_observer_motion_diagnosis} and Appendix~\ref{sec:app_pair_construction}. We further evaluate WING on LIBERO~\citep{libero}, RoboTwin~2.0~\citep{robotwin2},
RoboCasa--GR1~\citep{robocasa}, and a real-world dual-arm platform.
We compare against representative vision-language-action, latent-action, and
world-action modeling methods, reporting success rate in simulation and both
task completion and progress in the real world.
Full evaluation protocols are provided in
Appendix~\ref{sec:app_evaluation_details}.

%% file: text/6_2_experiments_lam.tex
\newpage
\input{img/6_observer_debiasing_evaluation}
\subsection{\textbf{(Q1):} Characterizing Egocentric Latent Structure For Robot Learning.}
\input{tab/6_3_robocasa_summary}
\label{sec:observer_debiasing_evaluation}

We first examine whether WING-LAM suppresses observer-related variation while preserving interaction semantics. As shown in Figure~\ref{fig:lam}(a), both its teacher and student retain strong interaction-motion information but substantially less camera-motion information than DreamDojo~\citep{dreamdojo}. Notably, although CD-LAM~\citep{cdlam} is explicitly designed to reduce action-irrelevant variation, it still retains more camera-motion information under our evaluation protocol. This selectivity persists under direct perturbation: as shown in Figure~\ref{fig:lam}(b), the student exhibits the lowest latent sensitivity under increasing camera motion. Importantly, this robustness does not come at the expense of action semantics, as WING-LAM achieves the highest average action-classification accuracy on LARYBench~\citep{larybench} (Figure~\ref{fig:lam}(c)). We then ask which temporal components of these interaction-centric trajectories are most consistent across human and robot executions. As shown in Figure~\ref{fig:lam}(d), low-frequency components yield a substantially larger separation between semantically matched and mismatched pairs than either the full trajectory or the high-frequency components, indicating that shared interaction structure is concentrated more strongly in low-frequency dynamics and motivating their use as guidance for action generation.

%% file: img/6_observer_debiasing_evaluation.tex
\begin{figure}[htbp]
    \centering
    \includegraphics[width=1\linewidth]{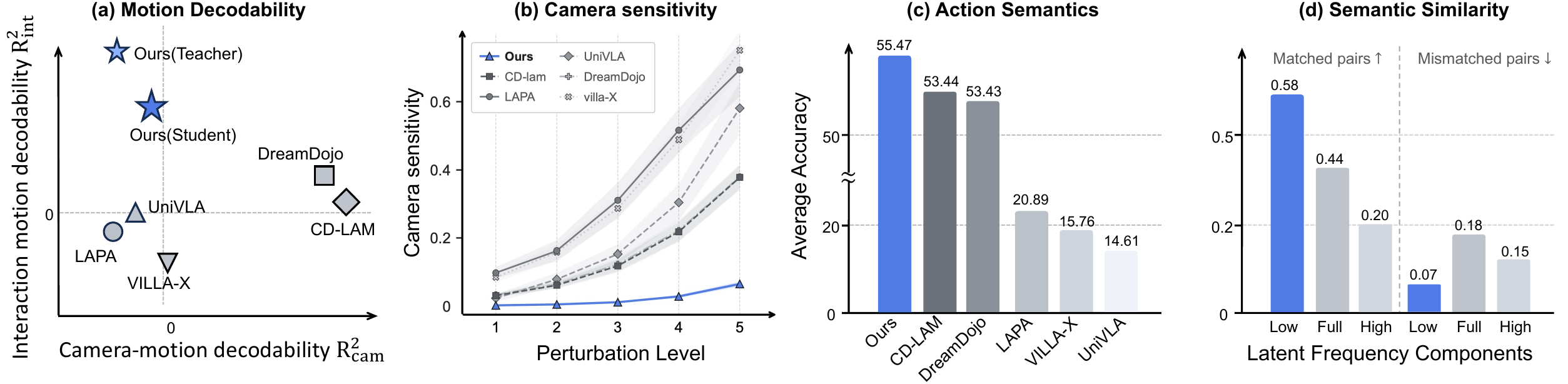}
    \vspace{-1.2em}
    \caption{
\textbf{Characterizing interaction structure in egocentric latent actions.}
\textbf{(a)} Camera-motion and interaction-motion decodability measured by
MLP-probe $R^2$.
\textbf{(b)} Normalized latent sensitivity under progressively stronger
camera-motion perturbations.
\textbf{(c)} Average action-category classification accuracy using frozen
latent representations across the LARYBench~\citep{larybench} egocentric datasets.
\textbf{(d)} Cross-embodiment semantic similarity of low-frequency,
high-frequency, and full latent-action trajectories for semantically matched
and mismatched human--robot pairs.
}
    \label{fig:lam}
\end{figure}

%% file: tab/6_3_robocasa_summary.tex
\begin{wraptable}[13]{r}{0.3\columnwidth}
\vspace{-1em}
\centering

\caption{\textbf{RoboCasa--GR1 Tabletop.}
Best in \textbf{bold}; second-best \underline{underlined}. 
$\dagger$ denotes our reproduced results.}
\label{tab:robocasa_gr1_summary}
\vspace{-1em}
\small
\setlength{\tabcolsep}{10pt}
\renewcommand{\arraystretch}{1}

\begin{tabular}{lc}
\toprule
Method & Avg. SR \\
\midrule
GR00T-N1.6 & 47.6 \\
StarVLA-OFT   & 48.8 \\
DiT4DiT    & 50.8 \\
LDA-1B     & \underline{55.4} \\
Being-H0.7    & 49.2  \\
FastWAM$^\dagger$ & 51.9 \\
\midrule
\rowcolor{gray!12}
Ours        & \textbf{57.7} \\
\bottomrule
\end{tabular}

\end{wraptable}

%% file: text/6_2_robot_policy_results.tex
\subsection{\textbf{(Q2):} Robot Policy Performance}
\label{sec:robot_policy_results}

\input{tab/6_3_libero_robotwin_main}
\vspace{1em}
\paragraph{Results on Simulation.}
\WFclear

Across single-arm, bimanual, and humanoid manipulation benchmarks,
WING achieves average success rates of $99.20\%$ on LIBERO
(Table~\ref{tab:libero_robotwin_main}(a)), $93.80\%$ on RoboTwin~2.0
(Table~\ref{tab:libero_robotwin_main}(b)), and $57.7\%$ on
RoboCasa--GR1 (Table~\ref{tab:robocasa_gr1_summary}), outperforming
the strongest baselines on all three. Notably, its advantage extends
to LIBERO's long-horizon suite and RoboTwin's randomized scenes,
suggesting that the gains are not confined to shorter tasks or clean
visual conditions. Together, these results demonstrate the effectiveness of WING for robot
manipulation across diverse embodiments and task settings. Full results for RoboTwin~2.0 and RoboCasa
are provided in Appendices~\ref{sec:app_robotwin_full_results}
and~\ref{sec:app_robocasa_full_results}, respectively.

\input{img/real_results}

\WFclear
\paragraph{Results on the real-world robot.}
As illustrated in Figure~\ref{fig:real_results}, WING achieves the
highest average success rate of 75.0\% under the standard setting.
Across both standard and generalization settings, WING consistently
outperforms FastWAM and GR00T-N1.7 in average success rate and
progress score. Under generalization settings, WING achieves
comparable average success to the strong baseline \(\pi_{0.5}\),
while attaining a slightly higher average progress score. These results demonstrate the effectiveness
of WING in real-world manipulation and its ability to maintain strong
performance under diverse generalization settings.


%% file: tab/6_3_libero_robotwin_main.tex
\begin{table}[t]
\centering
\caption{\textbf{Simulation results on LIBERO and RoboTwin~2.0.}
Success rate (\%). The best and second-best results in each column are shown in
\textbf{bold} and \underline{underline}, respectively.}
\label{tab:libero_robotwin_main}

\fontsize{6.8}{8.5}\selectfont
\renewcommand{\arraystretch}{1.05}

\begin{minipage}[t]{0.565\textwidth}
\centering
\textbf{(a) LIBERO}\par\vspace{0.3em}

\setlength{\tabcolsep}{2.5pt}

\begin{tabularx}{\linewidth}{
    @{}
    l
    *{5}{>{\centering\arraybackslash}X}
    @{}
}
\toprule[1.1pt]
Method & Spatial & Object & Goal & Long & Avg. \\
\midrule


\multicolumn{6}{c}{\textit{VLAs}} \\
\midrule

$\pi_0$~\citep{pi_0}
    & 98.00
    & 96.80
    & 94.40
    & 88.40
    & 94.40 \\

$\pi_{0.5}$~\citep{pi_0.5}
    & 98.80
    & 98.20
    & 98.00
    & 92.40
    & 96.90 \\

\midrule

\multicolumn{6}{c}{\textit{Latent-Action Based Methods}} \\
\midrule

LAPA~\citep{lapa}
    & 73.80
    & 74.60
    & 58.80
    & 55.40
    & 65.70 \\

UniVLA~\citep{univla}
    & 96.50
    & 96.80
    & 95.60
    & 92.00
    & 95.20 \\

LaWAM~\citep{lawam}
    & \textbf{99.40}
    & 99.60
    & \underline{98.40}
    & 97.00
    & \underline{98.60} \\

VLA-JEPA~\citep{vlajepa}
    & 96.20
    & 99.60
    & 97.20
    & 95.80
    & 97.20 \\

\midrule

\multicolumn{6}{c}{\textit{WAMs}} \\
\midrule

Motus~\citep{motus}
    & 96.80
    & \underline{99.80}
    & 96.60
    & 97.60
    & 97.70 \\

LingBot-VA~\citep{lbva}
    & 98.50
    & 99.60
    & 97.20
    & \underline{98.50}
    & 98.50 \\

Fast-WAM~\citep{fastwam}
    & 98.20
    & \textbf{100.00}
    & 97.00
    & 95.20
    & 97.60 \\

\midrule

\rowcolor{gray!12}
Ours
    & \underline{99.00}
    & \textbf{100.00}
    & \textbf{98.80}
    & \textbf{99.00}
    & \textbf{99.20} \\


\bottomrule[1.1pt]
\end{tabularx}

\end{minipage}
\hfill
\begin{minipage}[t]{0.415\textwidth}
\centering
\textbf{(b) RoboTwin~2.0}\par\vspace{0.3em}

\setlength{\tabcolsep}{3.2pt}

\begin{tabularx}{\linewidth}{
    @{}
    l
    *{3}{>{\centering\arraybackslash}X}
    @{}
}
\toprule [1.1pt]
Method & Clean & Rand. & Avg. \\
\midrule


\multicolumn{4}{c}{\textit{VLAs}} \\
\midrule

$\pi_0$~\citep{pi_0}
    & 65.92
    & 58.40
    & 62.16 \\

$\pi_{0.5}$~\citep{pi_0.5}
    & 82.74
    & 76.76
    & 79.75 \\

X-VLA~\citep{xvla}
    & 72.90
    & 72.80
    & 72.85 \\

LingBot-VLA~\citep{lbvla}
    & 88.56
    & 86.68
    & 87.62 \\

\midrule

\multicolumn{4}{c}{\textit{Latent-Action Based Methods}} \\
\midrule

LaWAM~\citep{lawam}
    & 92.64
    & 89.80
    & 91.22 \\

\midrule

\multicolumn{4}{c}{\textit{WAMs}} \\
\midrule

Motus~\citep{motus}
    & 88.66
    & 87.02
    & 87.84 \\

Fast-WAM~\citep{fastwam}
    & 91.88
    & \underline{91.78}
    & 91.83 \\

GigaWorld-Policy~\citep{giga}
    & 86.36
    & 85.04
    & 85.70 \\

LingBot-VA~\citep{lbva}
    & \underline{92.90}
    & 91.50
    & \underline{92.20} \\

\midrule

\rowcolor{gray!12}
Ours
    & \textbf{94.56}
    & \textbf{93.04}
    & \textbf{93.80} \\
    
\bottomrule [1.1pt]
\end{tabularx}

\end{minipage}
\vspace{2em}
\end{table}

%% file: img/real_results.tex
\begin{figure}
    \centering
    \includegraphics[width=1\linewidth]{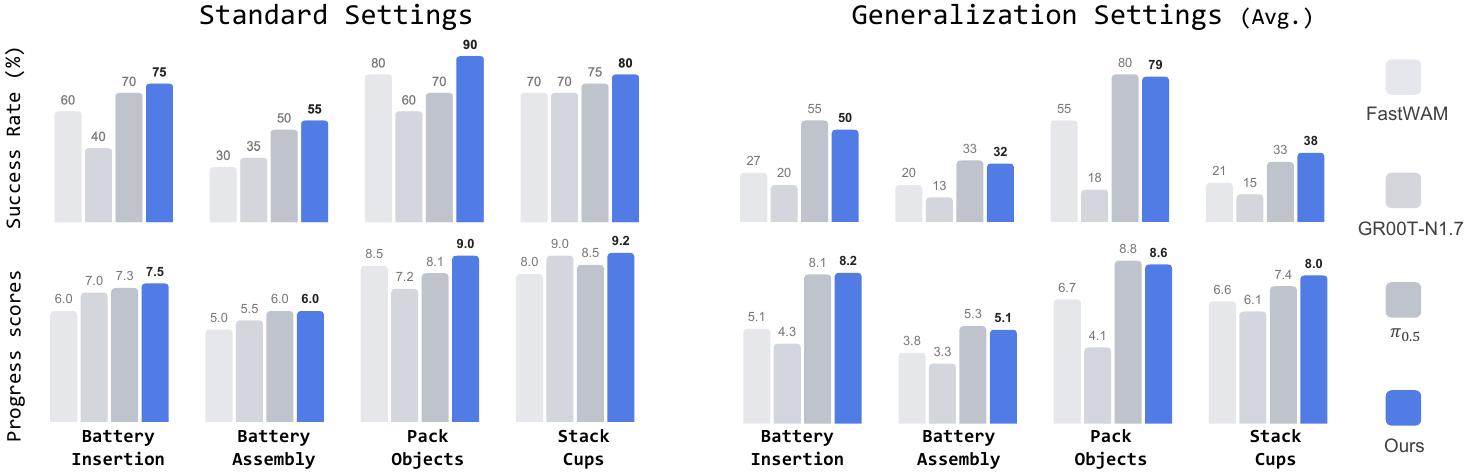}
    \vspace{-0.5em}
    \caption{\textbf{Real-world task evaluation results.} We evaluate different robot foundation models on both standard and generalization settings across four real-world manipulation tasks. Our method achieves consistently strong performance, achieving the strongest standard-setting performance and remaining competitive under generalization.}
    \label{fig:real_results}
    \vspace{1em}
\end{figure}

%% file: text/6_3_ablation_studies.tex
\newpage
\subsection{\textbf{(Q3):} Ablation Studies on WING-LAM and Spectral Guidance.}
\label{sec:ablation_studies}

\input{tab/6_5_controlled_ablation}

\paragraph{Ablation on observer debiasing and spectral guidance.}
Table~\ref{tab:module_ablation} evaluates the contribution of both
components under a controlled setting \textit{without ego-pretraining of
the WAM}. When WING-LAM is ablated, we replace it with
the pretrained DreamDojo LAM while keeping the spectral
guidance unchanged. When DCT module variant is ablated, we retain the
same latent encoder but bypass DCT and use the full time-domain latent
trajectory directly as guidance. Removing either component reduces average performance
across both LIBERO and real-world evaluations. The gains from the DCT variant show that frequency-domain decomposition
and low-frequency retention yield more effective guidance by emphasizing
slowly varying interaction dynamics. The additional gains from WING-LAM
indicate that latent actions with reduced observer-induced variation
provide more useful guidance for robot control. This advantage of
WING-LAM becomes even clearer in the ego-pretraining analysis below,
where it leads to larger improvements.

\vspace{-5pt}
\paragraph{Ablation on the number of DCT components.}
We further investigate how much spectral information is required for effective
latent guidance by varying the number of retained DCT coefficients
$K\in\{2,4,8\}$, where $K=8$ corresponds to the full spectrum. As shown in
Figure~\ref{fig:ablation}(a), $K=4$ consistently performs best across all suites. Retaining only two coefficients is overly
restrictive and discards useful temporal dynamics, while the full spectrum introduces additional high-frequency
components without improving downstream control. These results suggest that
effective latent guidance requires a sufficient, but not excessive, spectral
bandwidth, with $K=4$ providing the best balance between preserving
control-relevant temporal structure and filtering variation that is less consistently shared across human and robot executions.

\subsection{\textbf{(Q4):} Analysis of Ego Pre-training Strategies.}
\input{img/ablation}
Having established WING's downstream performance, we further examine
whether its interaction-centric spectral guidance also provides an
effective pre-training strategy for WAMs from egocentric videos.
We compare pre-training methods under the same 200h budget which differ in how video-derived information supervises the WAM.
In Figure~\ref{fig:ablation}(b), \textbf{Vid-P.} performs video-only pre-training, isolating the
effect of exposure to egocentric videos themselves. \textbf{HP-P.}
uses hand-pose labels as supervision. \textbf{LA-P.} follows
LAPA~\citep{lapa} which uses latent actions as direct action-prediction
targets. In contrast, \textbf{Ours} predicts spectral latent guidance
to condition action generation. To isolate the effect of suppressing
observer-induced variation, \textbf{w/o OD.} denotes variants in which
WING-LAM is replaced with DreamDojo LAM.
Implementation details are provided in
Appendix~\ref{sec:app_pretraining_details}. 

As shown in Figure~\ref{fig:ablation}(b), video-only pre-training consistently underperforms our method, with
particularly clear gaps on the real-world standard and generalization
settings, indicating that the gains of egocentric pre-training cannot be explained by video exposure alone.
Replacing WING-LAM with a non-debiased latent-action model further degrades
performance in most settings, supporting the importance of suppressing
observer-induced variation.
Direct latent-action prediction also performs worse, possibly reflecting
the gap between visual transition representations and executable robot
actions.
Together, these results suggest that WING is most effective when
ego-derived latent actions serve as auxiliary guidance, rather than as
direct action targets, while the robot policy retains its original
action-prediction objective. We further analyze the effectiveness and efficiency of pre-training in Appendix~\ref{sec:app_pre_eff} and \ref{sec:app_data_eff} by examining downstream validation loss after different amounts of egocentric pretraining and downstream performance under varying amounts of robot training data. These analyses show that WING not only learns effective representations during pre-training, but also retains stronger downstream performance as robot supervision becomes more limited.

%% file: tab/6_5_controlled_ablation.tex
\begin{wraptable}[10]{l}{0.5\textwidth}

\centering

\caption{
\textbf{Ablation on LIBERO and Real-World.}
We report SR on these tasks.
}
\label{tab:module_ablation}

\vspace{-0.5em}

\small
\setlength{\tabcolsep}{3.5pt}
\renewcommand{\arraystretch}{1}

\resizebox{\linewidth}{!}{%
\begin{tabular}{cc|cccc|cc}
\toprule

\multicolumn{2}{c|}{Method}
& \multicolumn{4}{c|}{LIBERO}
& \multicolumn{2}{c}{Real-World} \\

WING-LAM & DCT
& Spatial & Object & Goal & Long
& Base & Gen. \\

\midrule
\ding{55} & \ding{55}
&  99.2 &  97.2 &  96.2 &  92.6
& 65.0 & 36.0 

\\
\ding{51} & \ding{55}
& 98.0 & 99.2 & 95.6 & 94.2
& 67.5 & 40.3 \\

\ding{55} & \ding{51}
& 99.2 & 98.2 & 98.0 & 96.2
& 71.3 & 43.0 \\

\rowcolor{gray!12}
\ding{51} & \ding{51}
& \textbf{99.6}
& \textbf{100.0}
& \textbf{99.2}
& \textbf{96.2}
& \textbf{72.5} & \textbf{45.0} \\

\bottomrule
\end{tabular}%
}

\end{wraptable}

%% file: img/ablation.tex



\begin{figure}
    \centering
    \includegraphics[width=1\linewidth]{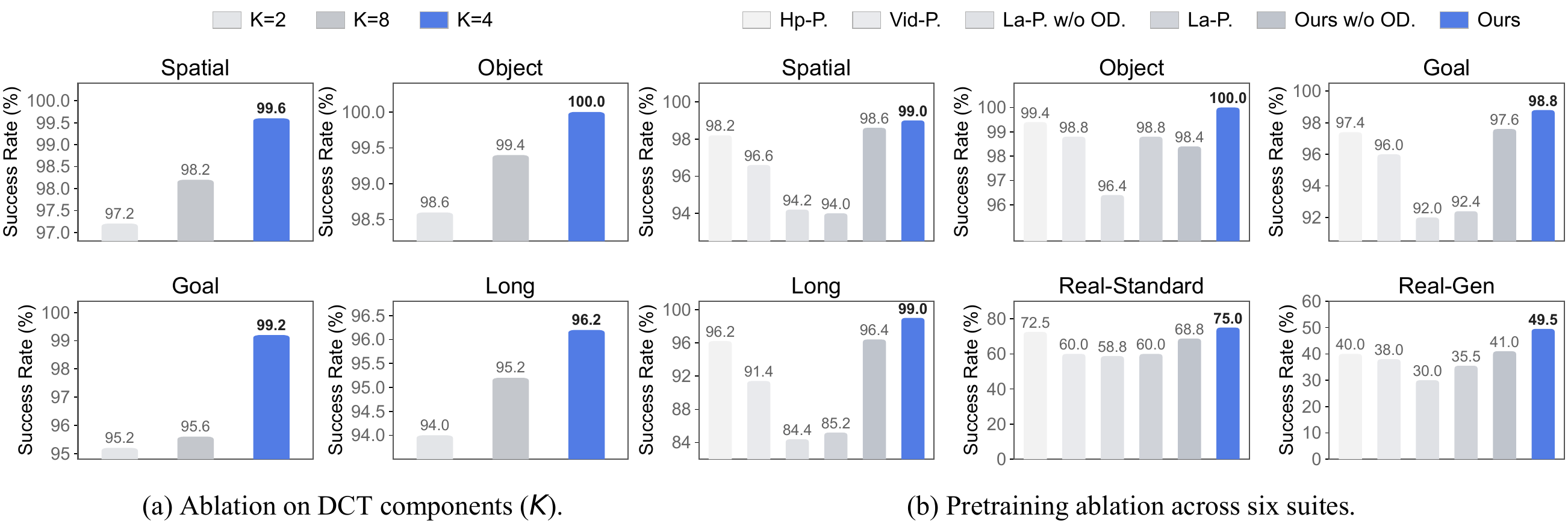}
    \caption{\textbf{Analysis of spectral guidance and egocentric pretraining strategies.}
     \textbf{(a)}  Effect of the number of retained DCT components \(K\) on latent guidance.
     \textbf{(b)} Comparison of different egocentric pretraining strategies including
video-only pretraining, hand-pose supervision, direct latent-action pretraining, and spectral latent guidance, with additional variants to
evaluate the effect of interaction-centric debiasing, across LIBERO and real-world tasks.
     }
    \label{fig:ablation}
\vspace{1em}
\end{figure}

%% file: text/7_conclusion.tex
\vspace{-0.5em}
\section{Limitations and Future Work}
\vspace{-0.5em}
\label{sec:limitation}
Although WING provides an effective way to exploit egocentric videos for robot
learning, several limitations remain. First, constructing the WING-LAM teacher
relies on external motion cues and an image-space camera--interaction
decomposition. The dominant observer motion is approximated by a single global
affine warp, which can become less accurate under strong parallax, large depth
variation, or pronounced 3D camera motion, potentially introducing noise into
the interaction supervision. Second, latent-action pretraining and spectral
guidance add training cost and pipeline complexity. Future work could explore
more self-contained, depth-aware motion modeling and more efficient guidance
mechanisms.

\vspace{-0.5em}
\section{Conclusion}
\vspace{-0.5em}
\label{sec:conclusion}
In this work, we study how to effectively transfer interaction knowledge from egocentric human videos to robot learning.
We show that transferable latent actions require suppressing observer-induced variation and capturing temporal structures that remain consistent across human and robot executions.
Based on these findings, we propose \textbf{WING}, a framework that
learns and transfers interaction-centric latent knowledge from egocentric
videos. WING suppresses observer-induced variations through
interaction-centric latent action learning and leverages low-frequency latent dynamics that are more consistent across embodiments through spectral latent guidance. Extensive experiments across simulation and real-world manipulation
benchmarks demonstrate that WING effectively transfers embodied
interaction knowledge from egocentric videos to robot control, achieving
strong performance and generalization across diverse tasks. Our results
suggest that scaling robot learning with human experience requires not
simply modeling all observed dynamics, but identifying interaction
structures that remain consistent across observers and embodiments.




%% file: text/appendix/theory.tex
\section{Diagnostic Analysis}
\label{sec:app_prop}

\subsection{Observer Variation in a Reconstruction Bottleneck}
\label{app:observer_reconstruction}

We analyze when reconstruction-based latent action learning
retains observer-induced variation. Following the observation
decomposition in \citep{whylatentfail} and the linear LAM
analysis in \citep{whatlamlearn}, we separate interaction
and observer contributions to a visual transition.

\paragraph{Setup.}
Let $\mathbf m$ denote the hand--object interaction state
and $\mathbf c$ the observer state.
The observation map $h$ generates visual representations
$\mathbf o=h(\mathbf m,\mathbf c)$ and
$\mathbf o'=h(\mathbf m',\mathbf c')$,
both in $\mathbb R^{d_o}$.
Using the intermediate observation $h(\mathbf m',\mathbf c)$,
we decompose the observation change as
\begin{equation}
    \begin{aligned}
        \Delta\mathbf o_{\mathrm{int}}
        &=h(\mathbf m',\mathbf c)-h(\mathbf m,\mathbf c),\\
        \Delta\mathbf o_{\mathrm{obs}}
        &=h(\mathbf m',\mathbf c')-h(\mathbf m',\mathbf c),\\
        \Delta\mathbf o
        &=\mathbf o'-\mathbf o
        =\Delta\mathbf o_{\mathrm{int}}
         +\Delta\mathbf o_{\mathrm{obs}}.
    \end{aligned}
    \label{eq:app_obs_source_changes}
\end{equation}
The first component changes the interaction state at a fixed
observer state; the second changes the observer state at the
final interaction state. This identity holds for nonlinear $h$
and does not require the two components to be independent.

For the linear analysis below, we subtract the respective
means from all observation and transition variables and
reuse the same notation. The additive identity above is
preserved. All expectations are taken over the distribution
of observation pairs, with finite second moments assumed.

Consider a linear LAM with a $d_z$-dimensional bottleneck,
where $1\le d_z<d_o$:
\begin{equation}
    \mathbf z=\mathbf C\mathbf o+\mathbf D\mathbf o',
    \qquad
    \widehat{\mathbf o}'=\mathbf A\mathbf o+\mathbf B\mathbf z.
    \label{eq:app_obs_linear_lam}
\end{equation}
Here,
$\mathbf A\in\mathbb R^{d_o\times d_o}$,
$\mathbf B\in\mathbb R^{d_o\times d_z}$, and
$\mathbf C,\mathbf D\in\mathbb R^{d_z\times d_o}$
are unconstrained trainable matrices.
They minimize the expected squared reconstruction error,
\begin{equation}
    \mathcal L
    =
    \mathbb E\!\left[
        \|\widehat{\mathbf o}'-\mathbf o'\|_2^2
    \right].
    \label{eq:app_obs_reconstruction_loss}
\end{equation}
Following \citep{whatlamlearn}, we assume that the current
observation and the total observation change are uncorrelated:
\begin{equation}
    \mathbb E[\mathbf o\,\Delta\mathbf o^\top]=\mathbf0.
    \label{eq:app_obs_decorrelation}
\end{equation}
This is an additional condition of the linear analysis,
not a consequence of centering.

\begin{takeaway}
\begin{proposition}[Observer variation retained by reconstruction]
\label{prop:observer_reconstruction}
For the linear LAM above under
Eq.~\ref{eq:app_obs_decorrelation}, let
\[
    \boldsymbol\Sigma
    =
    \operatorname{Cov}(\Delta\mathbf o)
    =
    \mathbb E[\Delta\mathbf o\,\Delta\mathbf o^\top]
\]
be the covariance matrix of the centered observation change.
Let $\mathbf u_1,\ldots,\mathbf u_{d_o}$ be its orthonormal
eigenvectors, with eigenvalues
$\lambda_1\ge\cdots\ge\lambda_{d_o}\ge0$.
Each eigenvalue is the variance along its corresponding
direction:
$\lambda_i=\mathbb E[(\mathbf u_i^\top\Delta\mathbf o)^2]$.

Define
$\mathbf U=[\mathbf u_1,\ldots,\mathbf u_{d_z}]$
and $\mathbf P=\mathbf U\mathbf U^\top$, the orthogonal
projector onto the $d_z$ leading eigenvectors.
One globally optimal reconstruction and its minimum loss are
\begin{equation}
    \widehat{\mathbf o}'=\mathbf o+\mathbf P\Delta\mathbf o,
    \qquad
    \mathcal L^\star=\sum_{i=d_z+1}^{d_o}\lambda_i.
    \label{eq:app_obs_optimal_projector}
\end{equation}
Let
$\boldsymbol\Sigma_{\mathrm{obs}}
=\operatorname{Cov}(\Delta\mathbf o_{\mathrm{obs}})$
be the covariance of the observer component.
Its contribution to the reconstructed transition satisfies
\begin{equation}
    \mathbb E\!\left[
        \|\mathbf P\Delta\mathbf o_{\mathrm{obs}}\|_2^2
    \right]
    =
    \operatorname{tr}
    (\mathbf P\boldsymbol\Sigma_{\mathrm{obs}}),
    \label{eq:app_obs_retained_energy}
\end{equation}
where $\operatorname{tr}$ denotes the matrix trace.
Thus observer variation is retained whenever it has nonzero
variance in the selected principal subspace.
\end{proposition}
\end{takeaway}

\begin{proof}
Substituting $\mathbf o'=\mathbf o+\Delta\mathbf o$
into Eq.~\ref{eq:app_obs_linear_lam} gives
\[
    \widehat{\mathbf o}'-\mathbf o'
    =
    [\mathbf A+\mathbf B(\mathbf C+\mathbf D)-\mathbf I]
    \mathbf o
    +
    (\mathbf B\mathbf D-\mathbf I)\Delta\mathbf o,
\]
where $\mathbf I$ is the $d_o\times d_o$ identity matrix.
The expected cross term vanishes by
Eq.~\ref{eq:app_obs_decorrelation}, so
\begin{equation}
    \begin{aligned}
        \mathcal L
        ={}&
        \mathbb E\!\left[
            \|[\mathbf A+\mathbf B(\mathbf C+\mathbf D)
            -\mathbf I]\mathbf o\|_2^2
        \right]\\
        &+
        \mathbb E\!\left[
            \|(\mathbf B\mathbf D-\mathbf I)
            \Delta\mathbf o\|_2^2
        \right].
    \end{aligned}
    \label{eq:app_obs_error_decomposition}
\end{equation}
For any $\mathbf B,\mathbf C,\mathbf D$, choosing
$\mathbf A=\mathbf I-\mathbf B(\mathbf C+\mathbf D)$
eliminates the first term.

The product $\mathbf B\mathbf D$ can represent any matrix
of rank at most $d_z$.
The remaining objective therefore seeks a linear
reconstruction of $\Delta\mathbf o$ through at most
$d_z$ dimensions.
By PCA optimality \citep{whatlamlearn}, it is minimized
by projecting onto the eigenvectors of
$\boldsymbol\Sigma$ with the $d_z$ largest eigenvalues.

To obtain the minimum loss, write
$\mathbf P=\sum_{i=1}^{d_z}\mathbf u_i\mathbf u_i^\top$.
The reconstruction discards the remaining orthogonal
components:
\begin{equation}
    \begin{aligned}
        \Delta\mathbf o-\mathbf P\Delta\mathbf o
        &=
        \sum_{i=d_z+1}^{d_o}
        (\mathbf u_i^\top\Delta\mathbf o)\mathbf u_i,\\
        \mathcal L^\star
        &=
        \mathbb E\!\left[
            \|\Delta\mathbf o-\mathbf P\Delta\mathbf o\|_2^2
        \right]\\
        &=
        \sum_{i=d_z+1}^{d_o}
        \mathbb E[(\mathbf u_i^\top\Delta\mathbf o)^2]
        =
        \sum_{i=d_z+1}^{d_o}\lambda_i.
    \end{aligned}
    \label{eq:app_obs_minimum_loss}
\end{equation}
The squared components add because the eigenvectors are
orthonormal. The minimum loss is therefore the total
variance along the discarded directions.
This optimum is attained in the original LAM by
$\mathbf A=\mathbf I$,
$\mathbf B=\mathbf U$,
$\mathbf C=-\mathbf U^\top$, and
$\mathbf D=\mathbf U^\top$.

Under this reconstruction, the predicted transition is
$\mathbf P\Delta\mathbf o_{\mathrm{int}}
+\mathbf P\Delta\mathbf o_{\mathrm{obs}}$.
Since the observer component is centered,
\begin{equation}
    \begin{aligned}
        \mathbb E\!\left[
            \|\mathbf P\Delta\mathbf o_{\mathrm{obs}}\|_2^2
        \right]
        &=
        \operatorname{tr}
        (\mathbf P\boldsymbol\Sigma_{\mathrm{obs}}\mathbf P^\top)\\
        &=
        \operatorname{tr}
        (\mathbf P\boldsymbol\Sigma_{\mathrm{obs}})\\
        &=
        \sum_{i=1}^{d_z}
        \mathbf u_i^\top
        \boldsymbol\Sigma_{\mathrm{obs}}\mathbf u_i.
    \end{aligned}
    \label{eq:app_obs_energy_derivation}
\end{equation}
The second equality uses
$\mathbf P^\top=\mathbf P$, $\mathbf P^2=\mathbf P$,
and the cyclic property of the trace.
Each term in the final sum is the observer variance
along a selected direction and is nonnegative.
The retained observer component is therefore nonzero
in mean square exactly when at least one selected
direction has positive observer variance.
\end{proof}

\paragraph{Interpretation and scope.}
The reconstruction objective selects directions according
to total transition variance, without distinguishing
interaction from observer motion.
The result concerns a selected optimal reconstruction
subspace; it does not establish statistical recoverability
of observer state from the latent.
When eigenvalues are tied at the cutoff, the optimal
subspace may be nonunique.

\subsection{Occurrence and Sensitivity of Observer Variation}
\label{app:observer_energy}

The preceding analysis explains why reconstruction can
retain observer variation. We now examine how observer
motion affects a latent-action encoder.
Following the occurrence--response decomposition
in \citep{whylatentfail}, we separate how often observer
motion occurs from the latent response when it occurs.

\paragraph{Setup.}
Let
$f:\mathbb R^{d_o}\times\mathbb R^{d_o}\to\mathbb R^{d_z}$
be a fixed latent-action encoder, which may be nonlinear.
Using the observation map $h$ defined above, fix an
interaction transition $(\mathbf m,\mathbf m')$ and
an initial observer state $\mathbf c$.
For a final observer state $\mathbf c'$, define
\begin{equation}
    \begin{aligned}
        \mathbf z(\mathbf c')
        &=f\bigl(h(\mathbf m,\mathbf c),
                 h(\mathbf m',\mathbf c')\bigr),\\
        \mathbf z(\mathbf c)
        &=f\bigl(h(\mathbf m,\mathbf c),
                 h(\mathbf m',\mathbf c)\bigr).
    \end{aligned}
    \label{eq:app_energy_composite_encoder}
\end{equation}
The second expression keeps the observer fixed across
the interaction transition.
Their difference,
\begin{equation}
    \Delta\mathbf z
    =\mathbf z(\mathbf c')-\mathbf z(\mathbf c),
    \label{eq:app_energy_latent_displacement}
\end{equation}
therefore measures the latent response to observer motion
while holding the interaction transition and current
observation fixed.

All probabilities and expectations below are over
$\mathbf c'$ conditional on the fixed
$(\mathbf m,\mathbf m',\mathbf c)$.
Let $p=\Pr(\mathbf c'\ne\mathbf c)$ be the probability
of observer motion, and assume
$\mathbb E[\|\Delta\mathbf z\|_2^2]<\infty$.
No centering or linearity assumption is required.

\begin{takeaway}
    \begin{proposition}[Occurrence and response to observer motion]
\label{prop:observer_energy}
For the fixed interaction transition and initial observer
state above, if $p>0$, the expected squared latent change
satisfies
\begin{equation}
    \mathbb E[\|\Delta\mathbf z\|_2^2]
    =
    p\,\mathbb E\!\left[
        \|\Delta\mathbf z\|_2^2
        \mid \mathbf c'\ne\mathbf c
    \right].
    \label{eq:app_energy_main}
\end{equation}
The two factors describe the occurrence probability of
observer motion and the mean squared latent response
when observer motion occurs.
If $p=0$, the latent change is zero almost surely.
\end{proposition}
\end{takeaway}

\begin{proof}
When $\mathbf c'=\mathbf c$, the two encoder inputs
in Eq.~\ref{eq:app_energy_composite_encoder} coincide,
so $\Delta\mathbf z=\mathbf0$.
For $0<p<1$, the law of total expectation gives
\[
    \begin{aligned}
        \mathbb E[\|\Delta\mathbf z\|_2^2]
        ={}&
        p\,\mathbb E\!\left[
            \|\Delta\mathbf z\|_2^2
            \mid \mathbf c'\ne\mathbf c
        \right]\\
        &+(1-p)\,\mathbb E\!\left[
            \|\Delta\mathbf z\|_2^2
            \mid \mathbf c'=\mathbf c
        \right].
    \end{aligned}
\]
The second term is zero, proving
Eq.~\ref{eq:app_energy_main}.
For $p=1$, conditioning on observer motion leaves
the expectation unchanged; for $p=0$,
$\Delta\mathbf z=\mathbf0$ almost surely.
\end{proof}

\paragraph{Interpretation.}
For a fixed observer-motion distribution, the conditional
expectation in Eq.~\ref{eq:app_energy_main} measures
the encoder's average sensitivity to observer motion.
It includes the effects of motion magnitude and direction.
Comparisons between encoders therefore require the same
perturbation distribution and a common latent scale.
If observer motion occurs almost surely, $p=1$ and
the expected latent change is determined entirely
by this response term.

WING-LAM aims to reduce the response of its interaction
channel to observer motion while preserving interaction
information.

\subsection{Information Preservation and Frequency-Ordered Coordinates}
\label{app:dct_coordinates}

We use the discrete cosine transform (DCT) to decompose latent-action
trajectories into temporally ordered components. Here we briefly recall
the standard properties needed for our analysis.

Consider $H\geq 2$ latent actions
$\mathbf z_0,\ldots,\mathbf z_{H-1}\in\mathbb R^{D_i}$,
sampled at a fixed temporal interval, and stack them as
\[
    \mathbf Z
    =
    [\mathbf z_0,\ldots,\mathbf z_{H-1}]^\top
    \in\mathbb R^{H\times D_i}.
\]
Let $\mathbf C\in\mathbb R^{H\times H}$ denote the orthonormal DCT-II
matrix,
\[
    C_{kh}
    =
    \alpha_k
    \cos\left[
        \frac{\pi k}{H}
        \left(h+\frac{1}{2}\right)
    \right],
    \qquad
    \alpha_k
    =
    \begin{cases}
        1/\sqrt{H}, & k=0,\\
        \sqrt{2/H}, & k>0,
    \end{cases}
\]
where $h=0,\ldots,H-1$ indexes time and
$k=0,\ldots,H-1$ indexes temporal frequency.
The DCT coefficients are
\[
    \mathbf Q=\mathbf C\mathbf Z,
\]
with $\mathbf q_k^\top$ denoting the $k$-th row of $\mathbf Q$.

Because $\mathbf C$ is orthonormal, the full transform is invertible
and preserves the trajectory energy:
\begin{equation}
    \mathbf Z=\mathbf C^\top\mathbf Q,
    \qquad
    \|\mathbf Z\|_F^2
    =
    \|\mathbf Q\|_F^2
    =
    \sum_{k=0}^{H-1}\|\mathbf q_k\|_2^2.
    \label{eq:dct_energy}
\end{equation}

The mode index also provides a natural ordering of temporal variation.
Let $\mathbf D\in\mathbb R^{(H-1)\times H}$ be the first-difference
operator,
\[
    (\mathbf D\mathbf v)_h=v_{h+1}-v_h.
\]
For the DCT-II basis, $\mathbf C$ diagonalizes
$\mathbf D^\top\mathbf D$ \citep{dct2}, giving
\[
    \mathbf C\mathbf D^\top\mathbf D\mathbf C^\top
    =
    \operatorname{diag}(\lambda_0,\ldots,\lambda_{H-1}),
    \qquad
    \lambda_k
    =
    4\sin^2\left(\frac{\pi k}{2H}\right).
\]
Therefore,
\begin{equation}
    \sum_{h=0}^{H-2}
    \|\mathbf z_{h+1}-\mathbf z_h\|_2^2
    =
    \sum_{k=0}^{H-1}
    \lambda_k\|\mathbf q_k\|_2^2.
    \label{eq:dct_temporal_variation}
\end{equation}
Since
$0=\lambda_0<\lambda_1<\cdots<\lambda_{H-1}$,
higher-index modes correspond to progressively faster temporal
variation.

These properties only establish that the DCT provides an
information-preserving, frequency-ordered representation of the full
trajectory. We further evaluate this
property empirically through the cross-embodiment analyses in
Appendix~\ref{sec:app_spectral_similarity_analysis}.

%% file: text/appendix/related_work.tex
\section{Related Work}
\label{sec:app_rw}
\subsection{Vision-Language-Action and World-Action Models}
Vision-language-action (VLA) models have emerged as a dominant paradigm for
generalist robot manipulation, leveraging large-scale vision-language
pretraining and heterogeneous robot demonstrations to map visual observations
and language instructions directly to actions~\citep{rt2,openvla,octo,pi_0,pi_0.5,gr00t,behaviorvla}.
Recent works further scale model capacity, data diversity, and
cross-embodiment training, leading to increasingly general robot
policies across tasks and platforms~\citep{openx, rdt, cogact}.
Despite their strong semantic and visual priors, these policies primarily
optimize action prediction and do not explicitly model how the observed world
evolves under interaction.

World-action models (WAMs) provide a complementary direction by jointly
modeling robot actions and future visual dynamics.
Early approaches explored video prediction as a planning interface or as
auxiliary supervision for control, while recent WAMs increasingly build on
large-scale pretrained video generation models and learn actions together
with future visual representations~\citep{
dreamzero, fastwam, lbva, flexpi, giga}.
Such joint prediction allows action learning to benefit from the
spatiotemporal priors encoded by video models and has shown strong
data efficiency and generalization in robot manipulation.
Recent studies have further explored more efficient future prediction,
action-centered world modeling, and richer geometric or semantic prediction
targets~\citep{fastwam, flexpi, giga}.

Our work builds on this world-action modeling paradigm but focuses on a
different source of supervision.
Rather than relying only on robot trajectories to learn interaction dynamics,
we study how large-scale egocentric human videos can provide
interaction priors for robot action generation.
WING extracts these priors into a compact latent guidance signal, allowing
human interaction dynamics to complement the visual world-modeling capability
of a pretrained WAM without requiring human actions to be directly mapped to
robot controls.

\subsection{Learning from Egocentric Human Videos for Robot Manipulation}

Egocentric human videos offer a scalable source of manipulation experience,
covering diverse objects, environments, and interaction behaviors that are
costly to collect through robot teleoperation.
Large-scale datasets such as Ego4D, EPIC-KITCHENS, EgoDex, and EgoVerse have
made such experience increasingly accessible~\citep{
ego4d, epic-kitchens, egodex, egoverse}.
Early efforts primarily used human videos to learn transferable visual
representations or reward functions for downstream robot learning
~\citep{r3m, vip, liv}.
More recent approaches extract explicit motion cues, such as hand, wrist, or
body trajectories, and convert them into robot-compatible supervision through
retargeting, inverse kinematics, or embodiment-aware action representations
~\citep{ace-ego, egomimic, egovla, zeromimic, humanplus}.
These methods provide stronger behavioral supervision than visual pretraining,
but generally require recovering human motion in a form that can be explicitly
aligned with robot actions.

Latent-action learning provides a way to exploit human videos without
recovering executable actions. These methods infer action-related
representations from visual transitions and use them for policy or VLA
pretraining~\citep{lapa,univla,conla,como}. Because visual change is not
equivalent to interaction dynamics, however, the resulting representations
may also encode appearance, static scene content, and motion unrelated to
manipulation. Recent approaches reduce such variation through
task-conditioned representations, contrastive objectives, temporal cues, or
foreground--background disentanglement
~\citep{univla,conla,como,motion}. CD-LAM~\citep{cdlam} further studies
action-irrelevant bias in latent actions and improves their controllability
through embodiment-centric reconstruction, action-centric contrastive
learning, and latent-space calibration, including an evaluation of sensitivity
to synthetic camera shifts. Together, these methods establish the importance
of making latent actions more closely reflect embodied behavior, but primarily
focus on action relevance or world-model controllability rather than the
transferable temporal structure of egocentric interactions.

Observer motion makes this transfer problem particularly challenging.
Head- and body-induced camera motion can dominate the visual transition even
when the underlying hand--object interaction changes little. Prior work has
recovered or compensated for camera motion when reconstructing human
trajectories~\citep{egoaero}, while other methods treat observer motion as a
useful component of active perception~\citep{activemimic}. These approaches
clarify the role of observer motion in egocentric behavior, but pursue
different objectives from learning latent interaction dynamics for robot
policy guidance. Meanwhile, existing latent-action debiasing methods do not
explicitly use observer and interaction motion as separate supervision
channels. It therefore remains unclear how to reduce observer-induced
variation in egocentric latent actions while preserving the interaction
structure that can transfer to robot control.

WING addresses this problem in two stages. WING-LAM first uses geometric
motion routing to separate dominant observer-induced image motion from local
interaction residuals, and distills the resulting interaction representation
into an RGB-only latent-action encoder. Reducing observer sensitivity alone,
however, does not resolve differences in how humans and robots execute the
same interaction, particularly in execution speed and fine-scale motion.
WING therefore analyzes latent-action trajectories across temporal frequency
bands, identifies low-frequency dynamics with stronger correspondence between
semantically matched human and robot interactions, and uses these dynamics as
guidance for robot action generation. In this way, WING connects
observer-aware latent-action learning with the temporal structure used for
ego-to-robot transfer.

%% file: text/appendix/a_implementation_training_details.tex
\section{Implementation and Training Details}
\label{sec:app_implementation_training}

\input{text/appendix/a_1_oilam_architecture_training_details}
\input{text/appendix/a_2_policy_training_details}

%% file: text/appendix/a_1_oilam_architecture_training_details.tex
\subsection{WING-LAM Architecture and Training Details}
\label{sec:app_oilam_training}

This section provides the implementation details of WING-LAM introduced in
Section~\ref{sec:oilam}. WING-LAM is trained in two stages. We first construct
structured motion supervision from egocentric frame pairs and train a
motion-routed teacher that separates dominant observer-induced image motion
from local interaction residuals. We then freeze the teacher and distill its
interaction representation into an RGB-only student. The resulting student
is frozen and used offline to construct latent-action targets for subsequent
WING training. The optimization hyperparameters are summarized in
Table~\ref{tab:oilam_configuration}.

\subsubsection{Offline Motion Supervision Construction}
\label{sec:oilam_motion_supervision}

\paragraph{Point tracks and reliability.}
Given a frame pair
$(\mathbf{o}_{t},\mathbf{o}_{t+\delta})$,
we use AllTracker~\citep{alltracker} to extract point correspondences:
\[
(\mathbf p_{t,i},\mathbf p_{t+\delta,i}),
\qquad i=1,\ldots,N_t,
\]
where $\mathbf p_{t,i}\in\mathbb{R}^2$ denotes the image-plane
location of the $i$-th tracked point in frame $\mathbf{o}_{t}$,
$i$ indexes tracked correspondences across the frame pair, and
$N_t$ is the number of correspondences returned for that pair. AllTracker additionally provides visibility and tracking-confidence
estimates. We retain reliable correspondences and associate each
retained track with a scalar reliability weight $\rho_i$ derived
from these estimates.
EgoHOS~\citep{egohos} provides an initial hand-region seed.

\paragraph{Global observer-motion estimation.}
We exclude tracks inside the initial hand region and fit an affine
transformation to the remaining reliable background correspondences using
deterministic RANSAC~\citep{RANSAC}. The resulting warp
$W_{t\rightarrow t+\delta}$ provides an image-space approximation
of the dominant coherent motion induced by the moving observer. We also
record a fitting confidence $\rho_{\rm cam}$ from the RANSAC fit and use it
to weight the camera-motion supervision.

\paragraph{Interaction-region construction.}
For every reliable track, we compute its camera-compensated residual
following Eq.~\ref{eq:oilam_residual}:
\begin{equation}
    \mathbf r_{t,i}
    =
     W_{t\rightarrow t+\delta}^{-1}
    (\mathbf p_{t+\delta,i})
    -
    \mathbf p_{t,i}.
\end{equation}
Starting from the hand-region seed, we expand the interaction region to
spatially nearby tracks with sufficiently large compensated residual
motion. The resulting binary indicator $m_i\in\{0,1\}$ distinguishes
interaction tracks from background tracks. We denote
\begin{equation}
    \mathcal I_t=\{i:m_i=1\},
    \qquad
    \mathcal B_t=\{i:m_i=0\},
\end{equation}
and collect the privileged teacher inputs as
\begin{equation}
    \mathcal G_t
    =
    \{
      \mathbf p_{t,i},
      \mathbf p_{t+\delta,i},
      m_i,
      \rho_i
    \}_{i=1}^{N_t}.
\end{equation}
All quantities in $\mathcal G_t$ are used only during WING-LAM training.

\subsubsection{Motion-Routed Teacher Architecture and Objectives}
\label{sec:oilam_teacher}

\paragraph{Teacher architecture.}
The teacher follows the two-frame spatiotemporal Transformer architecture
of DreamDojo~\citep{dreamdojo}. The visual backbone encodes
$(\mathbf{o}_{t},\mathbf{o}_{t+\delta})$ into spatiotemporal visual features.
For each reliable correspondence from $\mathbf {p}_{t,i}$ to $\mathbf {p}_{t+\delta,i}$, we sample the visual features at the corresponding locations in the two frames and concatenate them with the normalized point coordinates, displacement, and tracking confidence to obtain the track descriptor $\mathbf d_{t,i}$. Two gated set-pooling modules aggregate background and interaction
descriptors separately:
\begin{equation}
\begin{aligned}
    \mathbf z_{t, \rm cam}^{T}
    &=
    P_{\rm cam}
    \left(
        \{\mathbf d_{t,i}\}_{i\in\mathcal B_t}
    \right),
\\
    \mathbf z_{t, \rm int}^{T}
    &=
    P_{\rm int}
    \left(
        \{\mathbf d_{t,i}\}_{i\in\mathcal I_t}
    \right),
\end{aligned}
\end{equation}
where
$\mathbf z_{\rm cam}^{T}\in\mathbb R^{D_c}$
and
$\mathbf z_{\rm int}^{T}\in\mathbb R^{D_i}$. $T$ denotes the teacher.
The two latent codes are decoded independently as: 
\begin{equation}
    \widehat{ W}_{t\rightarrow t+\delta}
    =
    D_{\rm cam}(\mathbf z_{t, \rm cam}^{T}),
    \qquad
    \widehat{\mathbf r}_{t,i}
    =
    D_{\rm int}
    (\mathbf p_{t,i},\mathbf z_{t, \rm int}^{T}).
\end{equation}

For every reliable track, the decoded global and interaction motions are
composed as:
\begin{equation}
    \widehat{\mathbf p}_{t+\delta,i}
    =
    \widehat{ W}_{t\rightarrow t+\delta}
    \left(
        \mathbf p_{t,i}
        +
        m_i\widehat{\mathbf r}_{t,i}
    \right).
    \label{eq:app_oilam_composition}
\end{equation}
Thus, background tracks are explained only by the global warp, whereas
interaction tracks additionally receive the decoded residual motion.

\paragraph{Motion supervision.}
Let $h_\beta(\cdot,\cdot)$ denote the coordinate-wise Smooth-$L_1$
distance with $\beta=0.01$, and let
$\langle f_i\rangle_{w_i}$ denote the normalized weighted average of
$f_i$ with weights $w_i$. We supervise the global motion, canonical
interaction residual, and their composition by:
\begin{equation}
\begin{aligned}
\mathcal L_{\rm cam}
&=
\rho_{cam}\left\langle
h_\beta
\left(
\widehat{\mathbf W}_{t\rightarrow t+\delta}(\mathbf p_{t,i}),
\mathbf W_{t\rightarrow t+\delta}(\mathbf p_{t,i})
\right)
\right\rangle_{\rho_i,\,i\in\mathcal B_t},
\\
\mathcal L_{\rm int}
&=
\left\langle
h_\beta
\left(
\widehat{\mathbf r}_{t,i},
\mathbf r_{t,i}
\right)
\right\rangle_{\rho_i,\,i\in\mathcal I_t},
\\
\mathcal L_{\rm trk}
&=
\left\langle
h_\beta
\left(
\widehat{\mathbf p}_{t+\delta,i},
\mathbf p_{t+\delta,i}
\right)
\right\rangle_{\rho_i}.
\end{aligned}
\label{eq:app_oilam_motion_losses}
\end{equation}
Here, $\mathcal L_{\rm cam}$ constrains the decoded observer-motion warp
on background tracks, $\mathcal L_{\rm int}$ supervises the
camera-compensated interaction residuals, and $\mathcal L_{\rm trk}$
ensures that the two components jointly explain the observed track
endpoints.

\subsubsection{Camera Intervention and Representation Regularization}
\label{sec:oilam_camera_intervention}

\paragraph{Camera intervention.}
Motion routing reduces observer-induced motion in the interaction
branch, but does not guarantee that the resulting representation is
insensitive to observer-specific cues. We therefore construct a
camera-intervened counterpart of each transition by perturbing only
the global camera motion while preserving the underlying interaction
motion. Encoding the natural and camera-intervened transitions yields
teacher interaction representations
$\mathbf z_{t,\mathrm{int}}^{T,\mathrm{nat}}$
and
$\mathbf z_{t,\mathrm{int}}^{T,\mathrm{cam}}
\in\mathbb R^{D_i}$, respectively.
We encourage them to agree through:
\begin{equation}
    \mathcal L_{\mathrm{cons}}
    =
    \frac{1}{D_i}
    \sum_{d=1}^{D_i}
    h_\beta
    \left(
        [\mathbf z_{t,\mathrm{int}}^{T,\mathrm{nat}}]_d,
        [\mathbf z_{t,\mathrm{int}}^{T,\mathrm{cam}}]_d
    \right),
    \label{eq:oilam_teacher_consistency}
\end{equation}
where $[\mathbf z]_d$ denotes the $d$-th coordinate of
$\mathbf z\in\mathbb R^{D_i}$. 
This consistency objective encourages the interaction representation
to retain interaction information while reducing its sensitivity to
camera motion.

\paragraph{Representation regularization.}
To prevent collapse and reduce redundancy, we apply variance and
covariance regularization to the interaction representations from the
natural and intervened views jointly. Let $\sigma_d$ denote the batchwise
standard deviation of latent dimension $d$. The variance regularizer is
\begin{equation}
    \mathcal L_{\rm var}
    =
    \frac{1}{D_i}
    \sum_{d=1}^{D_i}
    \max(0,\gamma-\sigma_d)^2,
    \qquad
    \gamma=0.1.
\end{equation}
Let $\mathbf C\in\mathbb R^{D_i\times D_i}$ denote the covariance
matrix of the standardized latent dimensions, where $C_{de}$ is the
covariance between latent dimensions $d$ and $e$. We further use
\begin{equation}
    \mathcal L_{\rm cov}
    =
    \frac{1}{D_i(D_i-1)}
    \sum_{d\neq e}
    C_{de}^{2}.
\end{equation}

The complete teacher objective is
\begin{equation}
\begin{aligned}
    \mathcal L_{\rm teacher}
    ={}&
    \lambda_{\rm trk}\mathcal L_{\rm trk}
    +
    \lambda_{\rm cam}\mathcal L_{\rm cam}
    +
    \lambda_{\rm int}\mathcal L_{\rm int}
\\
    &+
    \lambda_{\rm cons}\mathcal L_{\rm cons}
    +
    \lambda_{\rm var}\mathcal L_{\rm var}
    +
    \lambda_{\rm cov}\mathcal L_{\rm cov}.
    \label{eq:app_oilam_teacher}
\end{aligned}
\end{equation}
The motion-supervision terms establish the camera--interaction routing,
while the consistency and representation regularizers further suppress
observer-motion leakage and stabilize the interaction latent space.

\subsubsection{Interaction Distillation}
\label{sec:oilam_distillation}

The motion-routed teacher depends on privileged point tracks,
interaction masks, and global-motion estimates, making it unsuitable for
large-scale latent extraction. We therefore freeze the teacher and
distill only its interaction representation into an RGB-only student
$S_\phi$.

The student is initialized from the teacher visual encoder but receives
only the RGB frame pair. Its spatiotemporal visual tokens are mean-pooled
and linearly projected into a $D_i$-dimensional interaction latent; no
point tracks, masks, or motion parameters are provided. For each training transition, we use the frozen teacher's
interaction representation from the natural view as the
common distillation target:
\begin{equation}
    \mathbf z_{t,\mathrm{int}}^T
    :=
    \mathbf z_{t,\mathrm{int}}^{T,\mathrm{nat}}
    =
    E_{T,\mathrm{int}}
    (\mathbf{o}_{t},\mathbf{o}_{t+\delta},\mathcal G_t),
    \label{eq:oilam_distillation_target}
\end{equation}
where $E_{T,\mathrm{int}}$ denotes the interaction output
of the teacher $E_T$. The student encodes the natural and camera-intervened frame
pairs into $\mathbf z_t^{\mathrm{nat}}$ and
$\mathbf z_t^{\mathrm{aug}}$, respectively.
Both representations are trained to match
$(\mathbf z_{t,\mathrm{int}}^T)$
and to agree with each other through
Eq.~\ref{eq:oilam_distill}.
The teacher remains frozen, and gradients update only
the student $S_\phi$. After distillation, all privileged motion-processing components and the
teacher are discarded. The frozen RGB-only student defines the WING-LAM
latent action
\begin{equation}
    \mathbf z_t
    =
    S_\phi
    \left(
        \mathbf{o}_{t},
        \mathbf{o}_{t+\delta}
    \right),
\end{equation}
which is subsequently used offline to construct the latent-action
trajectories for WING.

\input{tab/a_1_oilam_configuration}

%% file: tab/a_1_oilam_configuration.tex
\begin{table}[t]
    \centering
    \caption{\textbf{WING-LAM training
    configuration.}}
    \label{tab:oilam_configuration}
    \footnotesize
    \setlength{\tabcolsep}{4pt}
    \renewcommand{\arraystretch}{1.05}
    \begin{tabularx}{\linewidth}{p{0.28\linewidth}X}
        \toprule
        Setting & Value \\
        \midrule
        \multicolumn{2}{l}{\textit{Motion-Supervised Teacher Learning}} \\
        Latent dimensions & $D_{\mathrm{c}}=16$; $D_{\mathrm{i}}=32$ \\
        Loss weights & $\lambda_{\mathrm{trk}}=\lambda_{\mathrm{cam}}
        =\lambda_{\mathrm{int}}=\lambda_{\mathrm{var}}=1$;
        $\lambda_{\mathrm{cons}}=0.1$; $\lambda_{\mathrm{cov}}=0.01$ \\
        Optimizer & AdamW \\
        Training steps & $60{,}000$ \\
        Learning rate & 
        $5\!\times\!10^{-6}$ \\
        Weight decay & $10^{-4}$ \\
        \midrule
        \multicolumn{2}{l}{\textit{Interaction Representation Distillation}} \\
        Latent dimensions & $D_{\mathrm{i}}=32$  \\
        Loss weights & $\lambda_{\mathrm{view}}=1$ \\
        Optimizer & AdamW \\
        Training steps & $60{,}000$ \\
        Learning rate & $10^{-5}$ \\
        Weight decay & $0.05$ \\
        \bottomrule
    \end{tabularx}
\end{table}

%% file: text/appendix/a_2_policy_training_details.tex
\subsection{Egocentric-Video Pretraining Details}
\label{sec:app_pretraining_details}

\subsubsection{Pre-training Data and Preprocessing.}
We pre-train our model on three large-scale egocentric video datasets: Ego4D, EgoDex, and EgoVerse, totaling approximately 200 hours of ego-videos. Specifically, Ego4D contributes approximately 20 hours, EgoDex contributes about 85 hours, and EgoVerse contributes 95 hours. Language instructions are obtained from the original task descriptions of EgoDex and EgoVerse, while those for Ego4D are constructed from the verb–noun action labels in its official long-term anticipation annotations. All egocentric videos are standardized to 30 FPS. The sampled RGB frames are resized to $384\times640$ and encoded by the frozen Wan2.2 VAE.

\paragraph{Latent Action Preprocessing.}
To construct latent-action supervision, we sample egocentric videos at
a fixed physical interval of $0.2\,\mathrm{s}$, corresponding to every
6 frames at 30 FPS. The sampled frames are center-cropped to a 4:3
aspect ratio, resized to $240\times320$, and consecutive sampled
endpoints are encoded by the frozen WING-LAM into 32-dimensional
latent transitions. Nine sampled endpoints yield $H=8$ transitions,
spanning a $1.6\,\mathrm{s}$ temporal window.

\paragraph{Hand Pose Preprocessing}
For hand-pose-based pretraining, we construct a 200-hour corpus by supplementing it with 10 hours of hand-annotated data from each of EgoDex and EgoVerse, replacing the Ego4D portion that lacks the required annotations.
We then follow the Ego2Robot~\citep{ego2robot} pipeline in RobotManip~\citep{robotmanip} to extract
hand end-effector (EEF) trajectories from these annotations.
The resulting trajectories serve as action supervision during
model pretraining.

\subsubsection{Guidance Predictor Training}
During egocentric-video pretraining, we train the guidance predictor to
estimate interaction-centric spectral guidance from the current visual
observation and language instruction. For each training sample, the frozen
WING-LAM encoder extracts an \(8\times32\) latent-action trajectory from the
sampled video frames. We apply an orthonormal temporal DCT and retain the
lowest four frequency components, producing a \(4\times32\) guidance target.

The predictor is a lightweight, current-observation-conditioned query
Transformer. It uses four learned queries that cross-attend to a context
formed by concatenating linearly projected current-frame tokens and
language-instruction tokens. A single Transformer block with a hidden
dimension of 1,024, eight attention heads, and a fourfold-expanded feed-forward
network produces four query features, which are projected to 32 dimensions to
predict the \(4\times32\) normalized guidance target. The prediction loss is
weighted by 0.1. Future frames are used only to
construct the training target and are never provided to the predictor.

\subsubsection{Optimization Setup}
We conduct egocentric-video pre-training for 40 epochs on 32 NVIDIA H100 GPUs. We use a per-GPU batch size of 16, resulting in an effective global batch size of 512. The learning rate is set to $1\times10^{-4}$. To reduce training overhead, both the Wan2.2 VAE latents and WING-LAM latent transitions are pre-computed and cached before training. The video backbone is initialized from Wan2.2-TI2V-5B~\citep{wan}.

\subsection{Robot Policy Training Details}
\label{sec:app_robot_policy_training}

\subsubsection{Training Framework}
For our model, the guidance predictor and shared visual backbone are
initialized from the egocentric-video pretrained checkpoint. During robot policy training, the guidance predictor is jointly optimized with
the robot policy. Given the current visual observation, language instruction,
and robot proprioceptive state, the predictor estimates the low-frequency
spectral representation of the anticipated interaction dynamics. The
predicted coefficients are projected into guidance tokens and injected into the action model through gated cross-attention, providing an additional conditioning
signal for action prediction. The spectral prediction targets are constructed
following the same procedure as in egocentric-video pretraining. The overall training objective is
\begin{equation}
    \mathcal{L}_{\mathrm{total}}
    =
    \lambda_{\mathrm{video}}\mathcal{L}_{\mathrm{video}}
    +
    \lambda_{\mathrm{action}}\mathcal{L}_{\mathrm{action}}
    +
    \lambda_{\mathrm{pred}}\mathcal{L}_{\mathrm{pred}},
\end{equation}
where
\(\lambda_{\mathrm{video}}=\lambda_\mathrm{action}=1\), and
\(\lambda_{\mathrm{pred}}=0.1\).

\subsubsection{Optimization Setup}
Robot policy post-training is conducted on 8 NVIDIA H100 GPUs with a global batch size of 128. We use the AdamW optimizer with a peak learning rate of $2\times10^{-4}$. The learning rate is linearly warmed up and then decayed according to a cosine schedule.

%% file: text/appendix/b_evaluation_details.tex
\section{Evaluation Details}
\label{sec:app_evaluation_details}
In this section, we present the details of our evaluation experiments.
\input{text/appendix/b_1_observer_motion_evaluation_details}
\input{text/appendix/b_3_matched}
\input{text/appendix/b_4_simulation_evaluation_details}
\input{text/appendix/b_5_real_world_evaluation_protocols}

%% file: text/appendix/b_1_observer_motion_evaluation_details.tex
\subsection{Latent-Action Representation Evaluation Details}
\label{sec:app_observer_motion_diagnosis}

We evaluate frozen latent-action representations along three
complementary dimensions:
\textit{(i) what camera- and interaction-motion information can be decoded
from the latent representation,}
\textit{(ii) how sensitive the representation is to controlled changes in
camera motion,} and
\textit{(iii) how well it preserves high-level action semantics.}
This section describes the datasets, targets, and evaluation
protocols used for these analyses.

\subsubsection{Datasets and Frozen Representations}

We evaluate WING-LAM together with DreamDojo~\citep{dreamdojo}, CD-LAM~\citep{cdlam}, LAPA~\citep{lapa},
VILLA-X~\citep{villax}, and UniVLA~\citep{univla} using their frozen latent-action encoders.
The motion decodability and camera-sensitivity analyses use EgoDex~\citep{egodex}, EgoVerse~\citep{egoverse}, and Ego4D~\citep{ego4d}.
EgoDex provides synchronized head-pose annotations, while EgoVerse
and Ego4D provide additional diversity in natural observer and
interaction motion. We split videos by source
trajectory to avoid overlap between probe training and evaluation. The action-semantic evaluation follows the LARYBench protocol and is described separately in Appendix~\ref{sec:app_action_semantics}.

\subsubsection{Motion Decodability}

We evaluate how much camera- and interaction-motion information can
be decoded from each frozen latent representation.
For camera motion, we use the relative head motion when ground-truth
head pose is available, and otherwise use the estimated global image
motion as a proxy.
For interaction motion, we remove the estimated global camera motion
from the tracked point motion and aggregate the remaining local motion
into a fixed-dimensional target.
We evaluate each frozen latent representation using three probe
families: a ridge linear probe, a two-layer MLP with GELU activations,
and a random Fourier feature (RFF) probe followed by regularized linear
regression. The MLP is trained with AdamW, while the ridge and RFF
probes use regularized linear regression. In all cases, the latent
encoder remains frozen and only the probe is optimized.

For each target, we train a probe on the frozen latent representation
and report the coefficient of determination ($R^2$) on the held-out
test set:
\begin{equation}
    R^2
    =
    1-
    \frac{
        \sum_{n=1}^{N}
        \left\|
            \mathbf y_n-\widehat{\mathbf y}_n
        \right\|_2^2
    }{
        \sum_{n=1}^{N}
        \left\|
            \mathbf y_n-\bar{\mathbf y}
        \right\|_2^2
    },
    \qquad
    \bar{\mathbf y}
    =
    \frac{1}{N}\sum_{n=1}^{N}\mathbf y_n,
    \label{eq:motion_probe_r2}
\end{equation}
where $\mathbf y_n$ and $\widehat{\mathbf y}_n$ denote the ground-truth
and predicted motion targets for sample $n$, respectively.
Higher $R^2$ indicates that the corresponding motion information is
more readily decoded from the latent representation.
An $R^2$ of zero corresponds to the constant-mean baseline, while
negative values indicate lower predictive accuracy than this baseline.

\subsubsection{Sensitivity to Camera Perturbations}
\label{sec:app_camera_sensitivity}

We complement motion decodability with a controlled intervention
that directly measures how much a frozen latent representation changes
under additional camera motion.
For each frame pair $(\mathbf x_t,\mathbf x_{t+\delta})$, we construct
matched \emph{dynamic} and \emph{static} perturbations.
Dynamic perturbations apply different global transformations to the two
frames, introducing additional inter-frame camera motion, whereas static
perturbations apply the same transformation to both frames.
The transformation magnitude and image-processing pipeline are matched
between the two conditions, so that their difference primarily reflects
sensitivity to relative camera motion rather than interpolation or boundary
artifacts. For the multi-model comparison in Fig.~\ref{fig:lam}, we instantiate this
protocol using two transformations $A$ and $B$ of equal magnitude and
opposite direction.
The pairs $AB$ and $BA$ form the dynamic conditions, while $AA$ and $BB$
serve as their matched static controls.
Let $\mathbf z^{PQ}$ denote the latent representation obtained after
applying transformation $P$ to the first frame and $Q$ to the second,
and let $\mathbf z^{0}$ denote the representation of the unperturbed pair.

Because latent scales differ across models, each latent dimension is
standardized using statistics computed from the corresponding training
split. We measure the standardized deviation of each perturbed representation
from the unperturbed pair as
\begin{equation}
    d^{PQ}
    =
    \sqrt{
        \frac{1}{D_i}
        \sum_{d=1}^{D_i}
        \left(
            \frac{
                [\mathbf z^{PQ}]_d-[\mathbf z^{0}]_d
            }{\sigma_d}
        \right)^2
    },
    \qquad
    PQ\in\{AB,BA,AA,BB\},
    \label{eq:camera_latent_deviation}
\end{equation}
where $\sigma_d$ is the training-set standard deviation of latent
dimension $d$ and $D_i$ denotes the latent action's dimension.
We report the dynamic-minus-static deviation,
\begin{equation}
    \frac{1}{2}
    \left[
        d^{AB}
        +
        d^{BA}
    \right]
    -
    \frac{1}{2}
    \left[
        d^{AA}
        +
        d^{BB}
    \right].
    \label{eq:camera_sensitivity}
\end{equation}
A smaller value indicates that the representation changes less in response
to additional inter-frame camera motion after accounting for matched
image perturbations.

We evaluate five increasing perturbation levels,
$(0.25\,\mathrm{px},0.05^\circ)$,
$(0.5\,\mathrm{px},0.1^\circ)$,
$(1\,\mathrm{px},0.2^\circ)$,
$(2\,\mathrm{px},0.4^\circ)$, and
$(4\,\mathrm{px},0.8^\circ)$,
where each pair specifies the translation and rotation magnitude.

\subsubsection{Action-Semantic Preservation}
\label{sec:app_action_semantics}

We further evaluate whether reducing camera-motion sensitivity
preserves interaction-relevant semantic information.
Following the LARYBench~\citep{larybench} evaluation protocol, we
evaluate frozen latent representations on action-category
classification across Ego4D~\citep{ego4d}, EgoDex~\citep{egodex}, EPIC-KITCHENS~\citep{epic-kitchens}, and HoloAssist~\citep{holoassist}.
For each method, the latent encoder remains frozen and the same
classification protocol is applied to the representations.

%% file: text/appendix/b_3_matched.tex
\subsection{Pair Construction for Spectral Similarity Analysis}
\label{sec:app_pair_construction}
\input{img/matched}
To evaluate cross-embodiment correspondence, we construct matched and mismatched human--robot clip pairs through a hierarchical annotation protocol. We first retrieve candidate human and robot episodes from EgoDex and EgoDex-Eval according to their task annotations. These task labels are used only for coarse-grained episode-level filtering, since episodes assigned to the same task may still contain substantially different local motion patterns, interaction phases, or even opposite motion directions.

Within each candidate episode pool, human annotators further inspect the temporal content and select short clips according to their dominant action semantics. Specifically, a matched pair is formed by selecting a human clip and a robot clip that exhibit similar dominant interaction dynamics, rather than requiring frame-wise or trajectory-level alignment. Consequently, matched clips may differ substantially in object instance, scene appearance, viewpoint, operator, and precise end-effector trajectory, while still sharing the same underlying interaction pattern. This procedure therefore performs coarse task-level retrieval first, followed by fine-grained clip-level semantic matching within the retrieved episodes.

Mismatched pairs are constructed from the same task-conditioned candidate pools, but the selected human and robot clips exhibit inconsistent dominant motion semantics. This preserves a comparable task distribution while removing action-level correspondence. To reduce dependence among evaluation samples, each source episode is used at most once in the matched set and at most once in the mismatched set. The final evaluation set contains 100 matched and 100 mismatched pairs, covering 25 distinct tasks. The pair set and all match/mismatch labels were finalized before computing any latent trajectories or spectral similarities. To assess annotation reliability, two additional annotators independently reviewed the constructed pairs and showed high agreement with the original labels.

We use our frozen WING-LAM, the same latent-action encoder used to construct spectral guidance in WING, to encode both human and robot clips. WING-LAM is trained only on egocentric videos and is not adapted using robot data for this analysis. Therefore, this evaluation directly tests whether the interaction-centric latent space learned from egocentric data generalizes across embodiments, and whether semantically matched human and robot interactions exhibit consistent temporal structure in this shared representation.

For both egocentric and robot clips, we use the same physical sampling interval of $0.2\mathrm{s}$, corresponding to an effective sampling rate of 5 Hz. Each $H=8$ latent-action trajectory is constructed from 9 sampled observation endpoints and spans $1.6\mathrm{s}$. We then apply DCT-II along the temporal dimension and compute cosine similarity between the corresponding low-, high-, and full-frequency representations. We report the average similarity over the 100 matched and 100 mismatched pairs in Figure~\ref{fig:lam}.

%% file: img/matched.tex
\begin{figure}
    \centering
    \includegraphics[width=1\linewidth]{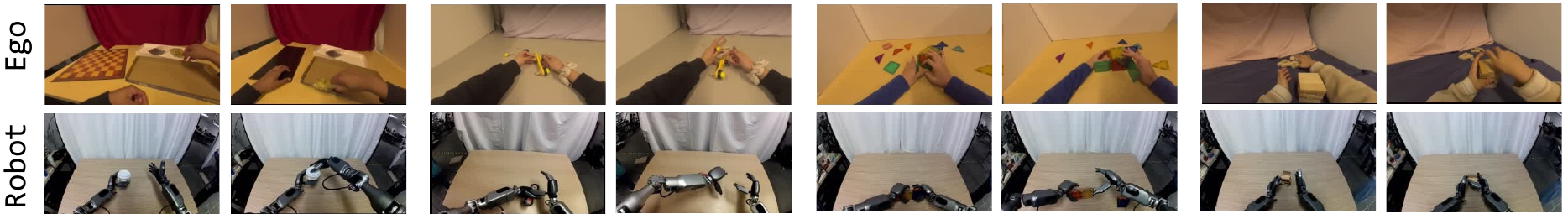}
    \caption{\textbf{Examples of matched ego--robot interaction pairs.}
Each group shows two frames from an egocentric clip and two frames
from its semantically matched robot clip.}
    \label{fig:matched_pair_examples}
\end{figure}

%% file: text/appendix/b_4_simulation_evaluation_details.tex
\subsection{Simulation Evaluation Details}
\label{sec:app_simulation_results}

\subsubsection{Benchmark Protocols and Metrics}
\label{sec:app_simulation_protocols}

\paragraph{LIBERO.}
LIBERO contains four suites---Spatial, Object, Goal, and Long---with 10 tasks
and 500 demonstrations per suite. We train one checkpoint across the four
suites and evaluate each task for 50 rollouts, giving 2,000 rollouts in total.
We use the benchmark episode budgets for each suite and report the mean task
success rate for every suite together with their four-suite average.

\paragraph{RoboTwin~2.0.}
RoboTwin evaluates 50 bimanual tasks in clean and domain-randomized scenes.
Following the standard multi-task protocol, training uses 2,500 clean
demonstrations and 25,000 randomized demonstrations. Each task is evaluated
for 100 rollouts under each scene condition. We report the 50-task clean mean,
the randomized mean, and their arithmetic average.

\paragraph{RoboCasa--GR1.}
RoboCasa--GR1 contains 24 tabletop tasks executed by a Fourier GR1 humanoid
with two 7-DoF arms, two 6-DoF hands, and a 3-DoF waist. We use the Unified 1K
dataset with 1,000 teleoperated demonstrations per task and egocentric RGB
observations. Each task is evaluated over 50 rollouts, and the main metric is the average of the 24
task success rates.

\paragraph{Evaluation.}
All simulation tables report success rate in percent. Notably, the FastWAM result on RoboCasa--GR1 is our
reproduction trained to convergence. Complete RoboTwin and RoboCasa task-level tables are reported
below; method-specific optimization details are given in Appendix~\ref{sec:app_robot_policy_training}.

\subsubsection{Baselines}
\label{sec:app_baselines}

We compare our method with recent robot manipulation policies spanning three
representative paradigms: vision-language-action models, latent-action
based policies, and world-action modeling approaches. Together, these baselines
cover methods that directly map multimodal observations to robot actions,
methods that introduce an intermediate latent action representation, and
methods that explicitly exploit predicted future dynamics or world
representations for policy learning.

\paragraph{Vision-language-action models.}
We first consider representative vision-language-action (VLA) policies that
directly predict robot actions from visual observations and language
instructions. These include $\pi_0$~\citep{pi_0} and
$\pi_{0.5}$~\citep{pi_0.5}, X-VLA~\citep{xvla}, LingBot-VLA~\citep{lbvla}, GR00T
(N1.6 and N1.7)~\citep{gr00t}, StarVLA-OFT~\citep{starvla} and 
LDA-1B~\citep{lda1b}. This group represents strong recent VLA systems that
learn robot control directly in the action space without explicitly introducing
a latent-action or future-dynamics interface.

\paragraph{Latent-action based methods.}
We additionally compare against approaches that introduce latent actions as an
intermediate representation for learning manipulation
behaviors, including LAPA~\citep{lapa}, UniVLA~\citep{univla},
LaWAM~\citep{lawam}, Being-H0.7~\citep{beingh07}, and VLA-JEPA~\citep{vlajepa}. These methods are
particularly relevant to our setting because they investigate whether
interaction or action information can be represented in a latent space rather
than being directly tied to embodiment-specific robot commands. In contrast, our method uses interaction-centric latent dynamics with reduced sensitivity to observer motion and emphasizes their low-frequency temporal structure as guidance for robot action learning.

\paragraph{World-action modeling methods.}
The third group contains methods that explicitly leverage predictive world
representations, future observations, or future latent dynamics to facilitate
action learning. We include Motus~\citep{motus},
LingBot-VA~\citep{lbva}, FastWAM~\citep{fastwam},
GigaWorld-Policy~\citep{giga}, LDA-1B~\citep{lda1b} and DiT4DiT~\citep{dit4dit}.
These baselines are the closest to our method in their use of anticipated future dynamics for control. Our method differs by constructing guidance from interaction-centric latent dynamics extracted from egocentric videos and using their low-frequency components, which exhibit stronger correspondence across human and robot executions, rather than conditioning the policy on the full latent trajectory.

\paragraph{Real-world baselines.}
For the real-world experiments, we focus on representative strong methods from
the above families and compare against GR00T-N1.7~\citep{gr00t},
$\pi_{0.5}$~\citep{pi_0.5}, and FastWAM~\citep{fastwam}. These baselines
provide comparisons with a strong general-purpose VLA model, a recent
large-scale robot policy, and a predictive world-action modeling approach,
respectively, under the same real-world evaluation protocol.

%% file: text/appendix/b_5_real_world_evaluation_protocols.tex
\subsection{Real-World Evaluation Protocols}
\label{sec:app_real_world_protocols}

\input{img/6_real_world_benchmark}

\subsubsection{Tasks and Progress Metrics}
\label{sec:app_real_world_tasks}

The real-world suite contains four tasks.  \emph{Pack Objects} requires placing a two-brick LEGO assembly, a tape roll, and an inverted plastic cup in a container.  \emph{Stack Cups} requires arranging three cups into a stable tower.  \emph{Battery Insertion} requires grasping a battery, aligning it with a mouse compartment, and seating it fully.  \emph{Long-Horizon Battery Assembly} additionally requires relocating the mouse and battery to the workspace before insertion and moving the assembled mouse to a target region. We collect 100 demonstrations each for Pack Objects and Stack Cups, and 150 demonstrations each for Battery Insertion and Long-Horizon Battery Assembly. For each task, we evaluate the policy over 20 trials. We report binary success and a stage-wise score.  Success requires the final task condition to remain satisfied for at least three seconds.  The score is
the furthest valid milestone reached while all preceding milestones remain
satisfied; Pack Objects is order-independent and scores the number of objects
placed. Table~\ref{tab:task_rubric} gives the complete rubric. The deployment code will be released.

\input{tab/6_real_world_results}

Each rollout is autonomous: no human correction, object reset, or policy
restart is permitted before termination.  Hardware- or communication-only
interruptions are discarded and repeated, whereas policy-induced safety
terminations count as failures.  The recorded protocol uses 20 trials per
reported setting and the same predefined initial configurations, observation
inputs, language instructions, action limits, termination rules, and low-level
controller for all methods.

\subsubsection{Generalization Conditions}
\label{sec:app_real_world_generalization}

The generalization study changes one factor at a time: object appearance, task-irrelevant background, object geometry, or initial
spatial layout.  Camera placement, robot base pose, and all other factors are held fixed.  Changed objects, colors, backgrounds, and poses are excluded from task-specific robot training demonstrations. Geometry is not evaluated for the battery tasks because changing the insertion dimensions would change the task mechanics. The task-specific generalization settings are defined as follows:

\input{tab/6_real_world_generalization}

All shifted conditions use the same milestone rubric as the base setting.
Multi-factor shifts, if evaluated, must be reported separately rather than
mixed into the single-factor averages.

\subsubsection{Platform and Deployment}
\label{sec:app_real_world_deployment}

We conduct real-world experiments on a dual-arm Agilex Cobot Magic robot
platform. The robot is equipped with three Orbbec DaBai DC1 cameras, including one
front-facing camera and two wrist-mounted cameras, which provide visual
observations for policy inference. The robot action space consists of 14 dimensions, including the 6-DoF pose
commands for each arm and additional gripper control signals. All real-world experiments use the same inference pipeline without additional task-specific or model-specific modifications. The deployment framework and corresponding code will be released to facilitate reproducibility.


%% file: img/6_real_world_benchmark.tex
\begin{figure}[htbp]
    \centering
    \includegraphics[width=0.9\linewidth]{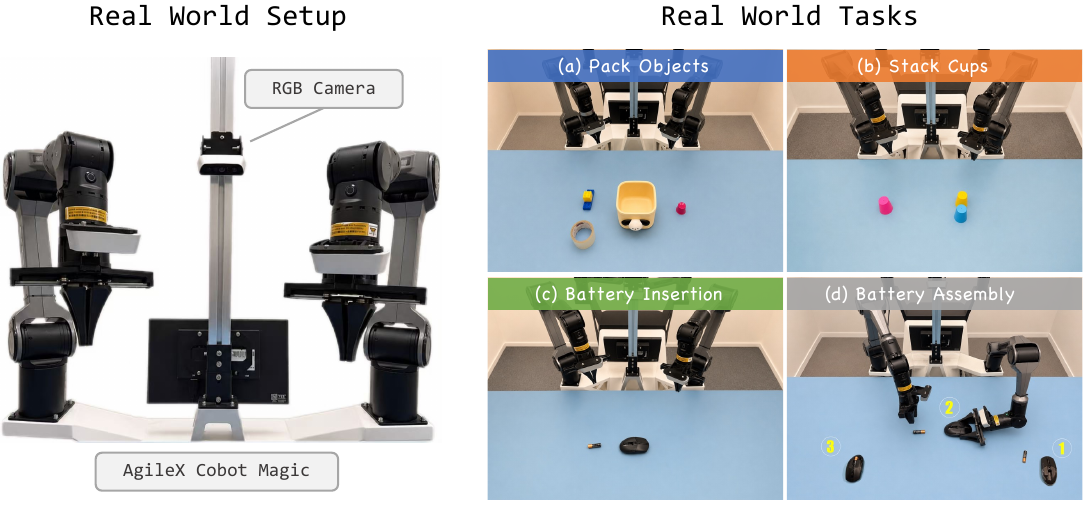}
    \caption{\textbf{Real-world tasks setup.} The left panel shows the robot platform and
    sensing setup; the remaining panels show base-setting executions of Pack
    Objects, Stack Cups, Battery Insertion, and Long-Horizon Battery Assembly.}
    \label{fig:real_world_suite}
\end{figure}

%% file: tab/6_real_world_results.tex
\begin{table*}[t]
    \centering
    \caption{\textbf{Stage-wise scores distinguish partial progress from task
    completion.} Each task receives a maximum score of 10.}
    \label{tab:task_rubric}
    \vspace{0.5em}

    \small
    \renewcommand{\arraystretch}{1.2}
    \setlength{\tabcolsep}{8pt}

    \begin{tabular}{
        p{0.15\textwidth}
        p{0.42\textwidth}
        c
    }
    \toprule
    \textbf{Task} & \textbf{Sub-goals} & \textbf{Score} \\
    \midrule

    \multirow{3}{*}{Pack Objects}
    & Pack one object    & 3.0 \\
    & Pack two objects   & 6.0 \\
    & Pack three objects & 10.0 \\
    \midrule

    \multirow{4}{*}{Stack Cups}
    & Place the left base cup    & 3.0 \\
    & Place the right base cup     & 5.0 \\
    & Place the top cup       & 7.5 \\
    & Achieve a stable stack & 10.0 \\
    \midrule

    \multirow{4}{*}{Battery Insertion}
    & Grasp the battery  & 3.0 \\
    & Align the battery with the slot and release it           & 5.0 \\
    & Press the battery into place & 7.5 \\
    & Release the assembled mouse         & 10.0 \\
    \midrule

    \multirow{4}{*}{Battery Assembly}
    & Move the parts to the workspace & 3.0 \\
    & Insert the battery               & 5.0 \\
    & Grasp the assembled mouse    & 7.5 \\
    & Move the mouse to the target place                 & 10.0 \\
    \bottomrule
    \end{tabular}
\end{table*}

%% file: tab/6_real_world_generalization.tex
\begin{itemize}
    \setlength{\itemsep}{8pt}
    \setlength{\topsep}{2pt}
    \setlength{\parsep}{0pt}

    \item \textit{Pack Objects:}
    \begin{itemize}
        \setlength{\itemsep}{2pt}
        \setlength{\topsep}{2pt}

        \item \textbf{Appearance:}
        Replace the original cup and LEGO bricks with instances of
        different colors while preserving their shapes and sizes.

        \item \textbf{Background:}
        Introduce task-irrelevant tabletop objects or vary the
        illumination without obstructing the target objects.

        \item \textbf{Geometry:}
        Change the way the LEGO pieces are connected, while keeping its function in the task, and change the size of the cup.

        \item \textbf{Spatial layout:}
        Relocate the tape roll and cup within predefined reachable
        regions while keeping the container fixed.
    \end{itemize}

    \item \textit{Stack Cups:}
    \begin{itemize}
        \setlength{\itemsep}{2pt}
        \setlength{\topsep}{2pt}

        \item \textbf{Appearance:}
        Replace the original cups with cups of different colors but
        identical geometry.

        \item \textbf{Background:}
        Add task-irrelevant clutter to the tabletop while keeping the
        manipulation area reachable.

        \item \textbf{Geometry:}
        Replace the original cups with a larger affordance-equivalent set
        that can still be arranged into the same tower structure.

        \item \textbf{Spatial layout:}
        Randomize the initial positions of all three cups within
        predefined regions while preserving task feasibility.
    \end{itemize}

    \item \textit{Battery Insertion:}
    \begin{itemize}
        \setlength{\itemsep}{2pt}
        \setlength{\topsep}{2pt}

        \item \textbf{Appearance:}
        Replace the original battery with a color-different but
        dimension-matched instance.

        \item \textbf{Background:}
        Add task-irrelevant tabletop clutter without occluding the mouse
        or battery.

        \item \textbf{Spatial layout:}
        Randomize the initial positions and orientations of the mouse and
        battery within predefined reachable regions.

        \item \textbf{Geometry:}
        Not evaluated, because changing the battery dimensions or
        compartment geometry would alter the insertion mechanics and
        therefore change the task itself.
    \end{itemize}

    \item \textit{Battery Assembly:}
    \begin{itemize}
        \setlength{\itemsep}{2pt}
        \setlength{\topsep}{2pt}

        \item \textbf{Appearance:}
        Replace the original battery with a color-different but
        dimension-matched instance.

        \item \textbf{Background:}
        Introduce task-irrelevant tabletop clutter while preserving the
        visibility of all task objects.

        \item \textbf{Spatial layout:}
        Vary the initial positions and orientations of the mouse and
        battery while retaining the required relocation, insertion, and
        final-placement stages.

        \item \textbf{Geometry:}
        Not evaluated, because changing the insertion geometry would
        alter the underlying task requirements.
    \end{itemize}
\end{itemize}

%% file: text/appendix/c_additional_results.tex
\section{Additional Results}
\label{sec:app_additional_results}

\input{text/appendix/c_1_additional_oilam_results}

\input{text/appendix/c_1_additional_dct_results}
\input{text/appendix/c_2_additional_policy_analyses}

\input{text/appendix/c_3_full_simulation_results}

\input{text/appendix/c_4_full_real_gen}
\clearpage
\input{text/appendix/case_study}

%% file: text/appendix/c_1_additional_oilam_results.tex
\subsection{Additional Latent Action Results}

\input{text/appendix/c_1_1_camera_interaction_readout_results}
\input{text/appendix/c_1_3_lary}
\input{text/appendix/c_1_2_qualitative_interaction_motion_visualization}

%% file: text/appendix/c_1_1_camera_interaction_readout_results.tex
\subsubsection{Camera--Interaction Readout Results}
\label{sec:app_factorization_results}

\paragraph{Factor-selective routing across datasets.}
As shown in Table~\ref{tab:factor_readout_all}, across EgoDex, EgoVerse, and Ego4D, WING-LAM consistently exhibits complementary factor routing: $\mathbf z_{\mathrm{int}}$ preserves interaction-related information while suppressing observer-motion information, whereas $\mathbf z_{\mathrm{cam}}$ captures observer-motion variation with minimal interaction information. Notably, for EgoDex, observer-motion readout is evaluated against the recorded 6D head pose rather than the affine motion proxy used to supervise WING-LAM during training. The reduced head-motion decodability from $\mathbf z_{\mathrm{int}}$ therefore indicates that suppression learned from the affine proxy transfers, at least partially, to genuine egocentric head motion. The student model further inherits this routing behavior from the trajectory-assisted teacher, demonstrating that the interaction-centric representation can be obtained without privileged inputs at inference time.

\paragraph{Probe robustness.}
We further evaluate the readout results using probes with different capacities,
including Ridge, MLP, and RFF. Although stronger probes improve absolute
readout scores for some baselines, the qualitative routing behavior of
WING-LAM remains unchanged, with complementary interaction and camera-motion
selectivity preserved across probe settings.

\input{tab/6_factor_selectivity_results}

%% file: tab/6_factor_selectivity_results.tex
\begin{table*}[t]
\centering
\small
\setlength{\tabcolsep}{10.0pt}
\renewcommand{\arraystretch}{1.05}
\caption{\textbf{WING-LAM provides complementary camera- and interaction-factor
readout across datasets.}
We report held-out $R^2$ for camera and interaction targets
using Ridge, MLP, and RFF probes.
WING-LAM's representations exhibit weak camera readout
while retaining strong interaction readout, whereas its camera
representation exhibits the opposite pattern.
Camera targets are ground-truth 6D head pose for EgoDex
and image-space camera-motion proxies for EgoVerse and Ego4D.
Each $\pm$ value denotes the half-width of a 95\%
grouped-bootstrap confidence interval.
Among interaction representations, \textbf{bold} and
\underline{underlined} values mark the lowest and second-lowest
$R^2_{cam}$ and the highest and second-highest $R^2_{int}$,
respectively, within each dataset and probe.}
\label{tab:factor_readout_all}
\resizebox{\textwidth}{!}{
\begin{tabular}{lcccccc}
\toprule
& \multicolumn{2}{c}{Ridge}
& \multicolumn{2}{c}{MLP}
& \multicolumn{2}{c}{RFF} \\
\cmidrule(lr){2-3}
\cmidrule(lr){4-5}
\cmidrule(lr){6-7}
Representation
& $R^2_{cam}$  & $R^2_{int}$ 
& $R^2_{cam}$  & $R^2_{int}$ 
& $R^2_{cam}$  & $R^2_{int}$  \\
\midrule
\multicolumn{7}{l}{
\textit{\textcolor{orange}{EgoDex}}
} \\
DreamDojo
& 0.32 \scalebox{0.7}{$\pm$ 0.04} & 0.12 \scalebox{0.7}{$\pm$ 0.08}
& 0.70 \scalebox{0.7}{$\pm$ 0.05} & 0.25 \scalebox{0.7}{$\pm$ 0.04}
& 0.54 \scalebox{0.7}{$\pm$ 0.08} & 0.17 \scalebox{0.7}{$\pm$ 0.02} \\
CD-LAM
& 0.54 \scalebox{0.7}{$\pm$ 0.03} & 0.12 \scalebox{0.7}{$\pm$ 0.00}
& 0.71 \scalebox{0.7}{$\pm$ 0.05} & 0.21 \scalebox{0.7}{$\pm$ 0.02}
& 0.56 \scalebox{0.7}{$\pm$ 0.08} & 0.17 \scalebox{0.7}{$\pm$ 0.03} \\
LAPA
& \underline{-0.23} \scalebox{0.7}{$\pm$ 0.05} & 0.02 \scalebox{0.7}{$\pm$ 0.07}
& \textbf{-0.27} \scalebox{0.7}{$\pm$ 0.04} & -0.03 \scalebox{0.7}{$\pm$ 0.01}
& \underline{-0.23} \scalebox{0.7}{$\pm$ 0.00} & 0.02 \scalebox{0.7}{$\pm$ 0.00} \\
VILLA-X
& -0.03 \scalebox{0.7}{$\pm$ 0.06} & -0.41 \scalebox{0.7}{$\pm$ 0.05}
& 0.07 \scalebox{0.7}{$\pm$ 0.02} & -0.32 \scalebox{0.7}{$\pm$ 0.06}
& -0.03 \scalebox{0.7}{$\pm$ 0.01} & -0.59 \scalebox{0.7}{$\pm$ 0.01} \\
UniVLA
& \textbf{-0.56} \scalebox{0.7}{$\pm$ 0.11} & -0.05 \scalebox{0.7}{$\pm$ 0.01}
& \underline{-0.03} \scalebox{0.7}{$\pm$ 0.05} & 0.01 \scalebox{0.7}{$\pm$ 0.06}
& \textbf{-0.56} \scalebox{0.7}{$\pm$ 0.21} & -0.05 \scalebox{0.7}{$\pm$ 0.00} \\
\rowcolor{gray!10}
teacher $\mathbf{z}_{\mathrm{int}}$
& -0.06 \scalebox{0.7}{$\pm$ 0.06} & \textbf{0.82} \scalebox{0.7}{$\pm$ 0.04}
& \underline{-0.03} \scalebox{0.7}{$\pm$ 0.04} & \textbf{0.74} \scalebox{0.7}{$\pm$ 0.09}
& -0.06 \scalebox{0.7}{$\pm$ 0.04} & \textbf{0.82} \scalebox{0.7}{$\pm$ 0.11} \\
\rowcolor{gray!10}
student ${\mathbf{z}}_{\mathrm{int}}$
& -0.02 \scalebox{0.7}{$\pm$ 0.08} & \underline{0.57} \scalebox{0.7}{$\pm$ 0.02}
& 0.00 \scalebox{0.7}{$\pm$ 0.03} & \underline{0.48} \scalebox{0.7}{$\pm$ 0.03}
& -0.02 \scalebox{0.7}{$\pm$ 0.03} & \underline{0.54} \scalebox{0.7}{$\pm$ 0.07} \\
\midrule
\rowcolor{orange!13}
teacher $\mathbf{z}_{\mathrm{cam}}$
& 0.93 \scalebox{0.7}{$\pm$ 0.03} & -0.04 \scalebox{0.7}{$\pm$ 0.05}
& 0.93 \scalebox{0.7}{$\pm$ 0.06} & -0.10 \scalebox{0.7}{$\pm$ 0.06}
& 0.92 \scalebox{0.7}{$\pm$ 0.08} & -0.04 \scalebox{0.7}{$\pm$ 0.00} \\
\midrule
\multicolumn{7}{l}{
\textit{\textcolor{green}{EgoVerse}}
} \\
DreamDojo
& 0.26 \scalebox{0.7}{$\pm$ 0.01} & -0.22 \scalebox{0.7}{$\pm$ 0.02}
& 0.74 \scalebox{0.7}{$\pm$ 0.08} & 0.45 \scalebox{0.7}{$\pm$ 0.04}
& 0.65 \scalebox{0.7}{$\pm$ 0.05} & -0.16 \scalebox{0.7}{$\pm$ 0.03} \\
CD-LAM
& 0.34 \scalebox{0.7}{$\pm$ 0.08} & -0.14 \scalebox{0.7}{$\pm$ 0.02}
& 0.75 \scalebox{0.7}{$\pm$ 0.03} & 0.46 \scalebox{0.7}{$\pm$ 0.02}
& 0.63 \scalebox{0.7}{$\pm$ 0.08} & -0.06 \scalebox{0.7}{$\pm$ 0.03} \\
LAPA
& -0.08 \scalebox{0.7}{$\pm$ 0.06} & -0.33 \scalebox{0.7}{$\pm$ 0.00}
& -0.25 \scalebox{0.7}{$\pm$ 0.01} & -0.27 \scalebox{0.7}{$\pm$ 0.07}
& -0.08 \scalebox{0.7}{$\pm$ 0.00} & -0.33 \scalebox{0.7}{$\pm$ 0.00} \\
VILLA-X
& 0.26 \scalebox{0.7}{$\pm$ 0.04} & -0.19 \scalebox{0.7}{$\pm$ 0.07}
& 0.07 \scalebox{0.7}{$\pm$ 0.05} & 0.02 \scalebox{0.7}{$\pm$ 0.02}
& 0.26 \scalebox{0.7}{$\pm$ 0.02} & 0.04 \scalebox{0.7}{$\pm$ 0.03} \\
UniVLA
& \textbf{-0.52} \scalebox{0.7}{$\pm$ 0.09} & -0.14 \scalebox{0.7}{$\pm$ 0.05}
& \textbf{-0.47} \scalebox{0.7}{$\pm$ 0.15} & -0.03 \scalebox{0.7}{$\pm$ 0.01}
& \underline{-0.52} \scalebox{0.7}{$\pm$ 0.01} & -0.22 \scalebox{0.7}{$\pm$ 0.08} \\
\rowcolor{gray!10}
teacher $\mathbf{z}_{\mathrm{int}}$
& \underline{-0.28} \scalebox{0.7}{$\pm$ 0.06} & \textbf{0.79} \scalebox{0.7}{$\pm$ 0.06}
& \underline{-0.27} \scalebox{0.7}{$\pm$ 0.03} & \textbf{0.79} \scalebox{0.7}{$\pm$ 0.04}
& -0.28 \scalebox{0.7}{$\pm$ 0.00} & \textbf{0.83} \scalebox{0.7}{$\pm$ 0.06} \\
\rowcolor{gray!10}
student ${\mathbf{z}}_{\mathrm{int}}$
& -0.23 \scalebox{0.7}{$\pm$ 0.07} & \underline{0.74} \scalebox{0.7}{$\pm$ 0.03}
& \underline{-0.27} \scalebox{0.7}{$\pm$ 0.02} & \underline{0.76} \scalebox{0.7}{$\pm$ 0.07}
& \textbf{-0.61} \scalebox{0.7}{$\pm$ 0.04} & \underline{0.69} \scalebox{0.7}{$\pm$ 0.04} \\
\midrule
\rowcolor{green!20}
teacher $\mathbf{z}_{\mathrm{cam}}$
& 0.97 \scalebox{0.7}{$\pm$ 0.02} & -0.35 \scalebox{0.7}{$\pm$ 0.05}
& 0.96 \scalebox{0.7}{$\pm$ 0.01} & -0.09 \scalebox{0.7}{$\pm$ 0.00}
& 0.97 \scalebox{0.7}{$\pm$ 0.01} & -0.64 \scalebox{0.7}{$\pm$ 0.16} \\
\midrule
\multicolumn{7}{l}{
\textit{\textcolor{blue}{Ego4D}}
} \\
DreamDojo
& 0.47 \scalebox{0.7}{$\pm$ 0.00} & 0.15 \scalebox{0.7}{$\pm$ 0.04}
& 0.72 \scalebox{0.7}{$\pm$ 0.03} & 0.08 \scalebox{0.7}{$\pm$ 0.07}
& 0.59 \scalebox{0.7}{$\pm$ 0.05} & 0.15 \scalebox{0.7}{$\pm$ 0.00} \\
CD-LAM
& 0.28 \scalebox{0.7}{$\pm$ 0.00} & 0.15 \scalebox{0.7}{$\pm$ 0.04}
& 0.77 \scalebox{0.7}{$\pm$ 0.04} & 0.09 \scalebox{0.7}{$\pm$ 0.07}
& 0.54 \scalebox{0.7}{$\pm$ 0.06} & 0.15 \scalebox{0.7}{$\pm$ 0.00} \\
LAPA
& \textbf{-0.03} \scalebox{0.7}{$\pm$ 0.06} & -0.11 \scalebox{0.7}{$\pm$ 0.06}
& \underline{0.02} \scalebox{0.7}{$\pm$ 0.05} & -0.10 \scalebox{0.7}{$\pm$ 0.03}
& \textbf{-0.02} \scalebox{0.7}{$\pm$ 0.00} & -0.12 \scalebox{0.7}{$\pm$ 0.02} \\
VILLA-X
& 0.34 \scalebox{0.7}{$\pm$ 0.08} & -0.06 \scalebox{0.7}{$\pm$ 0.07}
& 0.28 \scalebox{0.7}{$\pm$ 0.06} & -0.35 \scalebox{0.7}{$\pm$ 0.05}
& 0.34 \scalebox{0.7}{$\pm$ 0.00} & -0.08 \scalebox{0.7}{$\pm$ 0.00} \\
UniVLA
& 0.15 \scalebox{0.7}{$\pm$ 0.07} & -0.11 \scalebox{0.7}{$\pm$ 0.07}
& 0.16 \scalebox{0.7}{$\pm$ 0.05} & -0.06 \scalebox{0.7}{$\pm$ 0.01}
& 0.15 \scalebox{0.7}{$\pm$ 0.00} & -0.11 \scalebox{0.7}{$\pm$ 0.00} \\
\rowcolor{gray!10}
teacher $\mathbf{z}_{\mathrm{int}}$
& \underline{-0.02} \scalebox{0.7}{$\pm$ 0.07} & \textbf{0.66} \scalebox{0.7}{$\pm$ 0.00}
& \textbf{-0.03} \scalebox{0.7}{$\pm$ 0.01} & \textbf{0.64} \scalebox{0.7}{$\pm$ 0.06}
& \textbf{-0.02} \scalebox{0.7}{$\pm$ 0.02} & \textbf{0.71} \scalebox{0.7}{$\pm$ 0.05} \\
\rowcolor{gray!10}
student ${\mathbf{z}}_{\mathrm{int}}$
& 0.07 \scalebox{0.7}{$\pm$ 0.09} & \underline{0.42} \scalebox{0.7}{$\pm$ 0.08}
& 0.04 \scalebox{0.7}{$\pm$ 0.04} & \underline{0.38} \scalebox{0.7}{$\pm$ 0.05}
& \underline{0.07} \scalebox{0.7}{$\pm$ 0.00} & \underline{0.42} \scalebox{0.7}{$\pm$ 0.02} \\
\midrule
\rowcolor{blue!15}
teacher $\mathbf{z}_{\mathrm{cam}}$
& 0.93 \scalebox{0.7}{$\pm$ 0.03} & -0.11 \scalebox{0.7}{$\pm$ 0.01}
& 0.95 \scalebox{0.7}{$\pm$ 0.02} & -0.09 \scalebox{0.7}{$\pm$ 0.06}
& 0.93 \scalebox{0.7}{$\pm$ 0.02} & -0.09 \scalebox{0.7}{$\pm$ 0.07} \\
\bottomrule
\end{tabular}
}
\end{table*}

%% file: text/appendix/c_1_3_lary.tex
\subsubsection{Detailed LARYBench Results}
Table~\ref{tab:latent_action_comparison} reports the detailed LARYBench~\citep{larybench} action-classification results across Ego4D, EPIC-KITCHENS, HoloAssist, and EgoDex. All latent-action encoders are frozen and evaluated under the same classification protocol. WING-LAM achieves the highest average accuracy across the four datasets, indicating that observer debiasing preserves action-related semantic information in the learned latent representation. These per-dataset results provide the details corresponding to the aggregate LARYBench result reported in the main text.

\input{tab/lary_detail}

%% file: tab/lary_detail.tex
\begin{table}[t]
\centering
\caption{Comparison of different latent action models across four datasets.}
\label{tab:latent_action_comparison}
\setlength{\tabcolsep}{2pt}
\small
\renewcommand{\arraystretch}{1.1}
\begin{tabular}{lcccccc}
\toprule
\textbf{Model} 
& \textbf{WING-LAM} 
& \textbf{DreamDojo} 
& \textbf{CD-LAM} 
& \textbf{LAPA} 
& \textbf{UniVLA} 
& \textbf{VILLA-X} \\
\midrule
Ego4D 
& 64.13 & 60.92 & 60.67 & 21.02 & 15.22 & 14.80 \\
EPIC-KITCHENS 
& 40.57 & 40.65 & 41.19 & 17.58 & 9.95 & 10.35 \\
HoloAssist 
& 48.41 & 46.23 & 48.95 & 22.40 & 15.84 & 20.69 \\
EgoDex 
& 68.77 & 65.92 & 62.95 & 22.56 & 17.43 & 17.20 \\
\midrule
\textbf{Avg.} 
& \textbf{55.47} 
& 53.43 
& 53.44 
& 20.89 
& 14.61 
& 15.76 \\
\bottomrule
\end{tabular}
\end{table}

%% file: text/appendix/c_1_2_qualitative_interaction_motion_visualization.tex
\subsubsection{Qualitative Visualization of Interaction Motion}
\label{sec:app_qualitative_interaction_motion}

To further examine what motion patterns are captured by WING-LAM, we conduct a qualitative analysis by feeding its latent actions into the teacher decoder to reconstruct motion trajectories. Specifically, we extract trajectories from three consecutive clips and concatenate them sequentially to visualize longer-term motion patterns. The decoded trajectories are then visualized across diverse egocentric interactions in Figure~\ref{fig:oilam_interaction_motion}. Across all examples, the decoded trajectories
consistently concentrate on moving hands, forearms, manipulated objects, and their surrounding interaction regions. Their dominant directions broadly follow the local changes between the two frames across diverse activities,
including object transfer, tool use, surface wiping, and food preparation.
Comparatively few interaction responses appear in background regions, even
when the egocentric viewpoint changes noticeably. These examples provide
qualitative evidence that the WING-LAM latent preferentially
preserves interaction-relevant motion rather than camera-induced global scene motion.

\FloatBarrier
\renewcommand{\floatpagefraction}{0.80}
\input{img/d_6_qualitative_interaction_motion}

%% file: img/d_6_qualitative_interaction_motion.tex
\begin{figure}[p]
    \centering
    \includegraphics[height=0.60\textheight,keepaspectratio]{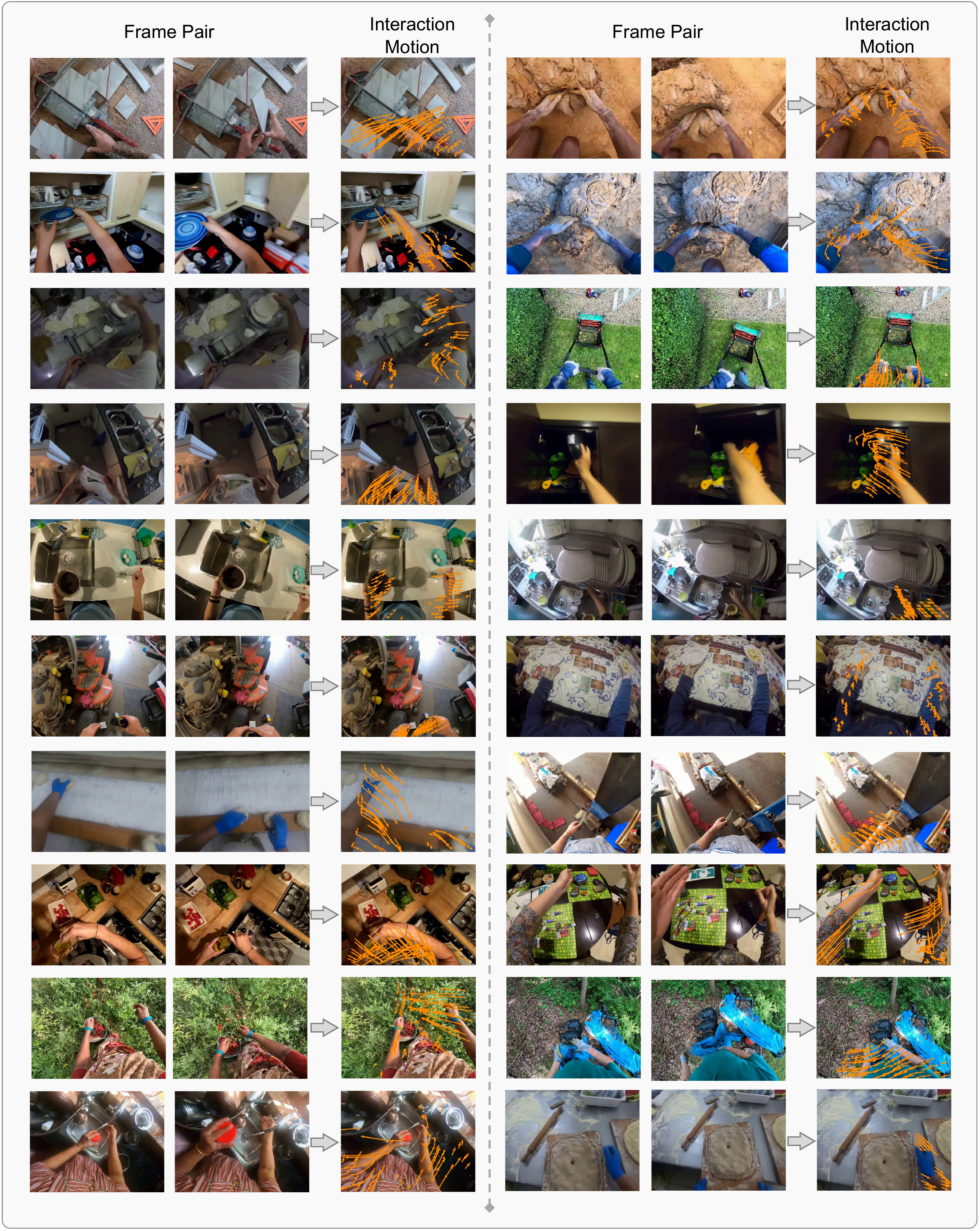}
    \caption{\textbf{Qualitative visualization of interaction motion decoded
    from WING-LAM.}}
    \label{fig:oilam_interaction_motion}
\end{figure}

%% file: text/appendix/c_1_additional_dct_results.tex
\subsection{Additional Spectral Results}

\subsubsection{Cross-Embodiment Spectral Similarity Analysis}
\label{sec:app_spectral_similarity_analysis}
Experimental settings are detailed in
Appendix~\ref{sec:app_pair_construction}.
As shown in Figure~\ref{fig:lam}(d), low-frequency
latents achieve the largest separation between semantically matched
and mismatched human--robot interaction pairs, with a matched--mismatched
gap of $0.51$ (95\% bootstrap CI: $[0.48,0.53]$).
In contrast, full-spectrum and high-frequency representations obtain
substantially smaller gaps of $0.26$ ($[0.22,0.27]$) and $0.05$
($[0.03,0.10]$), respectively. Importantly, repeating the same analysis with the pretrained DreamDojo LAM yields the same qualitative ordering, with low-frequency components exhibiting the largest matched--mismatched separation, suggesting that this frequency-selective correspondence is not specific to WING-LAM. To verify that this advantage is not merely driven by the dominant
energy of the zeroth component, we further exclude $q_0$ and evaluate
the remaining low-frequency dynamics. The $q_1:q_3$ components
retain a matched-pair similarity of $0.42$, compared with only $0.08$
for mismatched pairs. In contrast, extending
the representation to $q_1:q_7$ decreases the matched similarity
to $0.31$ while increasing the mismatched similarity to $0.14$. This shows that the low-frequency
correspondence persists even after removing the zeroth component, and
that incorporating higher-frequency components weakens the separation
between semantically matched and mismatched interactions.

Overall, these results suggest that low-frequency latent components preserve interaction semantics with stronger cross-embodiment correspondence, whereas higher-frequency components contain variation that is less consistently shared across human and robot executions.

\subsubsection{Spectral Bandwidth Gap Between Egocentric and Robot Data}
\input{img/bandwidth}
We additionally examine the spectral distributions of latent-action trajectories in egocentric and robot data. 
For this dataset-level bandwidth analysis, we use the pretrained DreamDojo LAM rather than WING-LAM. 
DreamDojo has been trained on both egocentric and robot data, making it better suited for characterizing differences in the temporal frequency distributions of the two domains without strongly confounding the analysis with domain-specific encoder generalization. 
In contrast, WING-LAM is trained specifically on egocentric videos; applying it directly to robot observations could make the estimated spectral distribution partly reflect its out-of-domain generalization behavior rather than the intrinsic temporal characteristics of robot trajectories. 
Therefore, this analysis is intended to provide a more domain-balanced characterization of the temporal bandwidth gap between egocentric and robot interaction data, rather than to evaluate the cross-embodiment correspondence of the WING-LAM representation itself.

Using the pretrained DreamDojo LAM, we encode ego and robot trajectories with the same physical temporal sampling interval of \(0.2\,\mathrm{s}\), forming \(H=8\) latent-action sequences that span \(1.6\,\mathrm{s}\). We then apply DCT-II along the temporal dimension.
We measure the fraction of spectral energy retained by the lowest $K=4$ frequency components. 
The spectral analysis reveals a noticeable difference in temporal bandwidth between egocentric and robot trajectories, with egocentric trajectories exhibiting broader frequency distributions than robot trajectories. 
Robot trajectories are substantially more concentrated in low-frequency components, whereas egocentric trajectories contain relatively more high-frequency variation. We observe the same qualitative trend when repeating the analysis with WING-LAM, indicating that this bandwidth difference is not specific to the DreamDojo representation.
This observation provides additional motivation for examining cross-embodiment correspondence separately across frequency bands.

\input{img/cluster}
\FloatBarrier
\renewcommand{\floatpagefraction}{0.50}
\clearpage
\subsubsection{Action Semantics Across Frequency Bands}
\label{sec:app_cluster}
We examine how action-semantic information is organized across temporal
frequency bands on the official Ego4D validation split of
LARYBench~\citep{larybench}. For each video, we apply temporal DCT to its
8-step WING-LAM latent trajectory and construct low-frequency ($q_{0:3}$),
high-frequency ($q_{4:7}$), and full-spectrum ($q_{0:7}$) representations.
Each representation is reduced to 32 dimensions using PCA and evaluated
with the same clustering and visualization procedures. We measure the
agreement between the resulting clusters and the official action labels
using normalized mutual information (NMI). The low-frequency representation
achieves substantially higher NMI ($0.55\pm0.02$) than the full-spectrum
($0.31\pm0.01$) and high-frequency ($0.17\pm0.05$) representations,
indicating that action semantics are more clearly organized in the
low-frequency latent subspace.

%% file: img/bandwidth.tex
\begin{wrapfigure}{r}{0.4\textwidth}
    \centering
    \vspace{-8pt}
    \includegraphics[width=0.4\textwidth]{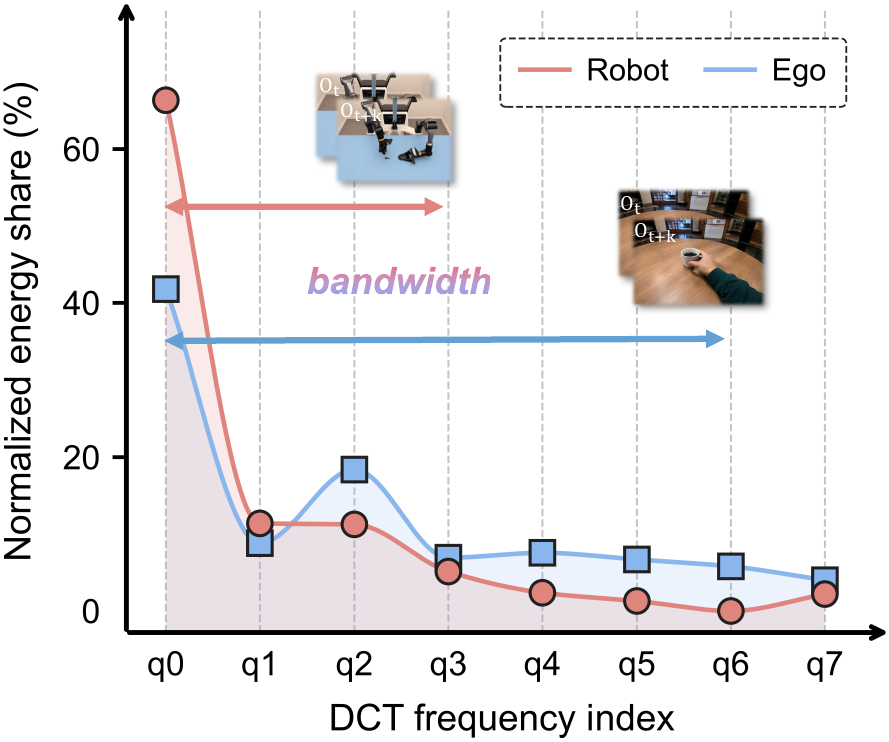}
    \vspace{-10pt}
    \caption{
    \textbf{Bandwidth gap between egocentric and robot data.}
    }
    \label{fig:bandwidth_gap}
    \vspace{-10pt}
\end{wrapfigure}

%% file: img/cluster.tex
\begin{figure}[p]
    \centering
    \includegraphics[width=0.85\linewidth]{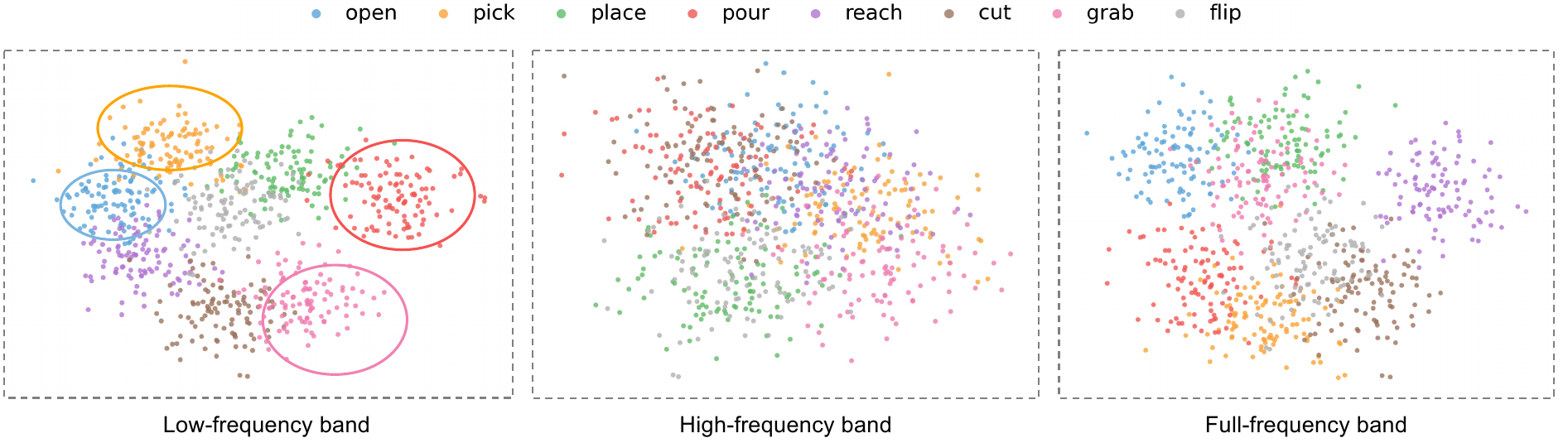}
    \caption{\textbf{Action-semantic organization across temporal frequency bands.}}
    \label{fig:cluster}
\end{figure}

%% file: text/appendix/c_2_additional_policy_analyses.tex
\vspace{-1em}

\subsection{Additional Policy Analyses}

\subsubsection{Latency Measurement}
\input{tab/latency}
We measure policy inference latency on a single NVIDIA RTX 4090 GPU
and report both the median and 95th-percentile latency. For evaluation, we follow the real-world deployment setting,
using an action chunk size of 32 and 10 denoising steps.
As shown in Table~\ref{tab:latency}, our model incurs higher
latency than lightweight policy baselines such as GR00T and $\pi_{0.5}$,
due to the additional computation introduced by the dense world-modeling
architecture. Cosmos-Policy~\citep{cosmos} and
Motus~\citep{motus} require 845.6\,ms and 1532.8\,ms per policy inference,
respectively.
Thus, although inference latency remains a limitation of the current
implementation, our model remains computationally competitive with
generative policies of comparable complexity.

\subsubsection{Ablation on Denoising Steps}
\input{tab/denoising_step}

The number of denoising steps \(N\) controls the amount of iterative
refinement at inference time. We conduct this ablation on RoboTwin~2.0
and fix the action chunk size to 32 for all settings. As shown in
Table~\ref{tab:denoising_steps}, single-step inference achieves an
average success rate of only \(80.46\%\), whereas increasing \(N\) to
2 substantially improves it to \(93.20\%\), with performance remaining
relatively stable for larger \(N\). Among the evaluated settings,
\(N=12\) achieves the highest average success rate, with \(94.56\%\)
and \(93.04\%\) on the Clean and Randomized settings, respectively,
and is therefore used for the RoboTwin~2.0 result reported in the main
table. Notably, nearby settings perform similarly (e.g.,
\(93.55\%\) success rate at \(N=10\)), indicating that the overall
performance is not sensitive to the precise choice of \(N\) once
sufficient iterative refinement is used.

\WFclear
\clearpage
\subsubsection{Egocentric Pretraining Efficiency}
\label{sec:app_pre_eff}
\input{img/pretrain_eff}
We investigate how the extent of egocentric pretraining affects subsequent robot policy optimization. All pre-training checkpoints are trained on the same 200h egocentric video corpus, and we select checkpoints obtained after 15k, 30k, 45k, and 60k pretraining steps. Each checkpoint is then used to initialize the robot policy, which is trained on the \emph{Pack Objects} task under the same optimization setting. We use a held-out split of the \emph{Pack Objects} data as the validation set and report the total validation loss of the full model throughout robot training. As shown in Fig.~\ref{fig:ego_pretraining_efficiency}, longer egocentric pretraining consistently leads to lower validation loss and faster convergence, indicating that Stage-I pretraining provides a stronger initialization and improves the optimization efficiency of downstream robot policy learning.

\subsubsection{Downstream Data Efficiency}
\label{sec:app_data_eff}
To further evaluate the effectiveness of egocentric pretraining, we study how performance changes as the amount of downstream robot training data is reduced. We compare models initialized with and without our pretraining under 100\%, 50\%, and 25\% of the downstream training data on LIBERO and the real-world standard setting.
\begin{figure}[htbp]
\centering
\includegraphics[width=0.75\linewidth]{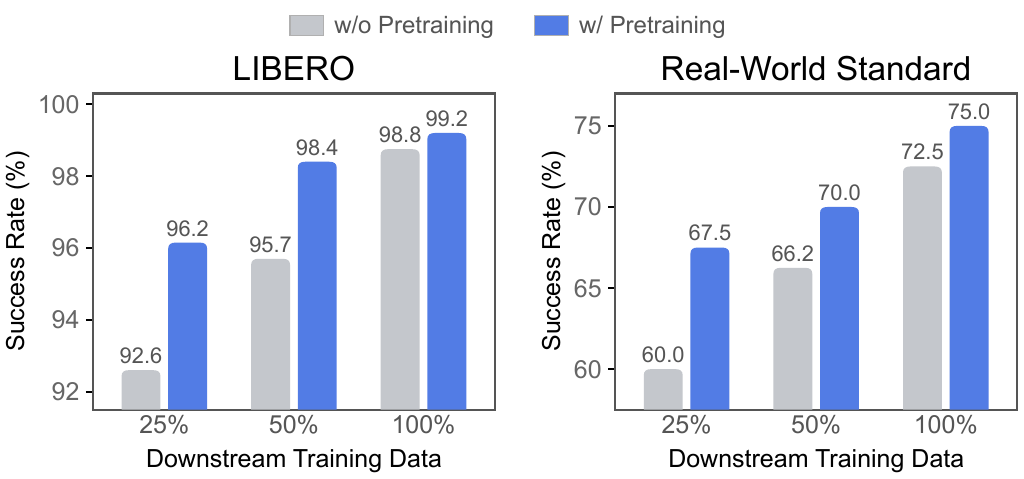}
\caption{\textbf{Downstream data efficiency.} We compare downstream performance with and without egocentric pretraining using different fractions of robot training data. Pretraining consistently improves success rates on both LIBERO and the real-world standard setting, with larger gains as the amount of downstream data decreases.}
\label{fig:data_eff}
\end{figure}

As shown in Figure~\ref{fig:data_eff}, egocentric pretraining consistently improves downstream performance on both LIBERO and real-world manipulation. More importantly, the advantage becomes increasingly pronounced as the amount of downstream robot data decreases, with the largest improvement observed in the 25\% real-world data setting. These results demonstrate that our pretraining provides transferable knowledge that improves robot-data efficiency, particularly in low-data regimes.

\newpage

%% file: tab/latency.tex
\begin{wraptable}{r}{0.4\columnwidth}
\vspace{-0.5\baselineskip}
\vspace{-1em}
\centering
\caption{
\textbf{Policy inference latency on a single NVIDIA RTX 4090 GPU.}
We report the median and P95 latency over repeated runs.
All values are in milliseconds.
}
\label{tab:latency}

\small
\setlength{\tabcolsep}{6pt}
\renewcommand{\arraystretch}{1.0}

\begin{tabular}{lcc}
\toprule
Method & Median & P95 \\
\midrule

$\pi_{0.5}$
& 90.6 & 93.8 \\

GR00T
& 63.4 & 67.3 \\

FastWAM
& 229.0 & 234.2 \\

FastWAM-Joint
& 510.1 & 515.4 \\

Cosmos-Policy
& 845.6 & 852.3 \\

Motus
& 1532.8 & 1551.6 \\

\textbf{Ours}
& \textbf{500.6} & \textbf{517.1} \\
\bottomrule
\end{tabular}

\vspace{-0.5\baselineskip}
\end{wraptable}

%% file: tab/denoising_step.tex
\begin{wraptable}{r}{0.4\textwidth}
\vspace{-1em}
\centering
\caption{\textbf{Denoising steps.} RoboTwin success rate (\%).}
\vspace{-8pt}
\label{tab:denoising_steps}
\setlength{\tabcolsep}{8pt}
\renewcommand{\arraystretch}{1}
\begin{tabular}{c|cc}
\toprule
$N_{\mathrm{denoise}}$ & \textbf{Clean} & \textbf{Rand.} \\
\midrule
1  & 80.40 & 80.52 \\
2  & 93.58 & 92.82 \\
4  & 93.78 & 92.30 \\
6  & 93.78 & 92.82 \\
8  & 94.06 & 92.56 \\
10 & 94.14 & 92.96 \\
12 & \textbf{94.56} & \textbf{93.04} \\
14 & 94.28 & 92.92 \\
16 & 94.34 & 92.92 \\
\bottomrule
\end{tabular}
\vspace{-6pt}
\end{wraptable}

%% file: img/pretrain_eff.tex
\begin{wrapfigure}{r}{0.45\textwidth}
    \centering
    \vspace{-2em}
    \includegraphics[width=0.3\textwidth]{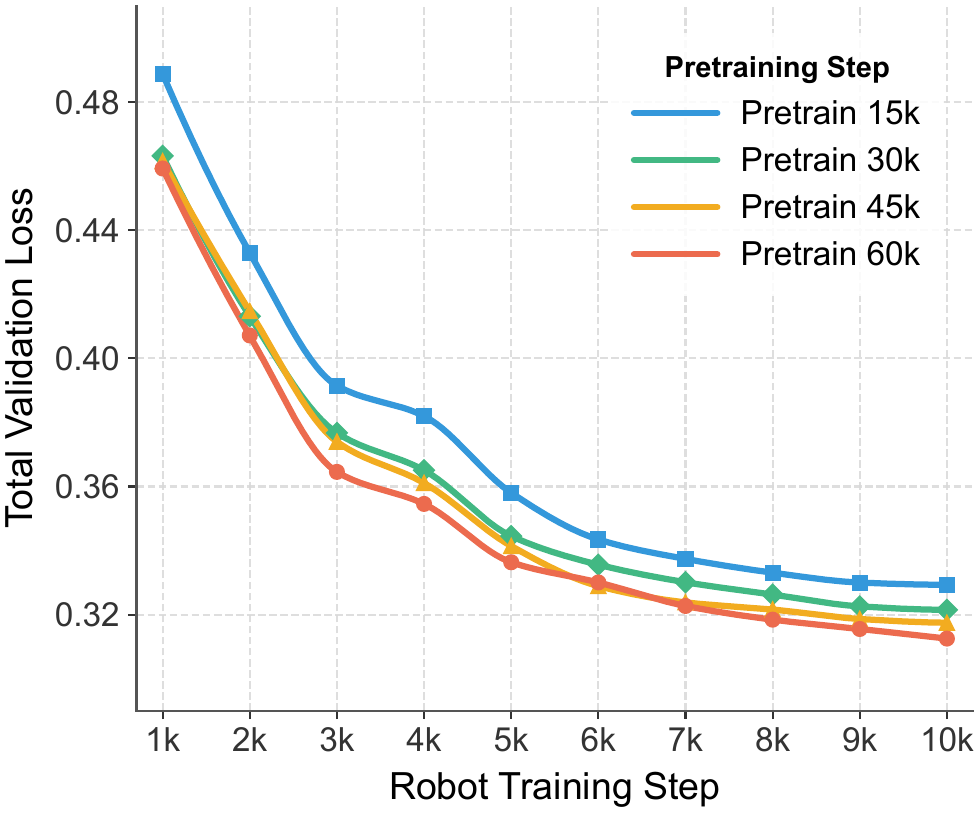}
    \caption{
    \textbf{Egocentric pretraining improves downstream optimization efficiency.}
    }
    \label{fig:ego_pretraining_efficiency}
    \vspace{-10pt}
\end{wrapfigure}

%% file: text/appendix/c_3_full_simulation_results.tex
\subsection{Full Simulation Results}
\input{text/appendix/c_3_2_robocasa_per_task_results}
\input{text/appendix/c_3_1_robotwin_per_task_results}

%% file: text/appendix/c_3_2_robocasa_per_task_results.tex
\subsubsection{RoboCasa--GR1 Per-Task Results}
\label{sec:app_robocasa_full_results}

Table~\ref{tab:robocasa_per_task} reports all 24 tabletop tasks. The FastWAM aggregate is our own reproduction.

\input{tab/i_2_robocasa_per_task}

%% file: tab/i_2_robocasa_per_task.tex
\begin{table}[htbp]
\centering
\caption{\textbf{Full RoboCasa--GR1 tabletop results.} Per-task success rate
(\%). Published per-task baseline entries are transcribed from the
corresponding method papers. $\dagger$ denotes our FastWAM reproduction.}
\label{tab:robocasa_per_task}
\scriptsize
\setlength{\tabcolsep}{2.6pt}
\renewcommand{\arraystretch}{1.02}
\resizebox{\textwidth}{!}{%
\begin{tabular}{lcccccc}
\toprule
Task & GR00T-N1.6 & StarVLA-OFT & DiT4DiT & LDA-1B & FastWAM$^\dagger$ & Ours \\
\midrule
BottleToCabinetClose
    & 51.5 & 30.0 & 48.0 & 76.0 & 76.0 & 66.0 \\
CanToDrawerClose
    & 13.0 & 76.0 & 74.0 & 71.0 & 78.0 & 90.0 \\
CupToDrawerClose
    & 8.5 & 44.0 & 52.0 & 41.0 & 22.0 & 44.0 \\
MilkToMicrowaveClose
    & 14.0 & 44.0 & 50.0 & 52.0 & 68.0 & 58.0 \\
PotatoToMicrowaveClose
    & 41.5 & 32.0 & 36.0 & 41.0 & 46.0 & 38.0 \\
WineToCabinetClose
    & 16.5 & 36.0 & 42.0 & 57.0 & 70.0 & 62.0 \\
FromCuttingboardToBasket
    & 58.0 & 50.0 & 52.0 & 65.0 & 38.0 & 52.0 \\
FromCuttingboardToCardboardbox
    & 46.5 & 40.0 & 48.0 & 69.0 & 48.0 & 48.0 \\
FromCuttingboardToPan
    & 68.5 & 70.0 & 76.0 & 75.0 & 82.0 & 74.0 \\
FromCuttingboardToPot
    & 65.0 & 54.0 & 62.0 & 61.0 & 74.0 & 68.0 \\
FromCuttingboardToTieredbasket
    & 46.5 & 38.0 & 50.0 & 51.0 & 30.0 & 44.0 \\
FromPlacematToBasket
    & 58.5 & 32.0 & 50.0 & 53.0 & 44.0 & 68.0 \\
FromPlacematToBowl
    & 57.5 & 58.0 & 56.0 & 55.0 & 44.0 & 60.0 \\
FromPlacematToPlate
    & 63.0 & 52.0 & 32.0 & 59.0 & 52.0 & 52.0 \\
FromPlacematToTieredshelf
    & 28.5 & 24.0 & 18.0 & 24.0 & 36.0 & 36.0 \\
FromPlateToBowl
    & 57.0 & 60.0 & 56.0 & 53.0 & 44.0 & 68.0 \\
FromPlateToCardboardbox
    & 43.5 & 50.0 & 58.0 & 43.0 & 40.0 & 42.0 \\
FromPlateToPan
    & 51.0 & 66.0 & 68.0 & 55.0 & 42.0 & 50.0 \\
FromPlateToPlate
    & 78.7 & 68.0 & 58.0 & 61.0 & 72.0 & 78.0 \\
FromTrayToCardboardbox
    & 51.5 & 44.0 & 38.0 & 65.0 & 40.0 & 70.0 \\
FromTrayToPlate
    & 71.0 & 56.0 & 56.0 & 63.0 & 68.0 & 62.0 \\
FromTrayToPot
    & 64.5 & 62.0 & 54.0 & 55.0 & 50.0 & 64.0 \\
FromTrayToTieredbasket
    & 57.0 & 54.0 & 46.0 & 51.0 & 56.0 & 44.0 \\
FromTrayToTieredshelf
    & 31.5 & 30.0 & 38.0 & 33.0 & 26.0 & 46.0 \\
\midrule
Average
    & 47.6 & 48.8 & 50.8 & 55.4 & 51.9 & 57.7 \\
\bottomrule
\end{tabular}}
\end{table}

%% file: text/appendix/c_3_1_robotwin_per_task_results.tex
\subsubsection{RoboTwin~2.0 Per-Task Results}
\label{sec:app_robotwin_full_results}

Table~\ref{tab:robotwin_per_task} provides the complete 50-task
clean/randomized breakdown. 

\input{tab/i_1_robotwin_per_task_a}

%% file: tab/i_1_robotwin_per_task_a.tex
\begin{table*}[htbp]
\centering
\caption{\textbf{RoboTwin~2.0 per-task success rates.}
Clean and randomized success rates (\%) across 50 tasks, evaluated
over 100 trials per task under each setting.
}
\label{tab:robotwin_per_task}

\scriptsize
\setlength{\tabcolsep}{0.8pt}
\renewcommand{\arraystretch}{1.12}

\begin{tabular*}{0.99\textwidth}{@{\extracolsep{\fill}}lcccccccccccccc@{}}
\toprule

& \multicolumn{2}{c}{$\pi_{0.5}$}
& \multicolumn{2}{c}{FastWAM}
& \multicolumn{2}{c}{GigaWorld-Policy}
& \multicolumn{2}{c}{LingBot-VA}
& \multicolumn{2}{c}{Motus}
& \multicolumn{2}{c}{LaWAM}
& \multicolumn{2}{c}{Ours} \\

\cmidrule(lr){2-3}
\cmidrule(lr){4-5}
\cmidrule(lr){6-7}
\cmidrule(lr){8-9}
\cmidrule(lr){10-11}
\cmidrule(lr){12-13}
\cmidrule(lr){14-15}

Task
& Clean & Rand.
& Clean & Rand.
& Clean & Rand.
& Clean & Rand.
& Clean & Rand.
& Clean & Rand.
& Clean & Rand. \\

\midrule

Adjust Bottle
& 100 & 99 & 100 & 100 & 100 & 100 & 90 & 94
& 89 & 93 & 100 & 100 & 100 & 100 \\

Beat Block Hammer
& 96 & 93 & 99 & 97 & 86 & 86 & 96 & 98
& 95 & 88 & 90 & 93 & 99 & 98 \\

Blocks Ranking RGB
& 92 & 85 & 100 & 100 & 92 & 96 & 99 & 98
& 99 & 97 & 97 & 100 & 99 & 99 \\

Blocks Ranking Size
& 49 & 26 & 94 & 98 & 44 & 48 & 94 & 96
& 75 & 63 & 93 & 89 & 87 & 82 \\

Click Alarmclock
& 98 & 89 & 100 & 100 & 100 & 100 & 99 & 100
& 100 & 100 & 100 & 100 & 100 & 100 \\

Click Bell
& 99 & 66 & 100 & 100 & 100 & 100 & 100 & 100
& 100 & 100 & 100 & 100 & 100 & 100 \\

Dump Bin Bigbin
& 92 & 97 & 97 & 96 & 92 & 100 & 89 & 96
& 95 & 91 & 97 & 95 & 95 & 98 \\

Grab Roller
& 100 & 100 & 100 & 100 & 100 & 100 & 100 & 100
& 100 & 100 & 100 & 100 & 100 & 99 \\

Handover Block
& 66 & 57 & 95 & 81 & 80 & 80 & 99 & 78
& 86 & 73 & 96 & 87 & 92 & 87 \\

Handover Mic
& 98 & 97 & 99 & 100 & 72 & 72 & 94 & 96
& 78 & 63 & 93 & 98 & 100 & 99 \\

Hanging Mug
& 18 & 17 & 58 & 62 & 16 & 12 & 40 & 28
& 38 & 38 & 51 & 43 & 68 & 60 \\

Lift Pot
& 96 & 85 & 100 & 100 & 98 & 98 & 100 & 99
& 96 & 99 & 100 & 99 & 100 & 100 \\

Move Can Pot
& 51 & 55 & 90 & 88 & 76 & 78 & 94 & 97
& 34 & 74 & 98 & 93 & 98 & 98 \\

Move Pillbottle Pad
& 84 & 61 & 100 & 99 & 90 & 90 & 99 & 99
& 93 & 96 & 97 & 90 & 100 & 100 \\

Move Playingcard Away
& 96 & 84 & 100 & 100 & 78 & 72 & 100 & 99
& 100 & 96 & 100 & 100 & 100 & 100 \\

Move Stapler Pad
& 56 & 42 & 77 & 64 & 92 & 82 & 91 & 79
& 83 & 85 & 94 & 87 & 91 & 76 \\

Open Laptop
& 90 & 96 & 98 & 100 & 96 & 98 & 92 & 94
& 95 & 91 & 100 & 100 & 100 & 100 \\

Open Microwave
& 34 & 77 & 62 & 45 & 74 & 66 & 82 & 86
& 95 & 91 & 41 & 43 & 91 & 76 \\

Pick Diverse Bottles
& 81 & 71 & 80 & 85 & 82 & 70 & 89 & 82
& 90 & 91 & 91 & 88 & 86 & 80 \\

Pick Dual Bottles
& 93 & 63 & 100 & 96 & 86 & 86 & 100 & 99
& 96 & 90 & 100 & 95 & 100 & 99 \\

Place A2B Left
& 87 & 82 & 95 & 93 & 94 & 88 & 97 & 93
& 88 & 79 & 98 & 91 & 98 & 95 \\

Place A2B Right
& 87 & 84 & 93 & 99 & 90 & 92 & 97 & 95
& 91 & 87 & 89 & 94 & 97 & 93 \\

Place Bread Basket
& 77 & 64 & 91 & 93 & 82 & 82 & 97 & 95
& 91 & 94 & 92 & 85 & 100 & 95 \\

Place Bread Skillet
& 85 & 66 & 90 & 93 & 94 & 90 & 95 & 90
& 86 & 83 & 90 & 83 & 92 & 89 \\

Place Burger Fries
& 94 & 87 & 96 & 99 & 98 & 96 & 97 & 95
& 98 & 98 & 93 & 96 & 97 & 97 \\

Place Can Basket
& 62 & 62 & 71 & 69 & 78 & 74 & 81 & 84
& 81 & 76 & 92 & 65 & 79 & 67 \\

Place Cans Plasticbox
& 94 & 84 & 99 & 96 & 100 & 100 & 100 & 99
& 98 & 94 & 100 & 95 & 96 & 99 \\

Place Container Plate
& 99 & 95 & 96 & 100 & 98 & 96 & 99 & 97
& 98 & 99 & 100 & 100 & 98 & 99 \\

Place Dual Shoes
& 75 & 75 & 94 & 88 & 96 & 84 & 94 & 89
& 93 & 87 & 98 & 94 & 91 & 94 \\

Place Empty Cup
& 100 & 99 & 100 & 100 & 90 & 90 & 100 & 100
& 99 & 98 & 99 & 100 & 99 & 99 \\

Place Fan
& 87 & 85 & 96 & 96 & 92 & 94 & 99 & 93
& 91 & 87 & 92 & 93 & 94 & 98 \\

Place Mouse Pad
& 60 & 39 & 83 & 89 & 88 & 90 & 93 & 96
& 66 & 68 & 91 & 84 & 96 & 93 \\

Place Object Basket
& 80 & 76 & 89 & 88 & 90 & 92 & 91 & 88
& 81 & 87 & 92 & 90 & 87 & 84 \\

Place Object Scale
& 86 & 80 & 90 & 97 & 88 & 80 & 96 & 95
& 88 & 85 & 95 & 88 & 96 & 96 \\

Place Object Stand
& 91 & 85 & 90 & 94 & 100 & 98 & 99 & 96
& 98 & 97 & 92 & 93 & 99 & 96 \\

Place Phone Stand
& 81 & 81 & 97 & 99 & 82 & 72 & 97 & 97
& 87 & 86 & 93 & 94 & 96 & 97 \\

Place Shoe
& 92 & 93 & 96 & 99 & 98 & 96 & 98 & 98
& 99 & 97 & 100 & 100 & 94 & 97 \\

Press Stapler
& 87 & 83 & 90 & 97 & 96 & 96 & 85 & 82
& 93 & 98 & 98 & 97 & 87 & 93 \\

Put Bottles Dustbin
& 84 & 79 & 95 & 90 & 72 & 70 & 87 & 91
& 81 & 79 & 94 & 92 & 96 & 89 \\

Put Object Cabinet
& 80 & 79 & 94 & 89 & 74 & 74 & 85 & 87
& 88 & 71 & 90 & 82 & 96 & 93 \\

Rotate QRcode
& 89 & 87 & 93 & 89 & 90 & 84 & 96 & 91
& 89 & 73 & 94 & 89 & 94 & 90 \\

Scan Object
& 72 & 65 & 89 & 92 & 60 & 64 & 96 & 91
& 67 & 66 & 96 & 90 & 95 & 92 \\

Shake Bottle
& 99 & 97 & 100 & 100 & 100 & 100 & 100 & 97
& 100 & 97 & 100 & 100 & 100 & 99 \\

Shake Bottle Horizontally
& 99 & 99 & 100 & 100 & 100 & 98 & 100 & 99
& 100 & 98 & 100 & 100 & 100 & 99 \\

Stack Blocks Three
& 91 & 76 & 95 & 97 & 70 & 78 & 99 & 98
& 91 & 95 & 90 & 75 & 95 & 98 \\

Stack Blocks Two
& 97 & 100 & 100 & 100 & 100 & 94 & 100 & 98
& 100 & 98 & 100 & 97 & 100 & 100 \\

Stack Bowls Three
& 77 & 71 & 80 & 81 & 70 & 72 & 86 & 83
& 79 & 87 & 90 & 80 & 87 & 87 \\

Stack Bowls Two
& 95 & 96 & 92 & 98 & 96 & 92 & 94 & 98
& 98 & 98 & 100 & 99 & 93 & 97 \\

Stamp Seal
& 79 & 55 & 90 & 94 & 96 & 98 & 96 & 97
& 93 & 92 & 89 & 88 & 95 & 97 \\

Turn Switch
& 62 & 54 & 61 & 59 & 82 & 84 & 44 & 45
& 84 & 78 & 47 & 56 & 75 & 79 \\

\midrule

\textbf{Average}
& \textbf{82.74} & \textbf{76.76}
& \textbf{91.88} & \textbf{91.78}
& \textbf{86.36} & \textbf{85.04}
& \textbf{92.90} & \textbf{91.50}
& \textbf{88.66} & \textbf{87.02}
& \textbf{92.64} & \textbf{89.80}
& \textbf{94.56} & \textbf{93.04} \\

\bottomrule
\end{tabular*}

\end{table*}

%% file: text/appendix/c_4_full_real_gen.tex
\subsection{Full Real-World Generalization Results}
Here in this section, we report the full real-world generalization results. All aggregate metrics are computed from unrounded underlying values; rounding is applied only to the final reported number.

\begin{table*}[htbp]
\centering
\caption{
Detailed real-world generalization results.
We report success rate (SR, \%) and average task score under different
visual and spatial generalization settings.
}
\label{tab:real_world_generalization}

\resizebox{\textwidth}{!}{
\begin{tabular}{ll|cc|cc|cc|cc}
\toprule

\multirow{2}{*}{Task}
& \multirow{2}{*}{Generalization Setting}
& \multicolumn{2}{c|}{FastWAM}
& \multicolumn{2}{c|}{GR00T-N1.7}
& \multicolumn{2}{c|}{$\pi_{0.5}$}
& \multicolumn{2}{c}{Ours}
\\

&
& SR (\%)
& Score
& SR (\%)
& Score
& SR (\%)
& Score
& SR (\%)
& Score
\\

\midrule

\multirow{4}{*}{Battery Insertion}
& Background
& 50.0 & 7.30
& 40.0 & 6.00
& 70.0 & 8.75
& 60.0 & 8.50
\\

& Appearance
& 10.0 & 3.60
& 5.0 & 3.00
& 35.0 & 7.00
& 40.0 & 8.00
\\

& Spatial Layout
& 20.0 & 4.50
& 15.0 & 4.00
& 60.0 & 8.50
& 50.0 & 8.00
\\

& Avg.
& 26.7 & 5.13
& 20.0 & 4.33
& 55.0 & 8.08
& 50.0 & 8.17
\\

\midrule

\multirow{4}{*}{Battery Assembly}
& Background
& 10.0 & 3.00
& 5.0 & 2.50
& 10.0 & 4.60
& 10.0 & 5.00
\\

& Appearance
& 30.0 & 4.00
& 15.0 & 3.00
& 40.0 & 5.25
& 40.0 & 5.00
\\

& Spatial Layout
& 20.0 & 4.50
& 20.0 & 4.25
& 50.0 & 6.00
& 45.0 & 5.25
\\

& Avg.
& 20.0 & 3.83
& 13.3 & 3.25
& 33.3 & 5.28
& 31.7 & 5.08
\\

\midrule

\multirow{5}{*}{Pack Objects}
& Background
& 80.0 & 8.00
& 20.0 & 5.00
& 85.0 & 8.95
& 80.0 & 8.45
\\

& Appearance
& 30.0 & 7.20
& 20.0 & 3.50
& 80.0 & 9.05
& 80.0 & 9.05
\\

& Geometry
& 60.0 & 6.60
& 25.0 & 5.05
& 75.0 & 8.55
& 85.0 & 8.95
\\

& Spatial Layout
& 50.0 & 5.15
& 5.0 & 2.90
& 80.0 & 8.75
& 70.0 & 8.05
\\

& Avg.
& 55.0 & 6.74
& 17.5 & 4.11
& 80.0 & 8.83
& 78.8 & 8.63
\\

\midrule

\multirow{5}{*}{Stack Cups}
& Background
& 10.0 & 6.00
& 5.0 & 3.50
& 0.0 & 5.00
& 10.0 & 7.00
\\

& Appearance
& 25.0 & 8.00
& 15.0 & 7.50
& 30.0 & 8.13
& 40.0 & 8.50
\\

& Geometry
& 30.0 & 7.00
& 0.0 & 4.80
& 40.0 & 8.15
& 50.0 & 8.63
\\

& Spatial Layout
& 20.0 & 5.50
& 40.0 & 8.50
& 60.0 & 8.50
& 50.0 & 8.00
\\

& Avg.
& 21.3 & 6.63
& 15.0 & 6.08
& 32.5 & 7.44
& 37.5 & 8.03
\\

\bottomrule
\end{tabular}
}

\end{table*}

%% file: text/appendix/case_study.tex
\subsection{Case Study}
In this section, we present qualitative real-world examples of WING
under both the base and generalization settings.

\input{img/case_study}

%% file: img/case_study.tex
\begin{figure}[htbp]
    \centering
    \includegraphics[width=\linewidth]{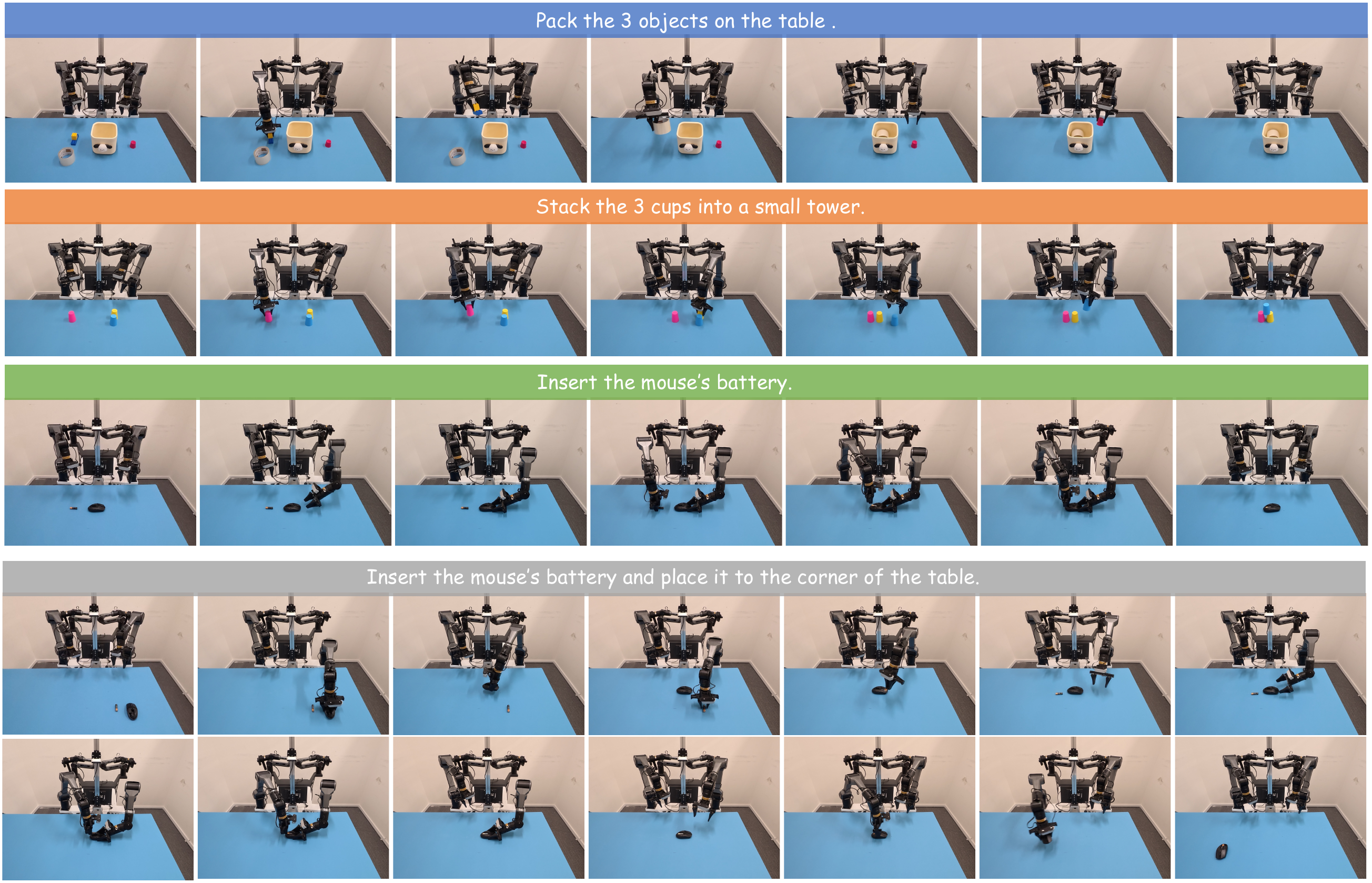}
    \caption{
    Qualitative results under the standard setting.
    From top to bottom, we illustrate four base tasks:
    Pack Objects, Stack Cups, Battery Insertion, and Battery Assembly.
    }
    \label{fig:case_standard}
\end{figure}

\begin{figure}[htbp]
    \centering
    \includegraphics[width=\linewidth]{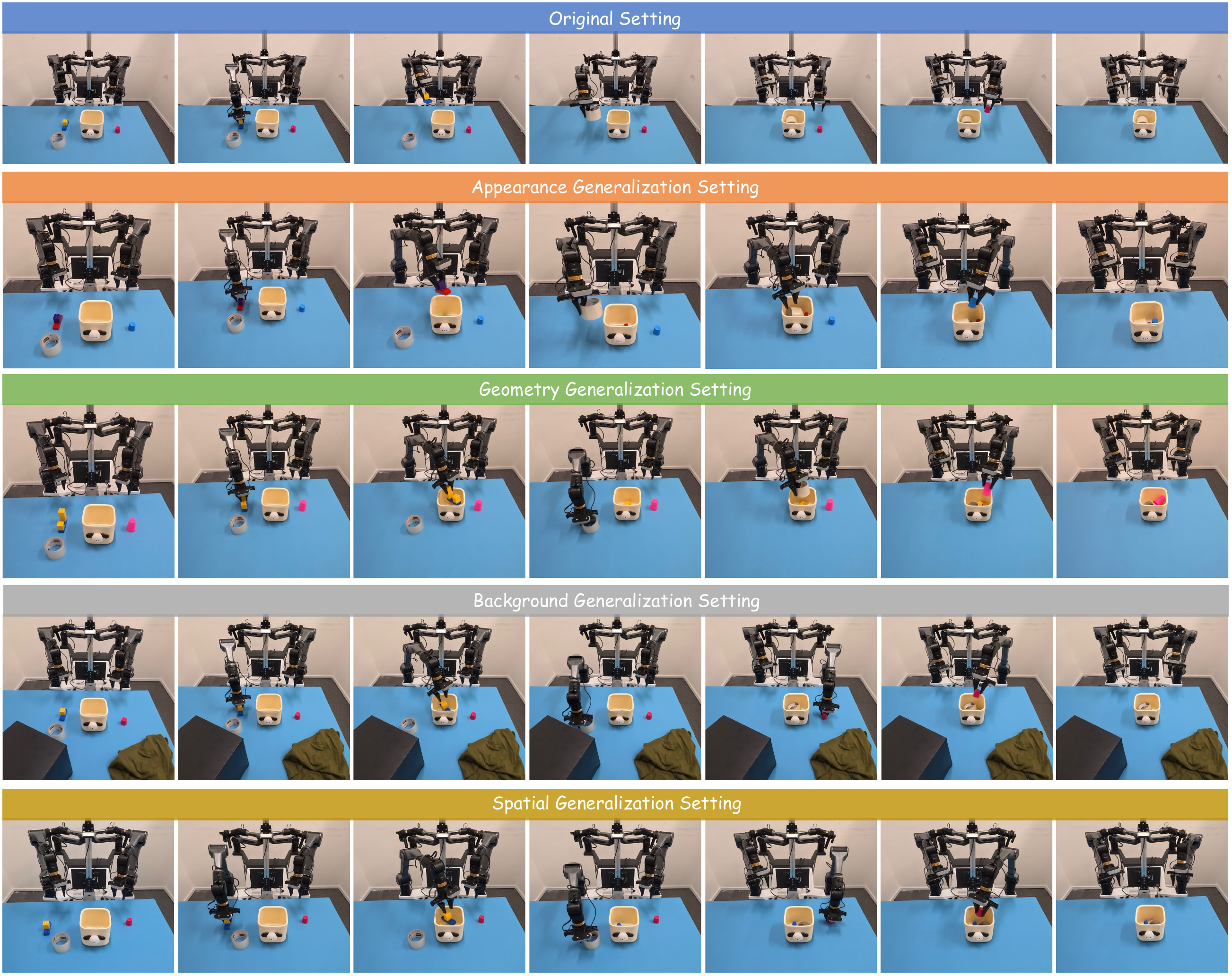}
    \caption{
    Qualitative results for Pack Objects under generalization settings.
    }
    \label{fig:case_pack_generalization}
\end{figure}

\begin{figure}[htbp]
    \centering
    \includegraphics[width=\linewidth]{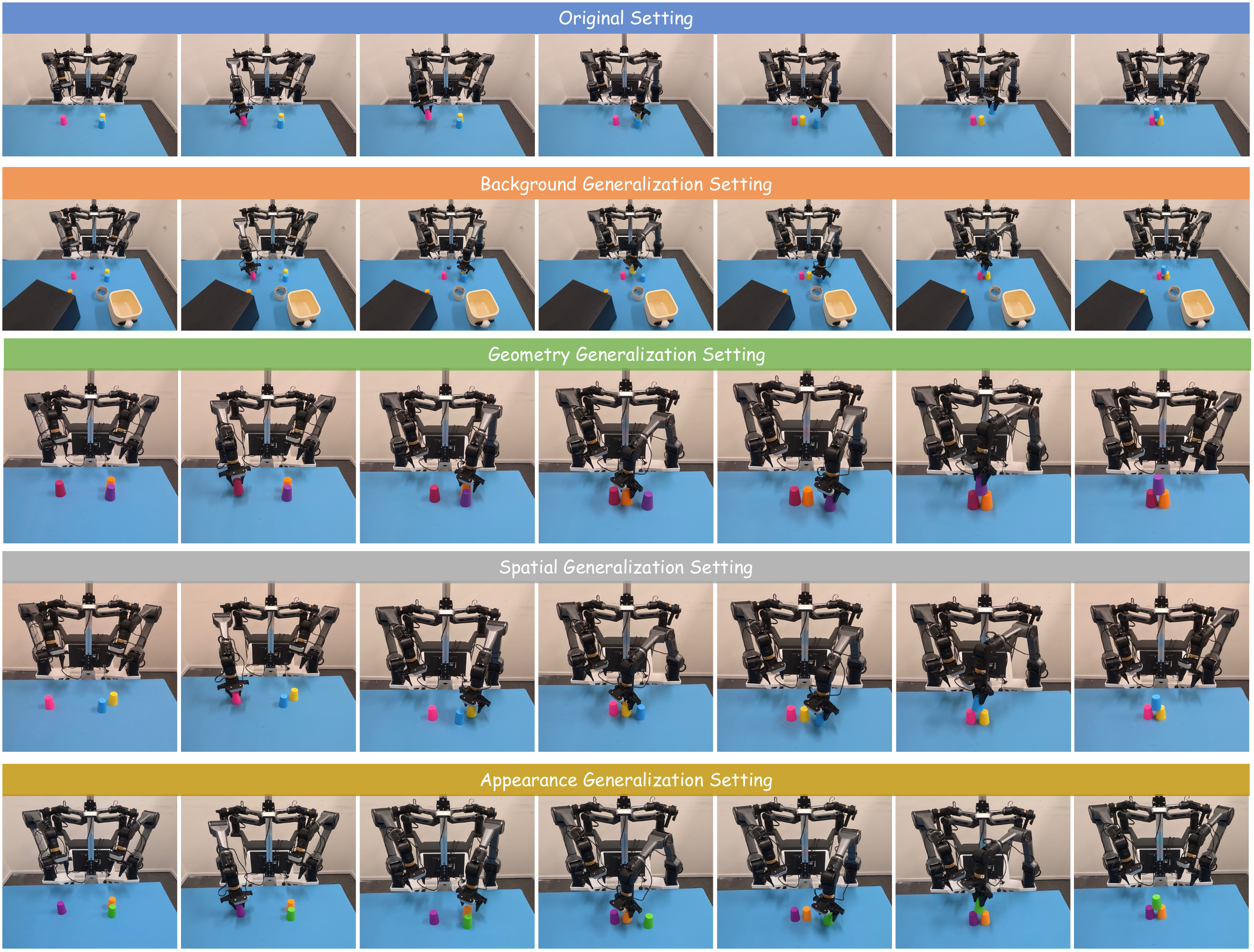}
    \caption{
    Qualitative results for Stack Cups under generalization settings.
    }
    \label{fig:case_stack_generalization}
\end{figure}

\begin{figure}[htbp]
    \centering
    \includegraphics[width=\linewidth]{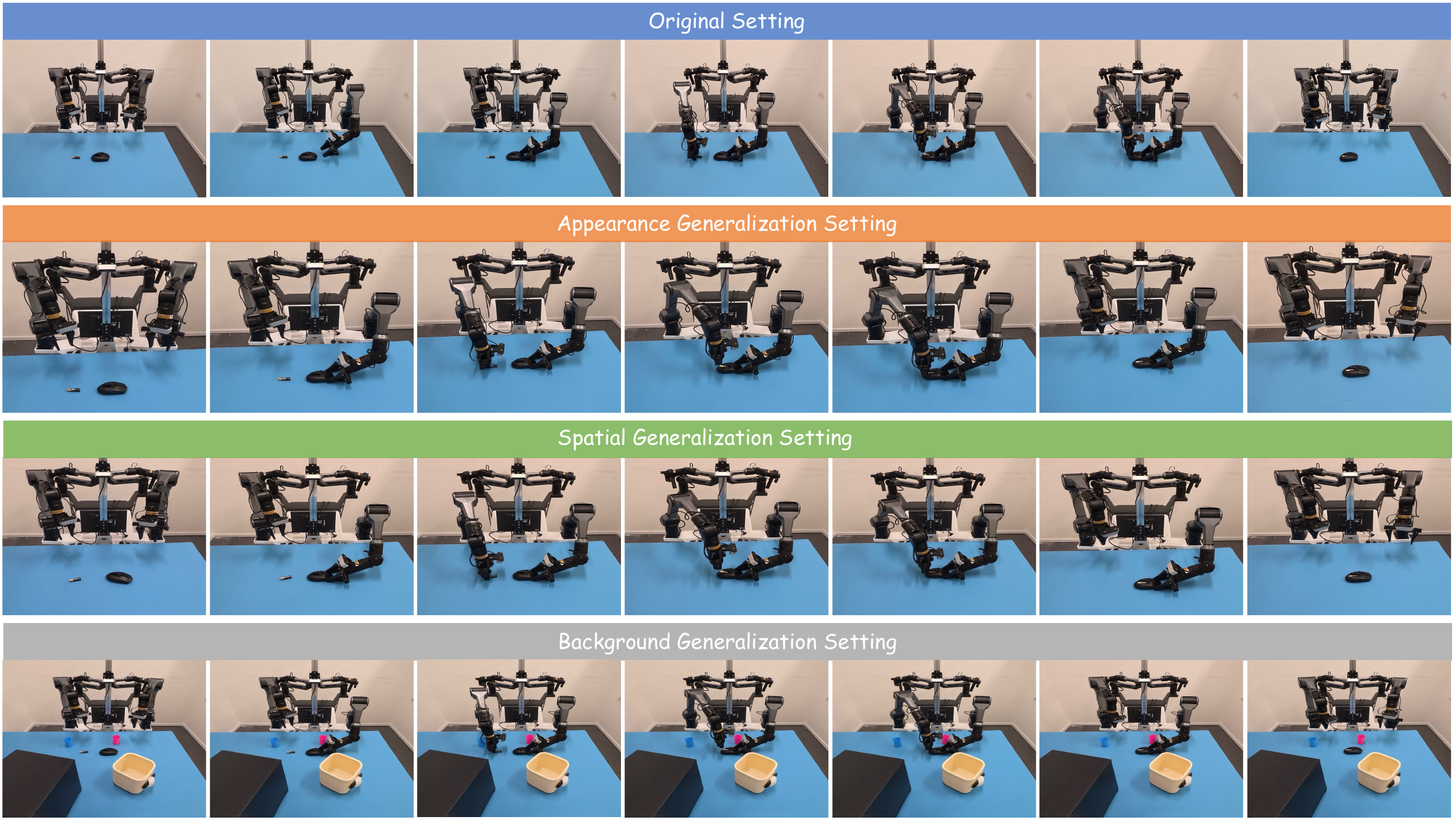}
    \caption{
    Qualitative results for Battery Insertion under generalization settings.
    }
    \label{fig:case_insertion_generalization}
\end{figure}

\begin{figure}[htbp]
    \centering
    \includegraphics[width=\linewidth]{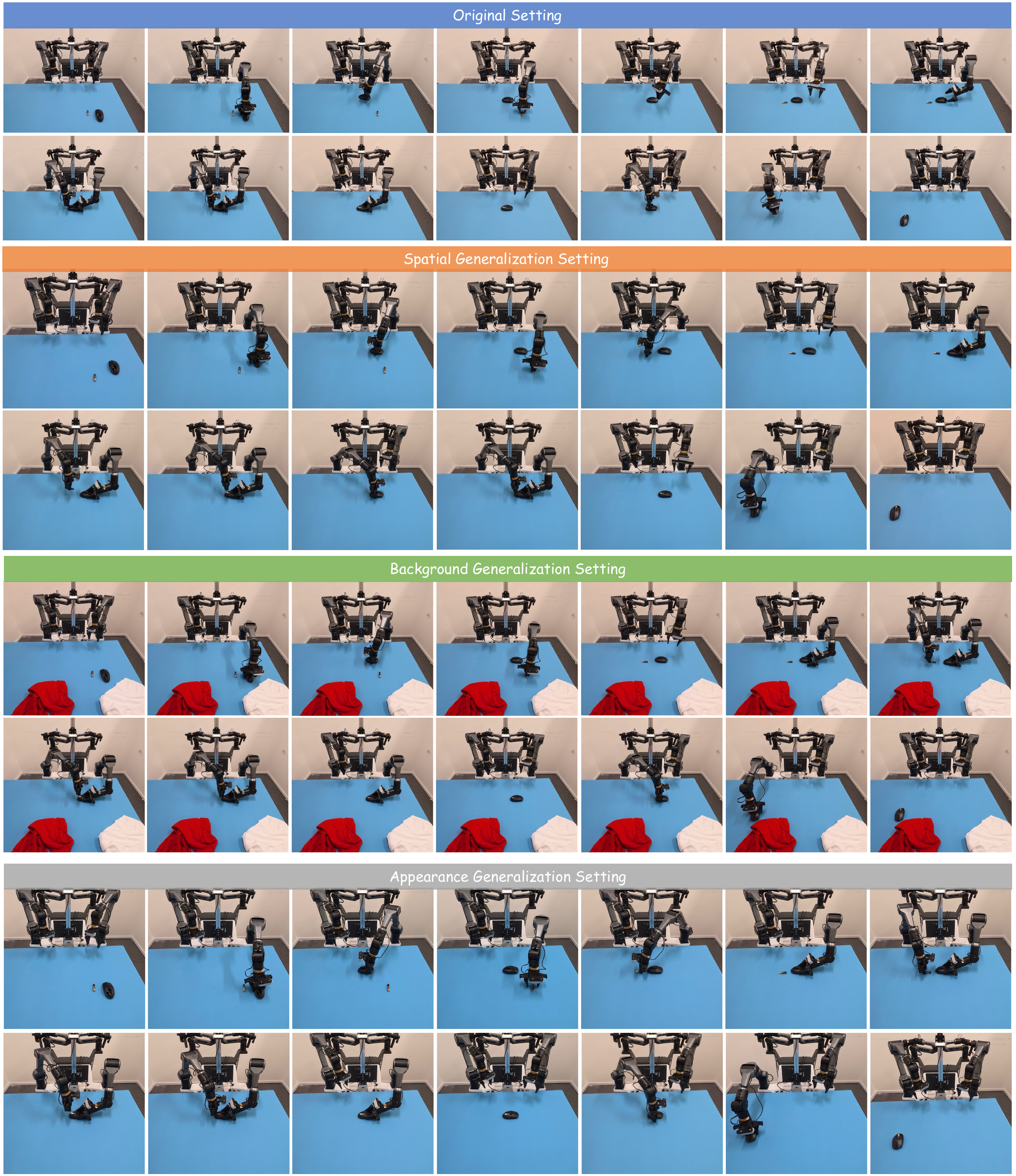}
    \caption{
    Qualitative results for Battery Assembly under generalization settings.
    }
    \label{fig:case_assembly_generalization}
\end{figure}